%% file: iclr2027_conference.tex
\documentclass{article} 
\usepackage{iclr2027_conference,times}

\input{math_commands.tex}

\usepackage[utf8]{inputenc} 
\usepackage[T1]{fontenc}    
\usepackage{hyperref}       
\usepackage{url}            
\usepackage{booktabs}       
\usepackage{amsfonts}       
\usepackage{nicefrac}       
\usepackage{microtype}      
\usepackage[dvipsnames,table]{xcolor}         

\usepackage{kotex}
\usepackage{algorithm}
\usepackage{algpseudocode}
\usepackage{amsmath,amssymb,amsfonts}

\usepackage{graphicx}
\usepackage{makecell}

\usepackage{titlesec} 

\usepackage{amsthm}
\usepackage{multirow}
\usepackage{caption}

\usepackage{subcaption}

\usepackage[most]{tcolorbox}
\tcbuselibrary{skins, breakable}
\usepackage{xcolor}

\definecolor{PromptGray}{HTML}{F6F8FA}
\definecolor{PromptGrayFrame}{HTML}{AAB4BE}
\definecolor{TranslationTeal}{HTML}{6FAFBE}
\definecolor{PromptTeal}{HTML}{2F8FA9}
\definecolor{CompletionTeal}{HTML}{1D6679}

\newtcolorbox[
  auto counter,
  list inside=examplelist
]{bluebox}[2][]{
  enhanced,
  breakable,
  title={Text~\thetcbcounter. \textbf{#1}},
  colback=PromptGray,
  colframe=PromptGrayFrame,
  colbacktitle=black!62,
  coltitle=white,
  fonttitle=\bfseries,
  boxrule=0.6pt,
  arc=1.5pt,
  left=6pt,
  right=6pt,
  top=6pt,
  bottom=6pt,
  #2
}

\newtcolorbox{translatedcontextbox}[1][]{
  enhanced,
  breakable,
  title={\textbf{Context -- Translated Cache Region}},
  colback=TranslationTeal!6!white,
  colframe=TranslationTeal,
  colbacktitle=TranslationTeal,
  coltitle=white,
  fonttitle=\bfseries,
  boxrule=0.8pt,
  arc=1.5pt,
  left=6pt,
  right=6pt,
  top=6pt,
  bottom=6pt,
  #1
}

\newtcolorbox{promptbox}[1][]{
  enhanced,
  breakable,
  title={\textbf{Prompt -- Target-Side Query Region}},
  colback=PromptTeal!7!white,
  colframe=PromptTeal,
  colbacktitle=PromptTeal,
  coltitle=white,
  fonttitle=\bfseries,
  boxrule=0.8pt,
  arc=1.5pt,
  left=6pt,
  right=6pt,
  top=6pt,
  bottom=6pt,
  #1
}

\newtcolorbox{completionbox}[1][]{
  enhanced,
  breakable,
  title={\textbf{Completion -- Target Output Region}},
  colback=CompletionTeal!8!white,
  colframe=CompletionTeal,
  colbacktitle=CompletionTeal,
  coltitle=white,
  fonttitle=\bfseries,
  boxrule=0.8pt,
  arc=1.5pt,
  left=6pt,
  right=6pt,
  top=6pt,
  bottom=6pt,
  #1
}

\newcommand{\nativecolor}{\color[HTML]{808080}}
\newcommand{\transbg}{\cellcolor[HTML]{E8EDF2}}
\newcommand{\spm}[1]{\scalebox{0.75}{$\pm$#1}}

\newtheorem{theorem}{Theorem}[section]

\newtheorem{proposition}[theorem]{Proposition}
\newtheorem{corollary}[theorem]{Corollary}
\newtheorem{assumption}[theorem]{Assumption}
\newtheorem{definition}[theorem]{Definition}

\usepackage{mathtools}
\DeclarePairedDelimiter\p{\lparen}{\rparen}
\DeclarePairedDelimiter\ang{\langle}{\rangle} 
\DeclarePairedDelimiter\abs{\lvert}{\rvert}   
\DeclarePairedDelimiter\norm{\lVert}{\rVert}  
\DeclarePairedDelimiter\bkt{[}{]}             
\DeclarePairedDelimiter\set{\{}{\}}           

\makeatletter
\let\oldp\p \def\p{\@ifstar{\oldp}{\oldp*}}
\let\oldang\ang \def\ang{\@ifstar{\oldang}{\oldang*}}
\let\oldabs\abs \def\abs{\@ifstar{\oldabs}{\oldabs*}}
\let\oldnorm\norm \def\norm{\@ifstar{\oldnorm}{\oldnorm*}}
\let\oldbkt\bkt \def\bkt{\@ifstar{\oldbkt}{\oldbkt*}}
\let\oldset\set \def\set{\@ifstar{\oldset}{\oldset*}}
\let\oldceil\ceil \def\ceil{\@ifstar{\oldceil}{\oldceil*}}
\let\oldfloor\floor \def\floor{\@ifstar{\oldfloor}{\oldfloor*}}
\makeatother

\usepackage{bm}

\newcommand{\ba}{\mathbf{a}}
\newcommand{\bh}{\mathbf{h}}
\newcommand{\bo}{\mathbf{o}}
\newcommand{\bs}{\normalfont{\textbf{s}}}
\newcommand{\bu}{\mathbf{u}}

\newcommand{\bx}{\normalfont{\textbf{x}}}
\newcommand{\by}{\mathbf{y}}
\newcommand{\bz}{\mathbf{z}}

\newcommand{\bA}{\mathbf{A}}

\newcommand{\bF}{\mathbf{F}}

\newcommand{\bK}{\mathbf{K}}
\newcommand{\bQ}{\mathbf{Q}}
\newcommand{\bV}{\mathbf{V}}
\newcommand{\bW}{\mathbf{W}}
\newcommand{\bX}{\mathbf{X}}
\newcommand{\bY}{\mathbf{Y}}

\usepackage{xspace}

\newcommand{\name}[1]{\textsf{#1}\xspace}
\newcommand{\native}{\name{Native}}
\newcommand{\ctoc}{\name{C2C-Project}}
\newcommand{\interlat}{\name{Interlat}}
\newcommand{\kvcomm}{\name{KVComm}}
\newcommand{\lsc}{\name{LSC}}
\newcommand{\mot}{\name{MoT}}
\newcommand{\single}{\name{Single Backbone}}

\newcommand{\gpt}{\name{GPT-2}}
\newcommand{\opt}{\name{OPT}}
\newcommand{\qwen}{\name{Qwen2.5}}

\newcommand{\Llama}{\name{Llama-3.2}}
\newcommand{\Gemma}{\name{Gemma-3}}
\newcommand{\Qwen}{\name{Qwen-3}}

\newcommand{\LlamaIcon}{\raisebox{-0.5ex}{\includegraphics[height=1.15em]{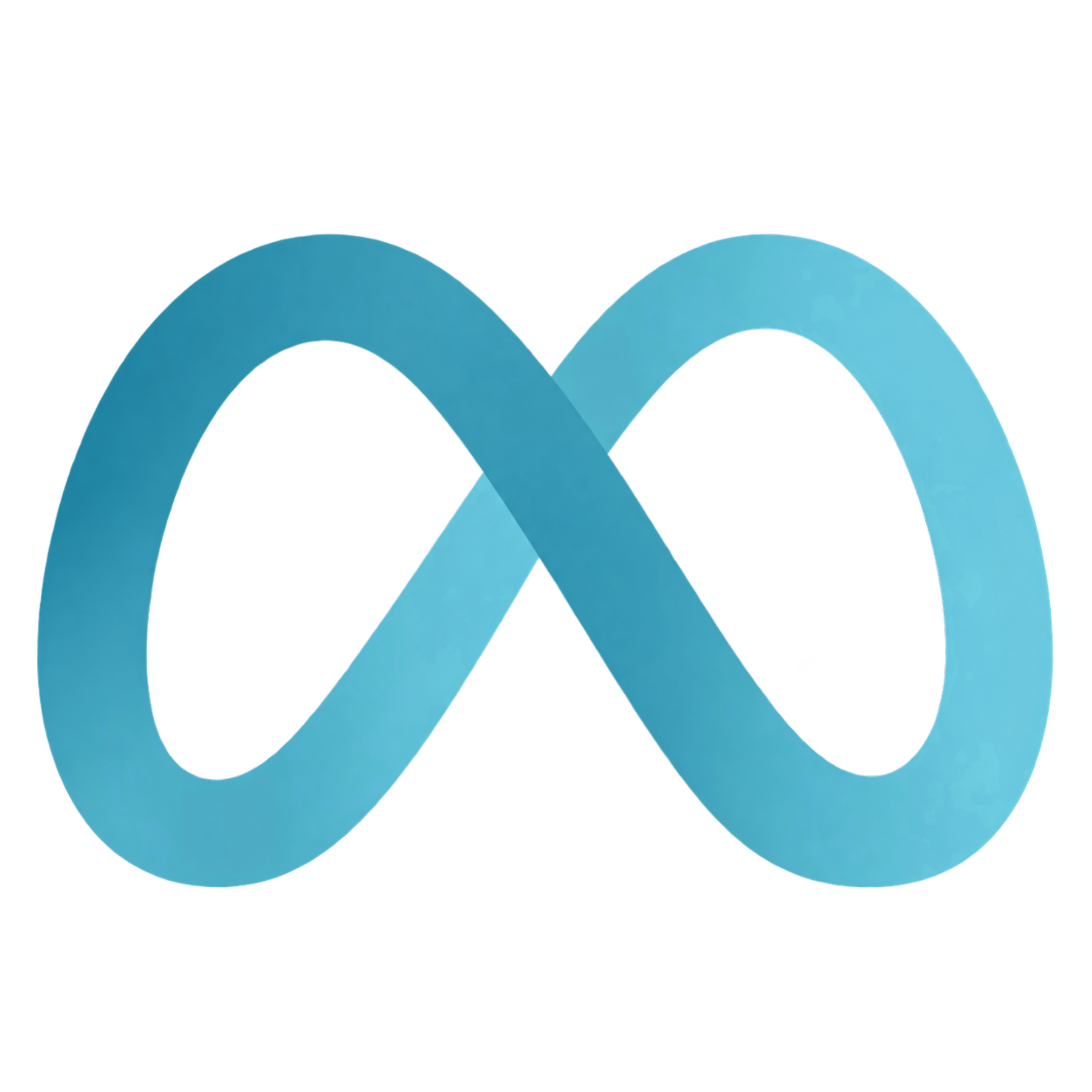}}\xspace}
\newcommand{\GemmaIcon}{\raisebox{-0.5ex}{\includegraphics[height=1.25em]{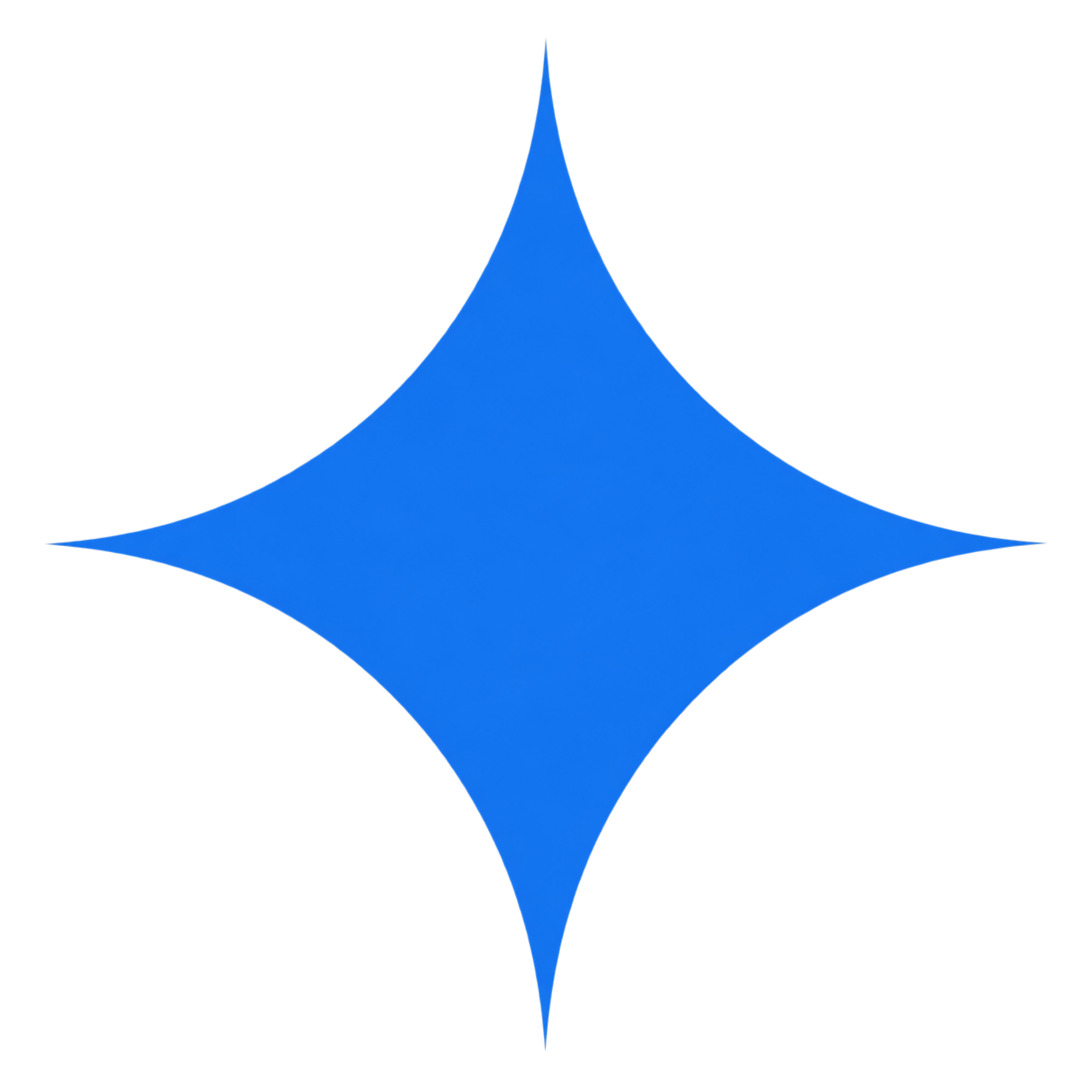}}\xspace}
\newcommand{\QwenIcon}{\raisebox{-0.5ex}{\includegraphics[height=1.05em]{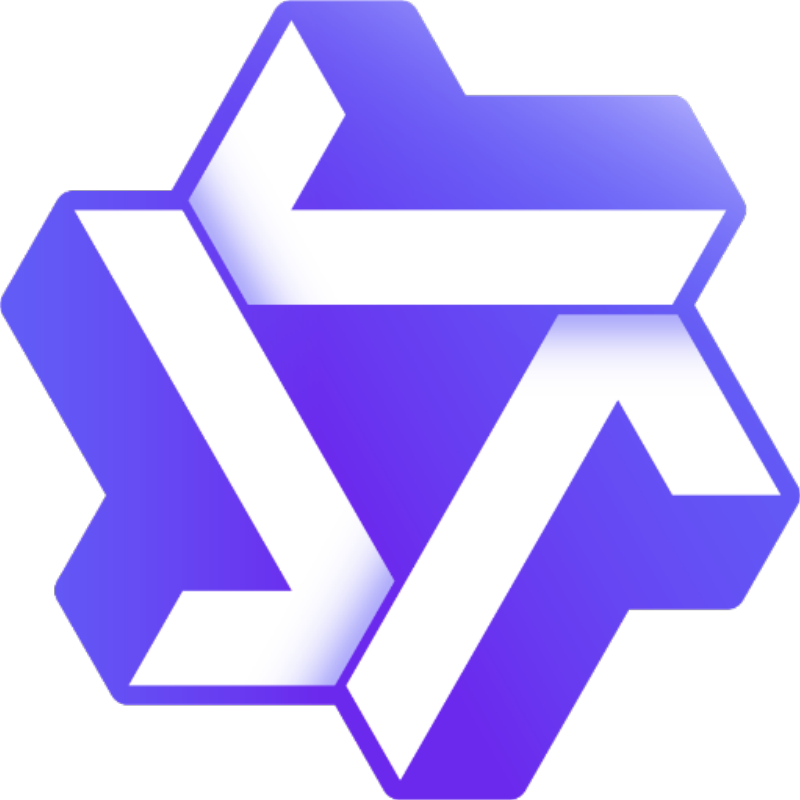}}\xspace}

\title{KV Cache Translation across Heterogeneous Large Language Models}

\author{%
  \textbf{Jin-woo Lee$^{1}$ \quad
  Minkyung Song$^{1}$\thanks{Equal contribution.} \quad
  Junghyun Oh$^{1}$\footnotemark[1] \quad
  Seunghoon Han$^{1}$} \\
  \textbf{Gwangseon Jang$^{2}$ \quad
  Soyoung Park$^{3}$ \quad
  Sungsu Lim$^{1}$\thanks{Corresponding author: \texttt{sungsu@cnu.ac.kr}}} \\
  $^{1}$Chungnam National University \qquad
  $^{2}$Korea Institute of Science and Technology Information \\
  $^{3}$Ulsan National Institute of Science and Technology
}

\iclrfinalcopy 
\begin{document}

\maketitle
\lhead{Preprint}

\begin{abstract}
Heterogeneous Large Language Model (LLM) systems increasingly share contexts,
retrieved evidence, and multi-agent dialogue histories, yet their internal
key-value (KV) caches remain model-specific and cannot be reused across
architectures. Consequently, each model must repeatedly prefill or store caches
for the same context, limiting the scalability of multi-model reasoning and
long-context generation. We propose \textsf{Mixture-of-Translators}
(\textsf{MoT}), a cache translation framework that maps context KV caches from a
source LLM into the cache space of a target LLM. Unlike prior approaches that
depend on a single projection path or global shared latent space, MoT uses
multiple translator modules with token-level routing to capture diverse cache
translations. We further introduce a \textsf{Context Correction Loss} that
aligns the replayed target trajectory with the native target trajectory,
reducing residual translation error. Our analysis reveals two competing failure
modes: propagation error from early injection and correction-deficit error from
late injection. We evaluate heterogeneous translation among \Llama, \Gemma, and
\Qwen, spanning substantially different architectures and KV-cache spaces.
MoT achieves an average accuracy of \(57.6\%\) and F1 of \(0.42\). These scores
recover \(95.0\%\) and \(120.6\%\) of native target performance and reach
\(1.5\times\) and \(3.5\times\) the strongest-baseline scores, respectively.
In practical case studies, MoT enables accurate heterogeneous-agent discussion
with scale-invariant KV-memory behavior as the number of agents grows, while
retaining near-native generation quality in long-context cache-augmented
generation.
These results demonstrate scalable KV-cache reuse across
heterogeneous LLMs.
\end{abstract}

\input{main}

\end{document}

%% file: math_commands.tex
\usepackage{amsmath,amsfonts,bm}

\def\eqref#1{equation~\ref{#1}}

\def\ceil#1{\lceil #1 \rceil}
\def\floor#1{\lfloor #1 \rfloor}
\def\1{\bm{1}}

\DeclareMathAlphabet{\mathsfit}{\encodingdefault}{\sfdefault}{m}{sl}
\SetMathAlphabet{\mathsfit}{bold}{\encodingdefault}{\sfdefault}{bx}{n}



%% file: main.tex
\section{Introduction}

A heterogeneous Large Language Model (LLM) system faces a fundamental
incompatibility: the context state computed by one model is not directly
readable by another. Without translation, each model must reconstruct or store
its own version of shared long contexts, retrieved evidence, dialogue histories,
or agent-generated states. Thus, the key-value (KV) cache becomes more than an
inference artifact; it is a model-specific memory whose reuse depends on
transfer across cache spaces. Prompt caching and cache-augmented generation
demonstrate the value of context reuse~\citep{openai_prompt_caching,
anthropic_prompt_caching, chan2024dontdorag, lu2025turborag}, but existing
mechanisms remain tied to a single model or cache format.

This limitation becomes critical in emerging LLM workflows. Multi-agent
reasoning repeatedly processes shared passages, questions, and dialogue
histories, causing memory to grow with the number of agents
~\citep{li2023camel, du2024improving, wu2024autogen}, while long-context
cache-augmented generation avoids repeated prefilling only within a
model-specific cache representation
~\citep{chan2024dontdorag, lu2025turborag, hcache}. Scalable context reuse
therefore requires translating internal KV states across heterogeneous models
without sacrificing generation quality.

Recent work explores language-free communication through cache projection,
selective KV sharing, hidden-state communication, and shared latent cache
alignment~\citep{c2c, kvcomm, interlat, lsc}. However, fixed projection paths,
compatibility assumptions, or shared spaces can be insufficient when source and
target models differ in depth, hidden dimension, attention layout, or KV-head
structure. Such heterogeneity may require greater translation capacity than a
single universal translator can provide.

We propose \textsf{MoT (Mixture-of-Translators)}, which maps source context KV
caches into the target cache space using token-level routing over multiple
translators. We further introduce a \textsf{Context Correction Loss} that aligns
the replayed target trajectory with the native target trajectory. Our analysis
identifies two competing errors: early translation can propagate translation
shift through many layers, whereas late translation leaves too few upper layers
for correction. MoT reduces the translation shift that drives propagation, while
context correction directly supervises target-side replay to reduce the
remaining last-state shift.

\begin{figure*}[t!]
    \centering

    \begin{minipage}{0.48\textwidth}
        \centering
        \includegraphics[width=\textwidth,height=4.3cm]{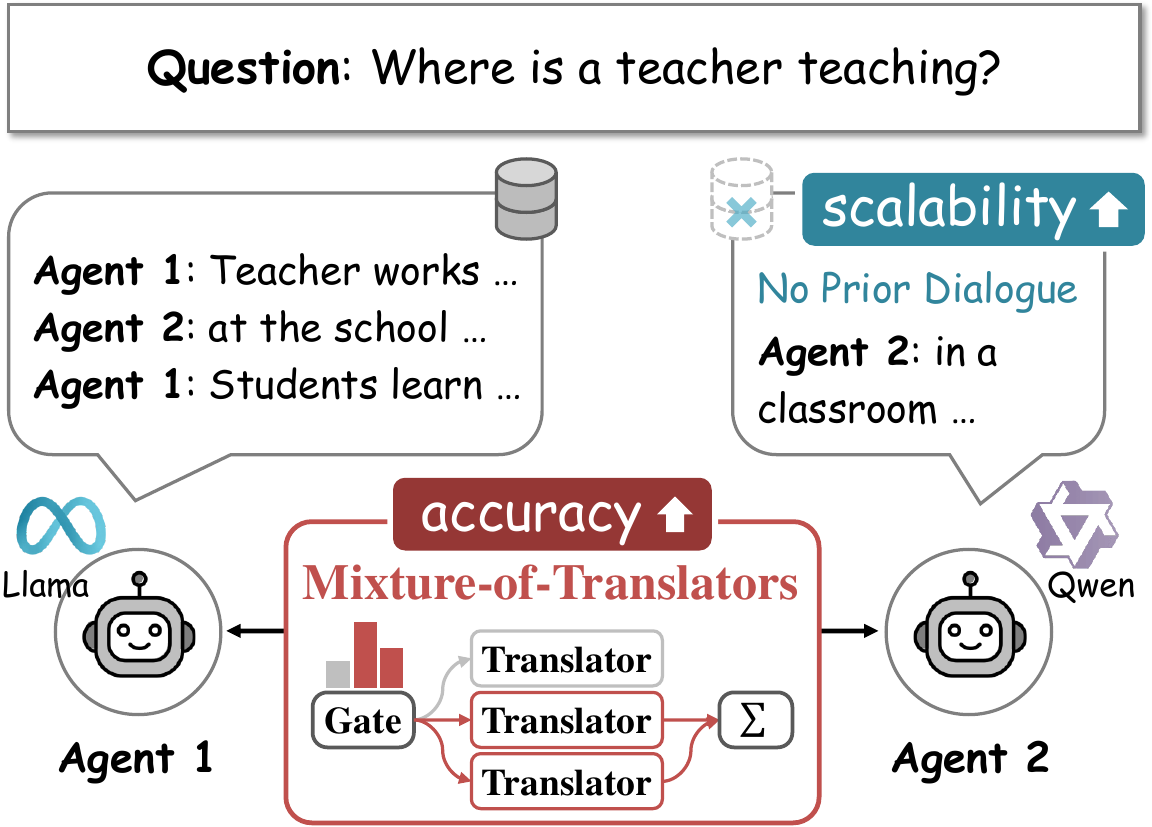}
        \captionof{figure}{\textsf{Mixture-of-Translators} enables accurate heterogeneous-agent discussion with scalable memory reuse through translated caches.}
        \label{fig:multi_agents_intro}
    \end{minipage}
    \hfill
    \begin{minipage}{0.48\textwidth}
        \centering
        \includegraphics[width=\textwidth,height=4.3cm]{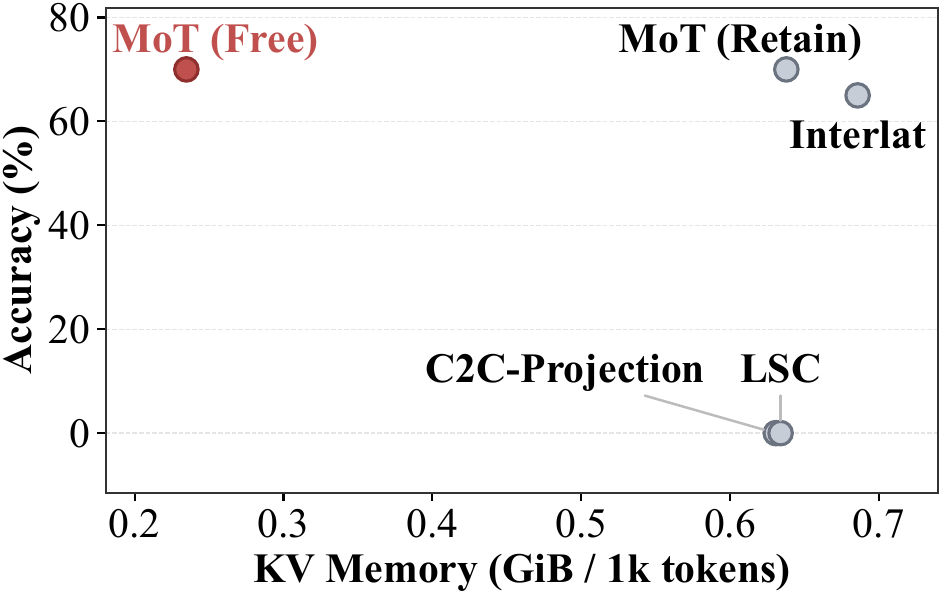}
        \captionof{figure}{10-agent performance with alternating \LlamaIcon~\Llama~\textsf{(3B)} and \QwenIcon~\Qwen~\textsf{(4B)} agents and heterogeneous cache translation across turns.}
        \label{fig:performance_landscape}
    \end{minipage}

\end{figure*}

Figure~\ref{fig:multi_agents_intro} illustrates how \mot enables accurate discussion among heterogeneous agents while reusing translated caches for scalable memory management.
Figure~\ref{fig:performance_landscape} shows a clear quality--efficiency
advantage in heterogeneous 10-agent reasoning: \mot preserves effective
cross-model discussion while substantially reducing KV-memory usage.

We evaluate \mot on closed-set and extractive QA across heterogeneous
translations among \Llama~\textsf{(3B)}, \Gemma~\textsf{(4B)}, and
\Qwen~\textsf{(4B)}. Across these translations, MoT achieves the strongest
average accuracy and F1 among the compared methods, with \(57.6\%\) accuracy
and \(0.42\) F1. These scores recover \(95.0\%\) and \(120.6\%\) of native
target performance and reach \(1.5\times\) and \(3.5\times\) the
strongest-baseline scores, respectively. Multi-agent reasoning and long-context cache-augmented generation further
demonstrate quality-preserving reuse, with scale-invariant cache-memory behavior
as agent interactions grow and efficient storage reuse for long contexts.

\paragraph{Key Contributions.}
\begin{enumerate}
    \item \textit{\bf Competing Translation Errors.}
    We identify propagation error and correction-deficit error as two competing
    causes of cache-translation failure.

    \item \textit{\bf Mixture-of-Translators Architecture.}
    We introduce \mot, which uses token-level routing over multiple translators
    to support heterogeneous KV-cache translation.

    \item \textit{\bf Context Correction Loss.}
    We propose a replay-trajectory alignment objective that reduces residual
    target-side error after cache translation.

    \item \textit{\bf Empirical Validation.}
    Across heterogeneous translations among \Llama, \Gemma, and \Qwen, \mot
    achieves \(57.6\%\) average accuracy and \(0.42\) average F1, recovering
    \(95.0\%\) and \(120.6\%\) of native target performance and reaching
    \(1.5\times\) and \(3.5\times\) the strongest-baseline scores, respectively.

    \item \textit{\bf Practical Case Studies.}
    Multi-agent reasoning demonstrates quality-preserving cache reuse with
    scale-invariant KV-memory behavior as interactions grow, while long-context
    CAG demonstrates effective storage reuse across large context budgets.
\end{enumerate}

\section{Preliminaries}

\subsection{Problem Setup}
We consider context-conditioned generation, where context
\(\bx^{\mathrm{Ctx}}\) and prompt \(\bx^{\mathrm{Prompt}}\) generate completion \(y\).
Let \(\mathcal M^{\mathrm{Src}}\) and \(\mathcal M^{\mathrm{Tgt}}\) denote the source and target models. Because their cache spaces differ, the source context cache
\(\bK\bV^{\mathrm{Src},\mathrm{Ctx}}\) cannot be reused directly by the target. We define \textbf{cache translation} as learning
\(\bK\bV^{\mathrm{Src},\mathrm{Ctx}} \to \widehat{\bK\bV}^{\mathrm{Tgt},\mathrm{Ctx}}\),
which maps the source cache into the target cache space.
We view cache translation as directional communication between source and target layers. A connected source-to-target layer pair \(c=(i,j)\) defines a \textbf{channel}, and the translation channel set is
\(\mathcal C=\{(i_r,j_r)\}_{r=1}^{|\mathcal C|}\),
where \(|\mathcal C|\) is the number of channels. Because translation quality depends on cache injection location, we analyze two competing error dynamics induced by channel placement.

\subsection{Competing Dynamics of Two Translation Errors}
\begin{figure*}[t!]
    \centering

    \begin{minipage}{0.45\textwidth}
        \centering
        \includegraphics[width=\textwidth,height=3.7cm]{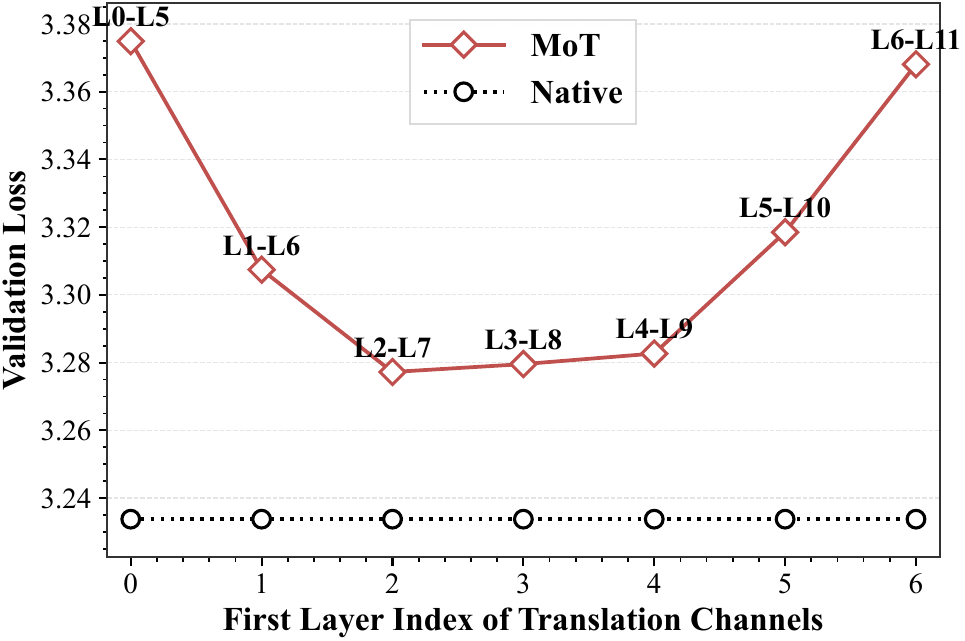}
        \captionof{figure}{U-shaped loss caused by two competing errors.}
        \label{fig:val_loss}
    \end{minipage}
    \qquad
    \begin{minipage}{0.45\textwidth}
        \centering
        \includegraphics[width=\textwidth,height=3.7cm]{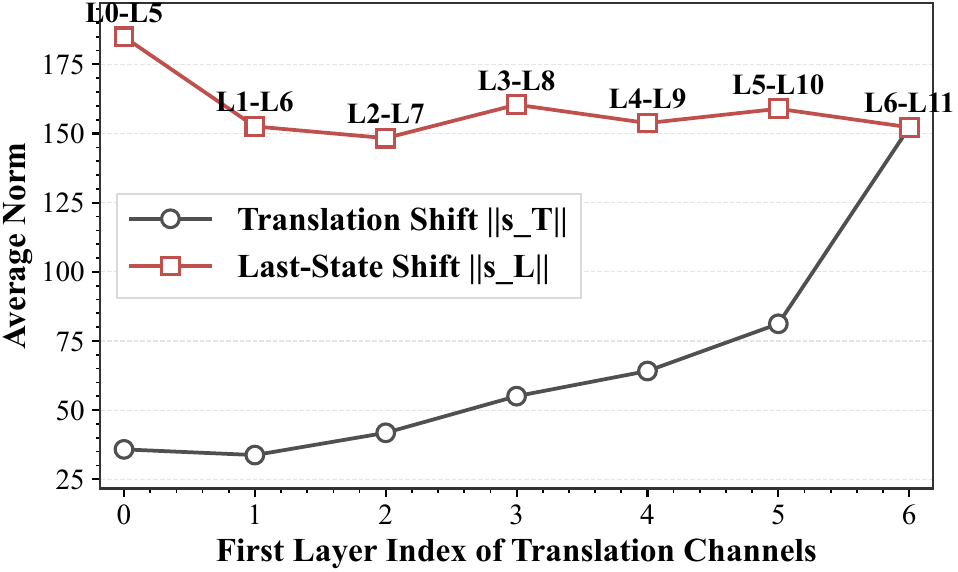}
        \captionof{figure}{Exponential decay of propagated last-state shift.}
        \label{fig:last_state_shift}
    \end{minipage}
\end{figure*}

As shown in Figure~\ref{fig:val_loss}, for \texttt{gpt2}$\to$\texttt{gpt2} translation, we varied the translation layer
position by sliding a channel window of size 6 while increasing the start layer
from 0 to 6. Across these positions, we consistently observed a U-shaped loss curve. This pattern suggests that different types of translation error dominate at the two extremes of the layer range: early windows suffer more from propagated translation shifts, whereas later windows leave fewer upper layers available for correction. In this section, we theoretically analyze the competing dynamics of these two errors.

We consider cache translation through the residual recurrence of a
Transformer decoder. Motivated by prior work that models perturbation-induced
feature shifts across layer positions~\citep{zheng2025spurious}, and by analyses
that use residual connections as a useful lens for modeling signal and
representation propagation in Transformer-like architectures
~\citep{he2023deep, qin2025convergence}, we analyze how a
translation-induced hidden-state shift behaves as a function of the highest
target translation layer.

\begin{definition}
\label{def:residual}
(Residual Transformer Recurrence) \\
For each target layer \(\ell\), let the layer recurrence be written as
\[
\bh_\ell
=
\bh_{\ell-1}+F_\ell(\bh_{\ell-1}),
\qquad
1\le \ell \le L.
\]
where \(\bh_\ell\in\mathbb{R}^{d}\) denotes the hidden states after the \(\ell\)-th
layer, and \(F_\ell\) denotes the residual branch. The residual branch
\(F_\ell\) absorbs the attention, MLP, normalization, and their sequential composition.
\end{definition}

\emph{Full details including formal assumptions and proofs are provided in Appendix~\ref{app:theoretical_analysis}.}

\paragraph{Propagation Error.}

As shown in Figure~\ref{fig:last_state_shift}, we empirically observe that when an error induced by translation at a lower layer propagates to the final layer, the resulting hidden-state shift decays exponentially. In this paragraph, we theoretically model this phenomenon.

\begin{proposition}[Exponential Decay of Propagated Last-State Shift]
\label{prop:source_propagation_error}
Suppose that the
translation-induced shift component formed at the first target translation layer
propagates to the final hidden states through the same target residual recurrence.
Then,
\begin{equation}
\|\bs_L\|
\le
(1+\delta)^{L-T_{\mathrm{Start}}}
\|\bs_{T_{\mathrm{Start}}}\|,
\qquad
T_{\mathrm{Start}}:=\min \{j \mid (i,j)\in\mathcal C\}.
\label{eq:combined_shift_bound}
\end{equation}
Thus, the upper bound on the last-state shift \(\|\bs_L\|\) decreases exponentially as \(T_{\mathrm{Start}}\) increases.
\end{proposition}

\paragraph{Correction Deficit Error.}

As shown in Figure~\ref{fig:correction_decomposition}, moving the translation window upward leaves fewer upper layers available to correct the translation-induced shift. We summarize this behavior using three quantities: \(\alpha_{T:L}\) for anti-shift correction, \(\beta_{T:L}\) for orthogonal shift, and \(d_{T:L}\) for the remaining shift along the original direction.
Empirically, \(\alpha_{T:L}\) decreases as the translation layer moves upward, whereas \(\beta_{T:L}\) remains nearly unchanged. This indicates that late translation leaves fewer layers available for anti-shift correction, making correction harder. We verify this trend through the increase of the correction-deficit coefficient \(d_{T:L}\). Formal definitions are deferred to Definition~\ref{def:shift_delta_and_correction_deficit}.

\begin{proposition}[Increase of Correction-Deficit Coefficient]
\label{prop:correction_deficit_coefficient_bound}
Suppose that the translation
shift at the highest target translation layer is nonzero, i.e.,
\(\bs_T\neq 0\), where \(T:=\max \{j \mid (i,j)\in\mathcal C\}\).
Then, as \(T\) approaches the final layer, the remaining correction opportunity
decreases, and the correction-deficit coefficient reaches its terminal value $d_{L:L}=1$.
\end{proposition}

\paragraph{Competing-Dynamics Translation Objective.}

\begin{figure*}[t]
    \centering

    \begin{minipage}{0.325\textwidth}
        \centering
        \includegraphics[width=\textwidth,height=3.2cm]{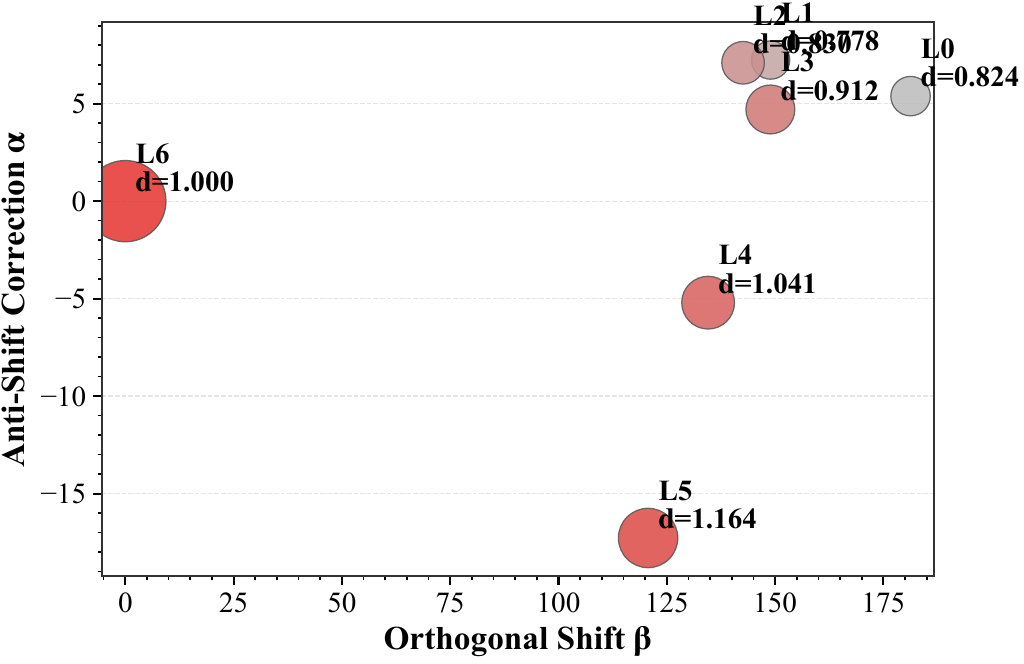}
        \captionof{figure}{Decomposition of\newline correction coefficients.}
        \label{fig:correction_decomposition}
    \end{minipage}
    \hfill
    \begin{minipage}{0.325\textwidth}
        \centering
        \includegraphics[width=\textwidth,height=3.2cm]{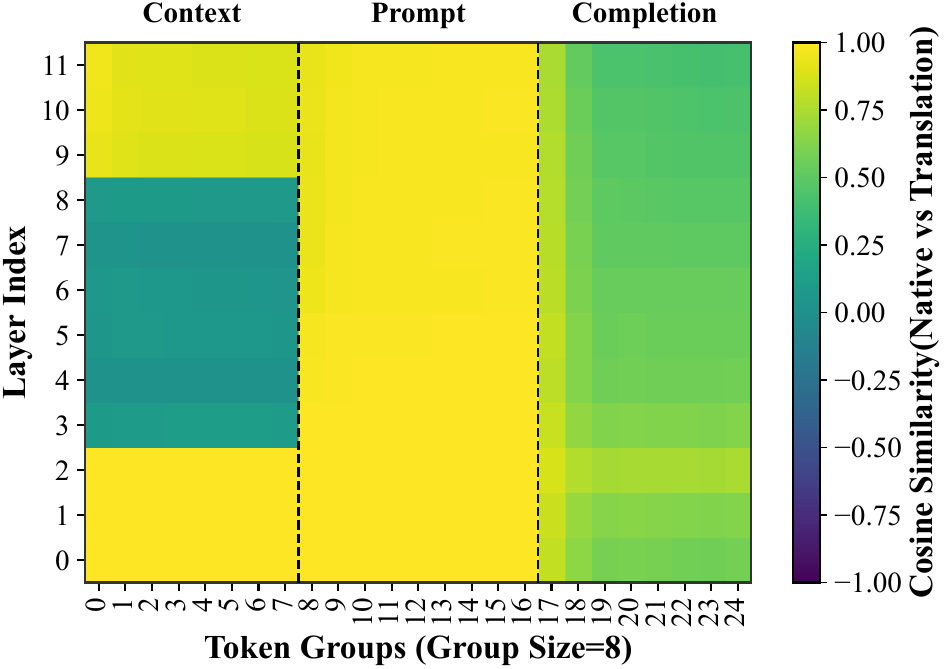}
        \captionof{figure}{Error propagation and correction.}
        \label{fig:error_prop_3}
    \end{minipage}
    \hfill
    \begin{minipage}{0.325\textwidth}
        \centering
        \includegraphics[width=\textwidth,height=3.2cm]{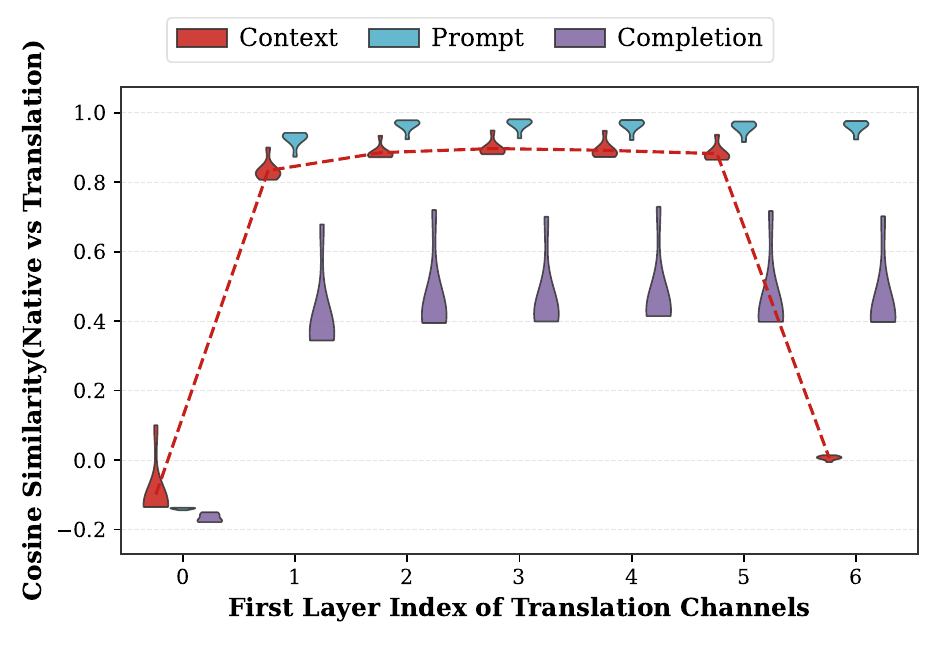}
        \captionof{figure}{Inverted U-shaped\newline
        similarity curve.}
        \label{fig:kv_sim_all_translations}
    \end{minipage}

\end{figure*}

Figure~\ref{fig:error_prop_3} compares similarity to the native KV cache when translating the context over layers $[3,8]$. Native denotes the reference setting where the target model directly prefills \(\bx^{\mathrm{Ctx}}\) without translation. Across the context, prompt, and completion token groups, we observe both correction and propagation: similarity increases toward later layers for token group 0, while error also grows in the upper-right region, indicating accumulated propagation.
Figure~\ref{fig:kv_sim_all_translations} further shows similarity trends while sliding a 6-layer translation window from layer 0 to layer 6. Unlike the prompt and completion regions, the context region exhibits an inverted U-shaped similarity curve, consistent with the U-shaped loss in Figure~\ref{fig:val_loss}. This suggests that insufficient correction is especially pronounced for translated context caches.

Together, these observations show that translation-layer selection induces two competing errors. The first is the translation shift \(\bs_T\) accumulated inside the channel window, whose bound can increase as the highest translated layer \(T\) moves upward. The second is the final last-state shift \(\bs_L\), which remains when upper layers after translation cannot sufficiently correct the shift. Thus, translating too early can amplify propagation error, whereas translating too late can increase correction-deficit error.

We now decompose the components that constitute \(\bs_L\).

\begin{proposition}[Last-State Shift Decomposition]
\label{prop:last_state_shift_decomposition}
Assume that \(\bs_T\neq 0\). For the correction deficit coefficient \(d_{T:L}\) and the orthogonal shift coefficient \(\beta_{T:L}\) defined in Definition~\ref{def:shift_delta_and_correction_deficit},
\begin{equation}
\|\bs_L\|^2
=
d_{T:L}^2\|\bs_T\|^2+\beta_{T:L}^2
\label{eq:last_state_shift_decomposition}
\end{equation}
holds. Therefore, the last-state shift can be decomposed into the correction deficit component that remains along the translation-shift direction and the orthogonal shift component.
\end{proposition}

Proposition~\ref{prop:last_state_shift_decomposition} shows that the last-state shift is not determined solely by the translation shift. Even when the translation shift \(\|\bs_T\|\) is small, the last-state shift \(\|\bs_L\|\) can remain large due to correction deficit \(d_{T:L}\) or orthogonal shift \(\beta_{T:L}\). Conversely, reducing only the last-state shift does not necessarily prevent a large translation shift from forming.

Therefore, translation should jointly minimize the translation shift \(\|\bs_T\|\) and the last-state shift \(\|\bs_L\|\).
The former captures error accumulated within the selected channel set, while the latter captures residual error after upper-layer correction. Optimizing only one is insufficient; robust cache translation requires controlling both competing error dynamics.

\section{Proposed Method: \textsf{Mixture-of-Translators}}
\label{sec:proposed_method}

\begin{figure}[t!]
    \centering
    \includegraphics[width=\textwidth]{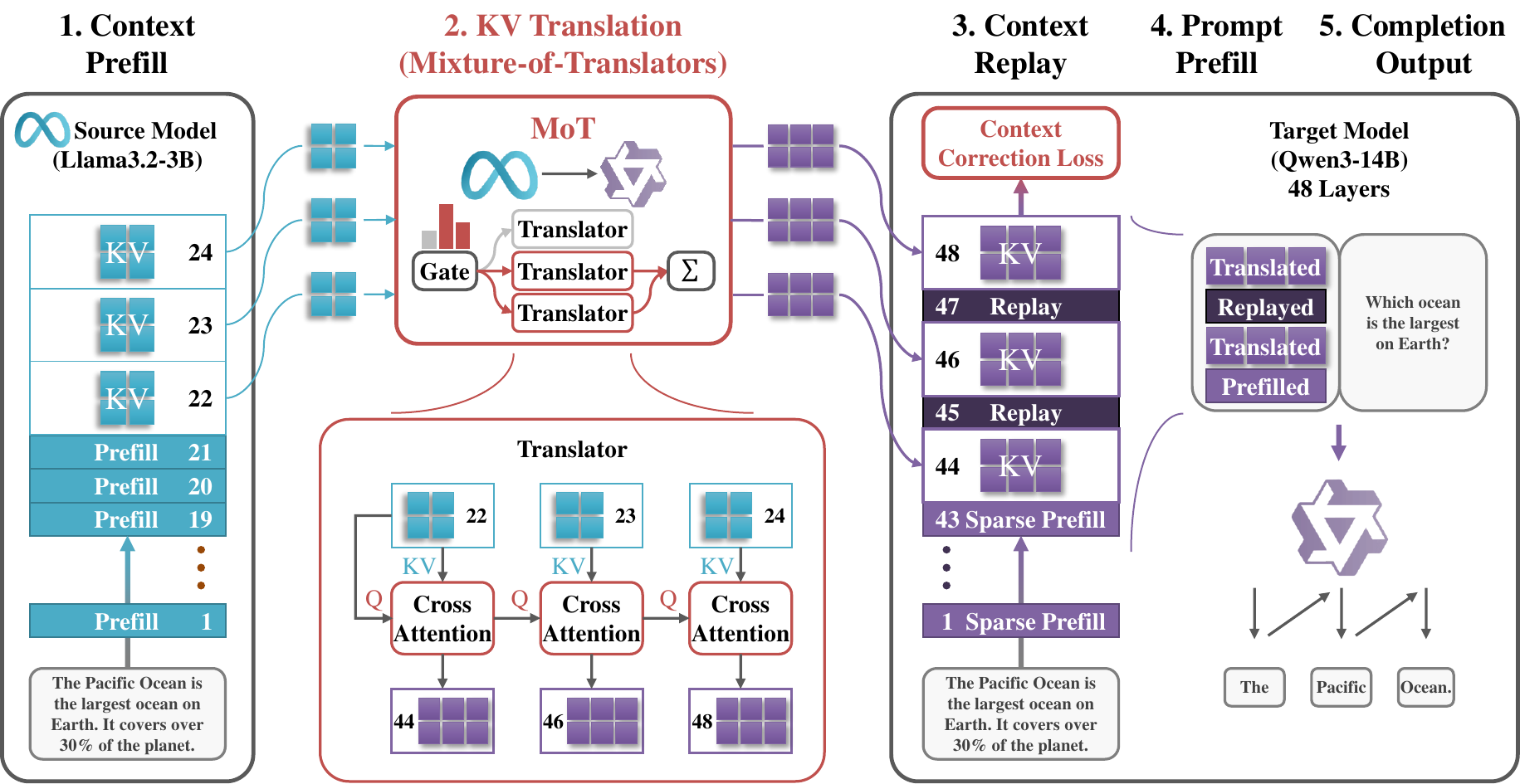}
    \caption{Overview of \textsf{Mixture-of-Translators} and \textsf{Context Correction Loss}.}
\label{fig:framework}
\end{figure}

The theoretical analysis in the preliminaries shows that cache translation must jointly control two error dynamics: the translation shift \(\|\bs_T\|\) formed inside the channel window and the last-state shift \(\|\bs_L\|\) remaining after insufficient upper-layer correction. Based on this analysis, we propose two learning components. First, Section~\ref{subsec:mixture-of-translators} introduces \textbf{\textsf{MoT (Mixture-of-Translators})}, which reduces token-level translation mismatch inside the channel window through gated translator routing. Second, Section~\ref{subsec:context-correction-loss} introduces \textbf{\textsf{Context Correction Loss}}, which aligns the replayed target trajectory with the native target trajectory after translation and directly targets the remaining last-state shift.

These two components are used within the overall cache-translation pipeline shown in Figure~\ref{fig:framework}. The pipeline consists of five stages: \textcircled{\scriptsize 1} \textbf{Context Prefill}, which computes the source context cache and selects translation channels; \textcircled{\scriptsize 2} \textbf{KV Translation}, where \mot maps the source cache into the target cache space; \textcircled{\scriptsize 3} \textbf{Context Replay}, which reconstructs the target-side context cache with context-correction supervision; \textcircled{\scriptsize 4} \textbf{Prompt Prefill}, which conditions the target model on the reconstructed cache; and \textcircled{\scriptsize 5} \textbf{Completion Output}, which generates the final response. Channel mapping and cache replay are detailed in Appendix~\ref{app:channel-mapping} and Appendix~\ref{app:cache-replay}.

\subsection{Mixture-of-Translators}
\label{subsec:mixture-of-translators}

Each translator in \mot uses the backbone architecture described in Appendix~\ref{app:backbone-translator}. A single backbone translator applies one translation rule to the entire cache, which can limit its ability to capture token-dependent source-to-target correspondences. To reduce the translation shift \(\|\bs_T\|\), \mot performs token-level routing over multiple translators, allowing different tokens to exploit distinct translation specializations. We formalize this effect and its connection to the translation shift in Appendix~\ref{app:mot}.

\begin{definition}[\textsf{Mixture-of-Translators}]
\label{def:mot}
For a fixed channel set \(\mathcal C\), let \(\bz_u\) denote the source context-cache input at token \(u\).
We introduce \(N_{\mathrm{Tr}}\) translators $\{\operatorname{Tr}_{m}\}_{m=1}^{N_{\mathrm{Tr}}}$
and a gating function that produces nonnegative token-level scores
$g(\bz_u)=\left(g_1(\bz_u),\dots,g_{N_{\mathrm{Tr}}}(\bz_u)\right)$.
Here, Top-\(K\) denotes the number of translators used for each token. Let the selected translator indices at token \(u\) be $\mathcal T_{\mathrm{Top}\text{-}K}(u):=\operatorname{Top-K}\left(g(\bz_u)\right)$.
The selected scores are renormalized as
\[
\widetilde g_{m}(\bz_u)
:=
\frac{
g_{m}(\bz_u)
}{
\sum_{r\in \mathcal T_{\mathrm{Top}\text{-}K}(u)}
g_{r}(\bz_u)
},
\qquad
m\in \mathcal T_{\mathrm{Top}\text{-}K}(u).
\]
Then, the output of \textsf{Mixture-of-Translators} is defined as
\[
\operatorname{Tr}_{\mathrm{MoT}}(\bz_u)
:=
\sum_{m\in\mathcal T_{\mathrm{Top}\text{-}K}(u)}
\widetilde g_{m}(\bz_u)\,
\operatorname{Tr}_{m}(\bz_u).
\]
\end{definition}

\subsection{Context Correction Loss.}
\label{subsec:context-correction-loss}

\mot reflects token-specific translation difficulty and thereby reduces the translation shift \(\|\bs_T\|\) formed inside the channel window. However, even if \(\|\bs_T\|\) becomes small, the last-state shift \(\|\bs_L\|\) may still remain when the upper layers after translation fail to provide sufficient correction. Therefore, it is necessary to directly constrain the target model after translation to follow the native hidden-state trajectory. To this end, we introduce the \textsf{Context Correction Loss}.

For a fixed channel set \(\mathcal C\), let \(\widehat{\bh}_\ell^{\mathrm{Tgt}}\) denote the hidden states obtained by replaying the target model with the translated context cache, and let \(\bh_\ell^{\mathrm{Tgt}}\) denote the hidden states obtained when the target model natively prefills the same context. The ideal objective of the \textsf{Context Correction Loss} is to directly align the hidden-state trajectory after translation with the native trajectory:
\[
\mathcal L_{\mathrm{CC}}^{\mathrm{Src}\to\mathrm{Tgt}}(\mathcal C)
=
{\textstyle\sum_{\ell=T_{\mathrm{Start}}+1}^{L}}
\left\|
\widehat{\bh}_\ell^{\mathrm{Tgt}}
-
\bh_\ell^{\mathrm{Tgt}}
\right\|^2.
\]
Since the loss includes the final-layer term $\left\|\widehat{\bh}_L^{\mathrm{Tgt}}-\bh_L^{\mathrm{Tgt}}\right\|^2=\|\bs_L\|^2$, it directly reduces the last-state shift. Implementation details, including the KV-based heuristic, are provided in Appendix~\ref{app:context_correction_loss}.

\section{Experiments}
\label{sec:experiments}

We evaluate \mot on downstream QA benchmarks with a primary focus on
heterogeneous cache translation among \Llama, \Gemma, and \Qwen.
Section~\ref{sec:experimental_setup} summarizes the experimental setup, and
Section~\ref{sec:results} reports the main heterogeneous translation results.
Section~\ref{sec:root_cause_analysis} analyzes the effects of correction behavior,
loss design, and the \mot architecture.
Additional translation results, ablations, and
scalability analyses are provided in
Appendix~\ref{app:additional_results}.

\subsection{Experimental Setup}
\label{sec:experimental_setup}

\textbf{Benchmarks.} \quad
We evaluate two QA benchmark groups: \textbf{closed-set QA}, including BoolQ, PubMedQA, and MMLU-Redux, reported by average accuracy; and \textbf{extractive QA}, including SQuAD-v1.1 and NewsQA, reported by average F1. For case studies, we additionally use StrategyQA for multi-agent reasoning and HotpotQA-E for long-context CAG in Section~\ref{sec:case_studies}.

\textbf{Methods.} \quad
We compare \mot with representative internal-state communication methods.
\lsc~\citep{lsc} maps layer-wise KV caches into a global shared latent space
through model-specific adapters;
\ctoc~\citep{c2c} uses the projection-based C2C variant to directly map source
KV states into the target cache space;
\interlat~\citep{interlat} maps a sender's last-layer hidden-state sequence into
the receiver's embedding space;
and \kvcomm~\citep{kvcomm} selectively shares important KV layers between
structurally compatible models without training.

\textbf{Models.} \quad
Our main evaluation uses a heterogeneous model pool consisting of
\Llama~\textsf{(3B)}, \Gemma~\textsf{(4B)}, and \Qwen~\textsf{(4B)}.
These model families differ in depth, hidden dimension, attention configuration,
and KV-cache structure, resulting in substantially different cache spaces.
See Appendix~\ref{app:translator-training-configuration} for the shared
translator training configuration.

\subsection{Results}
\label{sec:results}

\begin{table*}[t!]
\caption{
Cache translation results across nine translations among \LlamaIcon~\Llama~\textsf{(3B)}, \GemmaIcon~\Gemma~\textsf{(4B)}, and \QwenIcon~\Qwen~\textsf{(4B)}.
\native denotes the unshaded target model without translation. Darker green indicates higher recovery relative to \native.
}
\label{tab:main_results}
\centering
\resizebox{\textwidth}{!}{%
\begin{tabular}{c|c|c|c|c c c c|c c c}
\toprule
\multirow[c]{2}{*}{\textbf{Source Model}} &
\multirow[c]{2}{*}{\textbf{Target Model}} &
\multirow[c]{2}{*}{\textbf{Method}} &
\multirow[c]{2}{*}{\shortstack{\textbf{Cosine}\\\textbf{Sim}}} &
\multicolumn{4}{c|}{\textbf{Closed-set QA (Accuracy)}} &
\multicolumn{3}{c}{\textbf{Extractive QA (F1)}} \\
\cmidrule(lr){5-8}
\cmidrule(lr){9-11}
&
&
&
&
\textbf{BoolQ} &
\textbf{PubMedQA} &
\textbf{MMLU-Redux} &
\textbf{Avg} &
\textbf{SQuAD} &
\textbf{NewsQA} &
\textbf{Avg} \\
\midrule
\multirow{15}{*}{\shortstack{\LlamaIcon \\ \Llama}}
& \multirow{5}{*}{\shortstack{\LlamaIcon \\ \Llama}}
& \nativecolor \native   & \nativecolor N/A  & \nativecolor 74.0 & \nativecolor 71.7 & \nativecolor 56.5 & \nativecolor 67.4\spm{4.9} & \nativecolor 0.41 & \nativecolor 0.48 & \nativecolor 0.44\spm{0.01} \\
&            & \transbg \lsc      & \transbg 0.01 & \transbg 0.0  & \transbg 0.0  & \transbg 0.0  & \transbg 0.0\spm{0.0}  & \transbg 0.00 & \transbg 0.01 & \transbg 0.00\spm{0.00} \\
&            & \transbg \ctoc     & \transbg 0.11 & \transbg 54.7 & \transbg \textbf{62.3} & \transbg 21.1 & \cellcolor[HTML]{B7D4DF}46.0\spm{1.1} & \transbg 0.08 & \transbg 0.02 & \cellcolor[HTML]{E4EBF1}0.05\spm{0.03} \\
&            & \transbg \interlat & \transbg 0.84 & \transbg 62.7 & \transbg 60.0 & \transbg 42.7 & \cellcolor[HTML]{7DB7C8}55.2\spm{2.3} & \transbg 0.21 & \transbg 0.13 & \cellcolor[HTML]{AFD0DB}0.17\spm{0.03} \\
&            & \transbg \textbf{\mot}
& \transbg \textbf{1.00} & \transbg \textbf{73.7} & \transbg 70.0 & \transbg \textbf{56.2} & \cellcolor[HTML]{3491AB}\textbf{66.6}\spm{6.2} & \transbg \textbf{0.51} & \transbg \textbf{0.54} & \cellcolor[HTML]{077B99}\textbf{0.53}\spm{0.02} \\
\cmidrule(lr){2-11}
& \multirow{5}{*}{\shortstack{\GemmaIcon \\ \Gemma}}
& \nativecolor \native   & \nativecolor N/A  & \nativecolor 64.3 & \nativecolor 42.3 & \nativecolor 23.5 & \nativecolor 43.4\spm{2.5} & \nativecolor 0.49 & \nativecolor 0.35 & \nativecolor 0.42\spm{0.04} \\
&            & \transbg \lsc      & \transbg 0.03 & \transbg 29.0 & \transbg 12.7 & \transbg 0.4  & \transbg 14.0\spm{1.3} & \transbg 0.09 & \transbg 0.01 & \transbg 0.05\spm{0.03} \\
&            & \transbg \ctoc     & \transbg 0.77 & \transbg \textbf{58.7} & \transbg \textbf{39.3} & \transbg 0.2  & \cellcolor[HTML]{81B9C9}32.8\spm{2.1} & \transbg 0.17 & \transbg 0.05 & \cellcolor[HTML]{DFE8EE}0.11\spm{0.04} \\
&            & \transbg \interlat & \transbg 0.24 & \transbg 35.7 & \transbg 23.3 & \transbg 2.4  & \cellcolor[HTML]{E0E9EF}20.5\spm{3.0} & \transbg 0.17 & \transbg 0.02 & \cellcolor[HTML]{E3EBF0}0.10\spm{0.04} \\
&            & \transbg \textbf{\mot}
& \transbg \textbf{0.99} & \transbg \textbf{63.3} & \transbg 43.7 & \transbg \textbf{22.9} & \cellcolor[HTML]{308FA9}\textbf{43.3}\spm{3.2} & \transbg \textbf{0.43} & \transbg \textbf{0.32} & \cellcolor[HTML]{479BB2}\textbf{0.38}\spm{0.03} \\
\cmidrule(lr){2-11}
& \multirow{5}{*}{\shortstack{\QwenIcon \\ \Qwen}}
& \nativecolor \native   & \nativecolor N/A  & \nativecolor 83.0 & \nativecolor 59.3 & \nativecolor 72.4 & \nativecolor 71.6\spm{1.2} & \nativecolor 0.26 & \nativecolor 0.21 & \nativecolor 0.24\spm{0.03} \\
&            & \transbg \lsc      & \transbg -0.01 & \transbg 48.7 & \transbg 23.0 & \transbg 36.9 & \transbg 36.2\spm{1.9} & \transbg 0.15 & \transbg 0.00 & \cellcolor[HTML]{E5EBF1}0.08\spm{0.01} \\
&            & \transbg \ctoc     & \transbg 0.74 & \transbg 51.3 & \transbg 47.7 & \transbg 27.2 & \cellcolor[HTML]{E2EAF0}42.1\spm{1.9} & \transbg 0.09 & \transbg 0.03 & \transbg 0.06\spm{0.01} \\
&            & \transbg \interlat & \transbg 0.64 & \transbg 49.7 & \transbg 44.0 & \transbg 41.0 & \cellcolor[HTML]{D1E1E9}44.9\spm{3.1} & \transbg 0.20 & \transbg 0.04 & \cellcolor[HTML]{AFD0DC}0.12\spm{0.01} \\
&            & \transbg \textbf{\mot}
& \transbg \textbf{0.99} & \transbg \textbf{86.0} & \transbg \textbf{56.0} & \transbg \textbf{71.7} & \cellcolor[HTML]{3190AA}\textbf{71.2}\spm{1.1} & \transbg \textbf{0.40} & \transbg \textbf{0.25} & \cellcolor[HTML]{005D82}\textbf{0.32}\spm{0.03} \\
\midrule
\multirow{15}{*}{\shortstack{\GemmaIcon \\ \Gemma}}
& \multirow{5}{*}{\shortstack{\LlamaIcon \\ \Llama}}
& \nativecolor \native   & \nativecolor N/A  & \nativecolor 74.0 & \nativecolor 71.7 & \nativecolor 56.5 & \nativecolor 67.4\spm{4.9} & \nativecolor 0.41 & \nativecolor 0.48 & \nativecolor 0.44\spm{0.01} \\
&            & \transbg \lsc      & \transbg 0.00 & \transbg 0.0  & \transbg 0.0  & \transbg 0.0  & \transbg 0.0\spm{0.0}  & \transbg 0.01 & \transbg 0.01 & \transbg 0.01\spm{0.00} \\
&            & \transbg \ctoc     & \transbg 0.09 & \transbg 51.0 & \transbg 44.7 & \transbg 24.5 & \cellcolor[HTML]{C4DBE4}40.1\spm{1.1} & \transbg 0.15 & \transbg 0.01 & \cellcolor[HTML]{E2EAF0}0.08\spm{0.07} \\
&            & \transbg \interlat & \transbg 0.57 & \transbg 55.7 & \transbg \textbf{59.0} & \transbg 29.2 & \cellcolor[HTML]{99C5D3}48.0\spm{1.4} & \transbg 0.23 & \transbg 0.05 & \cellcolor[HTML]{C6DCE4}0.14\spm{0.07} \\
&            & \transbg \textbf{\mot}
& \transbg \textbf{0.96} & \transbg \textbf{72.3} & \transbg 44.3 & \transbg \textbf{55.0} & \cellcolor[HTML]{66ABBF}\textbf{57.2}\spm{3.6} & \transbg \textbf{0.52} & \transbg \textbf{0.44} & \cellcolor[HTML]{1B85A1}\textbf{0.48}\spm{0.05} \\
\cmidrule(lr){2-11}
& \multirow{5}{*}{\shortstack{\GemmaIcon \\ \Gemma}}
& \nativecolor \native   & \nativecolor N/A  & \nativecolor 64.3 & \nativecolor 42.3 & \nativecolor 23.5 & \nativecolor 43.4\spm{2.5} & \nativecolor 0.49 & \nativecolor 0.35 & \nativecolor 0.42\spm{0.04} \\
&            & \transbg \lsc      & \transbg 0.04 & \transbg 43.0 & \transbg 3.3  & \transbg 1.6  & \transbg 16.0\spm{0.8} & \transbg 0.12 & \transbg 0.01 & \cellcolor[HTML]{E7EDF2}0.07\spm{0.02} \\
&            & \transbg \ctoc     & \transbg 0.02 & \transbg 19.3 & \transbg \textbf{32.7} & \transbg 0.0  & \cellcolor[HTML]{E6ECF1}17.3\spm{1.2} & \transbg 0.09 & \transbg 0.03 & \transbg 0.06\spm{0.02} \\
&            & \transbg \interlat & \transbg 0.21 & \transbg 37.0 & \transbg 17.3 & \transbg \textbf{20.4} & \cellcolor[HTML]{B1D1DC}24.9\spm{1.4} & \transbg 0.24 & \transbg 0.04 & \cellcolor[HTML]{C1D9E3}0.14\spm{0.06} \\
&            & \transbg \textbf{\mot}
& \transbg \textbf{0.99} & \transbg \textbf{65.7} & \transbg 38.0 & \transbg 10.8 & \cellcolor[HTML]{54A2B8}\textbf{38.1}\spm{1.4} & \transbg \textbf{0.52} & \transbg \textbf{0.36} & \cellcolor[HTML]{258AA5}\textbf{0.44}\spm{0.03} \\
\cmidrule(lr){2-11}
& \multirow{5}{*}{\shortstack{\QwenIcon \\ \Qwen}}
& \nativecolor \native   & \nativecolor N/A  & \nativecolor 83.0 & \nativecolor 59.3 & \nativecolor 72.4 & \nativecolor 71.6\spm{1.2} & \nativecolor 0.26 & \nativecolor 0.21 & \nativecolor 0.24\spm{0.03} \\
&            & \transbg \lsc      & \transbg 0.01 & \transbg 0.0  & \transbg 0.0  & \transbg 0.0  & \transbg 0.0\spm{0.0}  & \transbg 0.00 & \transbg 0.01 & \transbg 0.00\spm{0.00} \\
&            & \transbg \ctoc     & \transbg 0.67 & \transbg 60.7 & \transbg 49.0 & \transbg 40.3 & \cellcolor[HTML]{9FC8D5}50.0\spm{2.0} & \transbg 0.17 & \transbg 0.05 & \cellcolor[HTML]{D1E1E9}0.11\spm{0.01} \\
&            & \transbg \interlat & \transbg 0.67 & \transbg 51.7 & \transbg 36.7 & \transbg 40.4 & \cellcolor[HTML]{C3DAE3}42.9\spm{1.6} & \transbg 0.18 & \transbg 0.04 & \cellcolor[HTML]{CDDFE7}0.11\spm{0.02} \\
&            & \transbg \textbf{\mot}
& \transbg \textbf{0.97} & \transbg \textbf{83.7} & \transbg \textbf{45.0} & \transbg \textbf{71.2} & \cellcolor[HTML]{499CB3}\textbf{66.6}\spm{1.4} & \transbg \textbf{0.54} & \transbg \textbf{0.27} & \cellcolor[HTML]{001E51}\textbf{0.41}\spm{0.03} \\
\midrule
\multirow{15}{*}{\shortstack{\QwenIcon \\ \Qwen}}
& \multirow{5}{*}{\shortstack{\LlamaIcon \\ \Llama}}
& \nativecolor \native   & \nativecolor N/A  & \nativecolor 74.0 & \nativecolor 71.7 & \nativecolor 56.5 & \nativecolor 67.4\spm{4.9} & \nativecolor 0.41 & \nativecolor 0.48 & \nativecolor 0.44\spm{0.01} \\
&            & \transbg \lsc      & \transbg 0.01 & \transbg 0.0  & \transbg 0.0  & \transbg 0.0  & \transbg 0.0\spm{0.0}  & \transbg 0.00 & \transbg 0.01 & \transbg 0.01\spm{0.00} \\
&            & \transbg \ctoc     & \transbg 0.09 & \transbg 48.7 & \transbg 56.0 & \transbg 27.9 & \cellcolor[HTML]{BBD6E0}44.2\spm{2.3} & \transbg 0.13 & \transbg 0.02 & \cellcolor[HTML]{E3EAF0}0.07\spm{0.06} \\
&            & \transbg \interlat & \transbg 0.58 & \transbg 57.0 & \transbg \textbf{60.0} & \transbg 31.4 & \cellcolor[HTML]{9BC6D4}49.5\spm{2.7} & \transbg 0.22 & \transbg 0.05 & \cellcolor[HTML]{C4DBE4}0.13\spm{0.06} \\
&            & \transbg \textbf{\mot}
& \transbg \textbf{1.00} & \transbg \textbf{74.3} & \transbg \textbf{68.7} & \transbg \textbf{55.8} & \cellcolor[HTML]{3692AC}\textbf{66.2}\spm{6.2} & \transbg \textbf{0.57} & \transbg \textbf{0.56} & \cellcolor[HTML]{007091}\textbf{0.57}\spm{0.02} \\
\cmidrule(lr){2-11}
& \multirow{5}{*}{\shortstack{\GemmaIcon \\ \Gemma}}
& \nativecolor \native   & \nativecolor N/A  & \nativecolor 64.3 & \nativecolor 42.3 & \nativecolor 23.5 & \nativecolor 43.4\spm{2.5} & \nativecolor 0.49 & \nativecolor 0.35 & \nativecolor 0.42\spm{0.04} \\
&            & \transbg \lsc      & \transbg 0.03 & \transbg 42.7 & \transbg 10.3 & \transbg 0.2  & \cellcolor[HTML]{E1EAEF}17.7\spm{2.2} & \transbg 0.11 & \transbg 0.01 & \transbg 0.06\spm{0.04} \\
&            & \transbg \ctoc     & \transbg 0.01 & \transbg 21.3 & \transbg 14.7 & \transbg 0.0  & \transbg 12.0\spm{0.6} & \transbg 0.10 & \transbg 0.02 & \transbg 0.06\spm{0.02} \\
&            & \transbg \interlat & \transbg 0.21 & \transbg 43.3 & \transbg 25.3 & \transbg 17.9 & \cellcolor[HTML]{94C2D1}28.8\spm{1.8} & \transbg 0.20 & \transbg 0.03 & \cellcolor[HTML]{CCDFE7}0.12\spm{0.06} \\
&            & \transbg \textbf{\mot}
& \transbg \textbf{0.99} & \transbg \textbf{60.0} & \transbg \textbf{44.0} & \transbg \textbf{18.9} & \cellcolor[HTML]{4097B0}\textbf{41.0}\spm{3.5} & \transbg \textbf{0.45} & \transbg \textbf{0.32} & \cellcolor[HTML]{4399B1}\textbf{0.38}\spm{0.03} \\
\cmidrule(lr){2-11}
& \multirow{5}{*}{\shortstack{\QwenIcon \\ \Qwen}}
& \nativecolor \native   & \nativecolor N/A  & \nativecolor 83.0 & \nativecolor 59.3 & \nativecolor 72.4 & \nativecolor 71.6\spm{1.2} & \nativecolor 0.26 & \nativecolor 0.21 & \nativecolor 0.24\spm{0.03} \\
&            & \transbg \lsc      & \transbg 0.00 & \transbg 0.0  & \transbg 0.0  & \transbg 0.0  & \transbg 0.0\spm{0.0}  & \transbg 0.00 & \transbg 0.00 & \transbg 0.00\spm{0.00} \\
&            & \transbg \ctoc     & \transbg 0.01 & \transbg 61.3 & \transbg 58.0 & \transbg 40.9 & \cellcolor[HTML]{8DBFCE}53.4\spm{2.9} & \transbg \textbf{0.22} & \transbg 0.05 & \cellcolor[HTML]{99C5D3}0.13\spm{0.02} \\
&            & \transbg \interlat & \transbg 0.34 & \transbg 44.0 & \transbg 43.7 & \transbg 41.7 & \cellcolor[HTML]{C3DAE3}43.1\spm{2.0} & \transbg 0.11 & \transbg 0.03 & \cellcolor[HTML]{DCE7ED}0.07\spm{0.01} \\
&            & \transbg \textbf{\mot}
& \transbg \textbf{0.99} & \transbg \textbf{83.3} & \transbg \textbf{59.3} & \transbg \textbf{72.4} & \cellcolor[HTML]{2E8FA9}\textbf{71.7}\spm{1.5} & \transbg 0.24 & \transbg \textbf{0.18} & \cellcolor[HTML]{499CB3}\textbf{0.21}\spm{0.03} \\
\bottomrule
\end{tabular}%
}
\end{table*}

\paragraph{Main Results.}
Table~\ref{tab:main_results} shows that \mot consistently achieves the strongest
cache-translation performance. Across heterogeneous translations, \mot obtains
an average cosine similarity of \(0.98\), accuracy of \(57.6\%\), and F1 of
\(0.42\). These scores recover \(95.0\%\) and \(120.6\%\) of \native accuracy
and F1 and reach \(1.5\times\) and \(3.5\times\) the strongest-baseline scores,
respectively; the strongest baseline, \interlat, achieves \(39.1\%\) accuracy
and \(0.12\) F1. The same advantage is consistently observed for homogeneous
translation. The lower baseline performance may partly reflect their respective
design constraints. \lsc requires source-to-shared and shared-to-target mappings
through a common translation space, which can make heterogeneous translation
more difficult. \ctoc relies on a direct projection, potentially limiting its
capacity to bridge substantial cache-space mismatch. \interlat communicates only
last-layer hidden states rather than the full token-wise KV cache, which may
limit the information available for reconstructing the target cache.

The recovery distributions are also strongly separated: the lower quartile
(Q1) of \mot reaches \(93.4\%\) for accuracy and \(95.1\%\) for F1, substantially
exceeding the upper quartile (Q3) of all baseline recoveries at \(66.2\%\) and
\(31.3\%\), respectively. This gap indicates that \mot's advantage is not driven
by a small number of favorable translations, but is consistently maintained
across model pairs and tasks. The same separation is further reflected across
MMLU-Redux categories in Figure~\ref{fig:mmlu_cross_family}.

\subsection{Root Cause Analysis}
\label{sec:root_cause_analysis}

\paragraph{Shifts and Correction.}

\begin{figure*}[t!]
    \centering

    \begin{minipage}[t]{0.3\textwidth}
        \centering
        \vspace{0pt}
        \includegraphics[width=\textwidth,height=3.5cm]{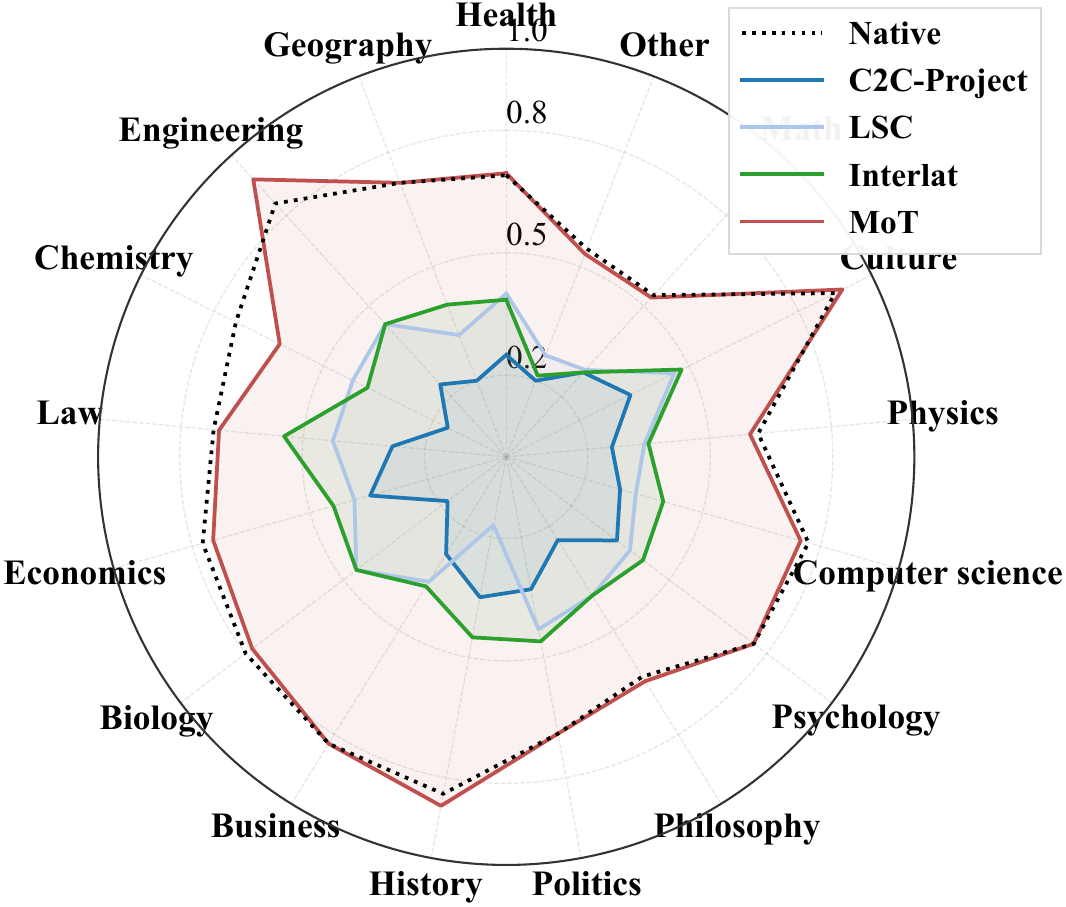}
        \vspace{0pt}
        \captionof{figure}{MMLU-Redux.}
        \label{fig:mmlu_cross_family}
    \end{minipage}
    \hfill
    \begin{minipage}[t]{0.34\textwidth}
        \centering
        \vspace{0pt}
        \includegraphics[width=\textwidth,height=3.5cm]{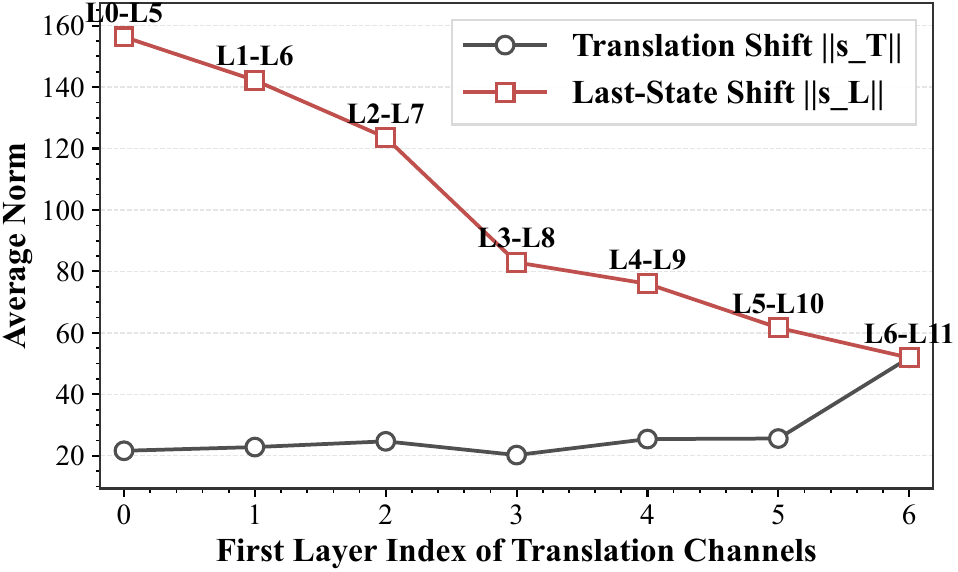}
        \vspace{0pt}
        \captionof{figure}{Improved shifts.}
        \label{fig:improved_shifts}
    \end{minipage}
    \hfill
    \begin{minipage}[t]{0.34\textwidth}
        \centering
        \vspace{0pt}
        \includegraphics[width=\textwidth,height=3.5cm]{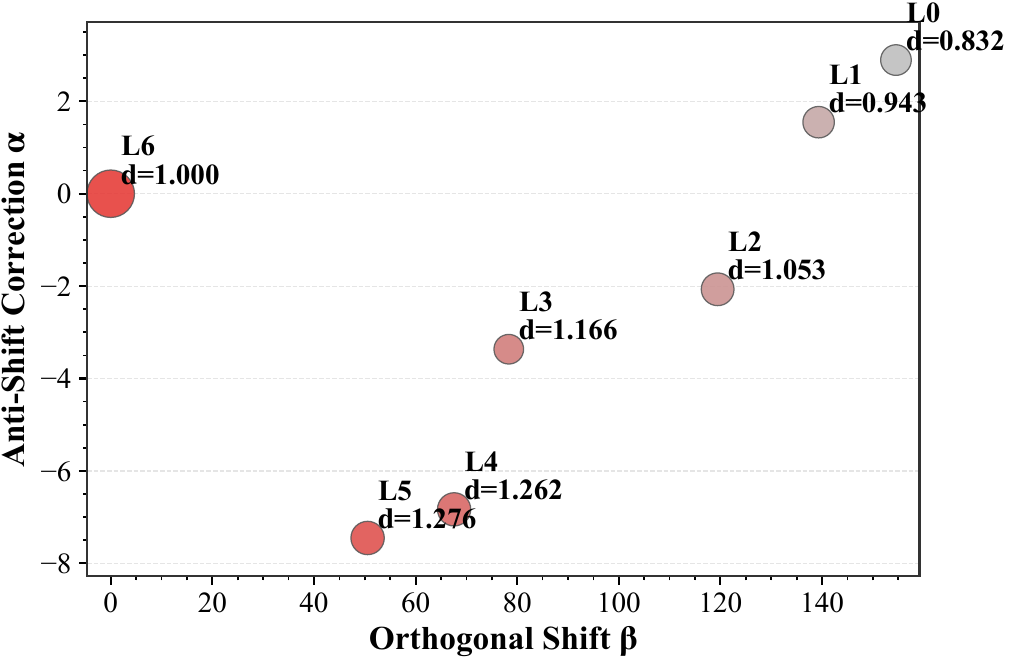}
        \vspace{0pt}
        \captionof{figure}{Improved correction.}
        \label{fig:improved_correction}
    \end{minipage}
\end{figure*}

Figure~\ref{fig:improved_shifts} shows that, compared with the \single \& \textsf{Prompt LM Loss} setting in Figure~\ref{fig:last_state_shift}, the proposed \mot \& \textsf{Context Correction Loss} sharply reduces both the translation shift \(\|\bs_T\|\) and the last-state shift \(\|\bs_L\|\) across translation-layer positions. This shows that the two components suppress their respective target error dynamics.
We further compare correction trends between Figure~\ref{fig:correction_decomposition} and Figure~\ref{fig:improved_correction}. Overall, \(\alpha_{T:L}\) slightly decreases and rises only at a few later positions, while \(\beta_{T:L}\) drops substantially from early translation layers and is further reduced in later layers. Although \(d_{T:L}\) slightly increases, its magnitude remains limited. Thus, from Eq.~\eqref{eq:last_state_shift_decomposition}, the reduction in \(\|\bs_L\|\) is better attributed to the joint decrease of \(\|\bs_T\|\) and \(\beta_{T:L}\), rather than to a substantial improvement in correction deficit. Additional results are provided in Appendix~\ref{app:root_cause_analysis}.

\paragraph{Loss Ablation.}

\begin{figure*}[t]
    \centering

    \begin{minipage}{0.42\textwidth}
        \centering
        \includegraphics[width=\textwidth]{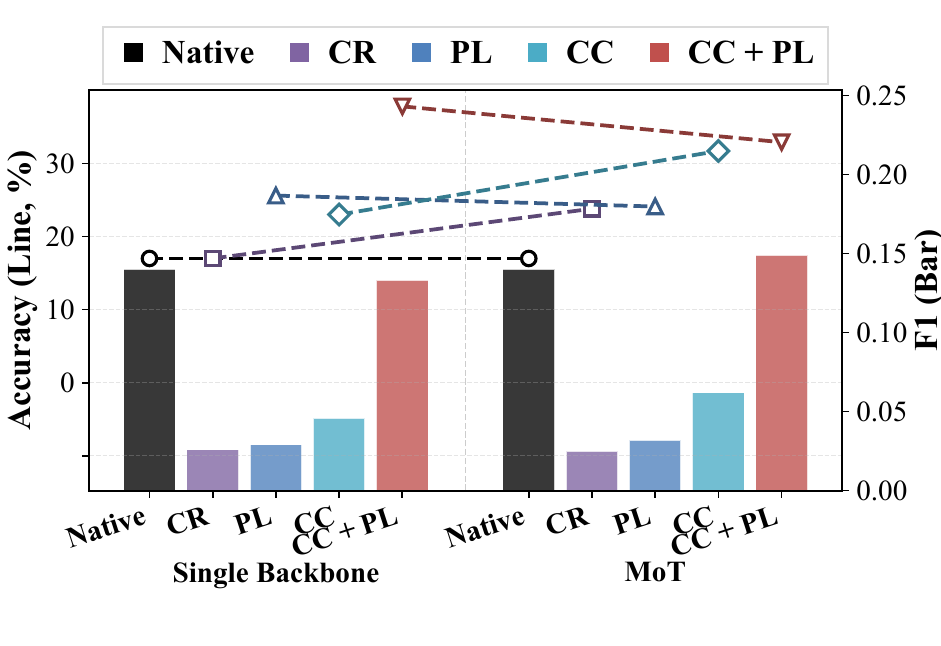}
        \captionof{figure}{Loss ablation.}
        \label{fig:ablation_loss}
    \end{minipage}
    \qquad
    \begin{minipage}{0.52\textwidth}
        \vspace{0.7cm}
        \centering
        \footnotesize
        \setlength{\tabcolsep}{4pt}
        \begin{tabular}{ccccc}
        \toprule
        \multirow{2}{*}{\textbf{Method}} &
        \multirow{2}{*}{\textbf{Acc}} &
        \multirow{2}{*}{\textbf{F1}} &
        \multicolumn{2}{c}{\textbf{Translator Usage}} \\
        \cmidrule(lr){4-5}
        &
        &
        &
        \(\mathbf{CV^2}\) &
        \textbf{Entropy} \\
        \midrule
        \textsf{MoE}     & 48.1 & 0.41 & 0.82 & 0.15 \\
        \textsf{MoT-Uni} & 49.0 & 0.41 & 0.00 & 0.69 \\
        \mot             & 51.0 & 0.43 & 0.32 & 0.51 \\
        \bottomrule
        \end{tabular}
        \vspace{0.7cm}
        \captionof{table}{\mot architecture ablation.}
        \label{tab:mot_analysis}
    \end{minipage}

\end{figure*}

Figure~\ref{fig:ablation_loss} shows that the proposed \texttt{CC} (\textsf{Context Correction}) loss remains effective independently of \mot. Even in the \single setting, without routing, \texttt{CC} yields more stable accuracy and F1 gains than the prior \texttt{CR} (Context Reconstruction) and \texttt{PL} (\textsf{Prompt LM}) losses~\citep{lsc}. This suggests that \texttt{CC} directly reduces the last-state shift by aligning the translated cache with the native target trajectory, rather than relying on translator routing. The same trend holds under \mot, supporting our analysis that reducing the translation shift with \mot and the last-state shift with \texttt{CC} are complementary.
Individually, \texttt{CR} and \texttt{PL} provide limited recovery, whereas \texttt{CC} gives stronger gains. Combining \texttt{CC} with \texttt{PL} achieves the best accuracy and F1 by coupling context-level KV correction with prompt-conditioned generation supervision, so we adopt \texttt{CC}+\texttt{PL} as the final heuristic loss.

\paragraph{\mot Architecture Ablation.}

For this controlled architecture ablation, we use
\qwen~\textsf{(7B)}$\to$\qwen~\textsf{(0.5B)} translation and compare
translator selection in \mot under the simplest configuration,
\((\mathrm{Top}\text{-}K,N_{\mathrm{Tr}})=(1,2)\).
As baselines, we use \textsf{MoE} (Mixture-of-Experts), which splits the internal Dense MLP into experts, and \textsf{MoT-Uni}, which always uses translators uniformly.
The results are summarized in Table~\ref{tab:mot_analysis}.
Usage Imbalance is \(\mathrm{CV}^2=\mathrm{Var}_e(I_e)/\mathrm{Mean}_e(I_e)^2\),
where \(I_e\) is the cumulative selection count of translator \(e\), and Usage Entropy is
\(-\sum_e p_e\log p_e\) with \(p_e=I_e/\sum_j I_j\). Higher entropy indicates more balanced translator usage.
Table~\ref{tab:mot_analysis} shows that \textsf{MoE} usage is highly skewed, while \textsf{MoT-Uni} is balanced but lacks input-specific selectivity. In contrast, \mot uses both translators selectively and achieves higher Acc and F1, supporting that token-wise translator selection helps reduce the translation shift \(\bs_T\).

\section{Case Studies}
\label{sec:case_studies}

\paragraph{Memory Reuse: Multi-Agent Reasoning.}
\label{sec:case_memory_reuse}

\begin{figure*}[t!]
    \centering

    \begin{minipage}{0.49\textwidth}
        \centering
        \begin{minipage}{0.48\textwidth}
            \centering
            \includegraphics[width=\textwidth,height=3cm]{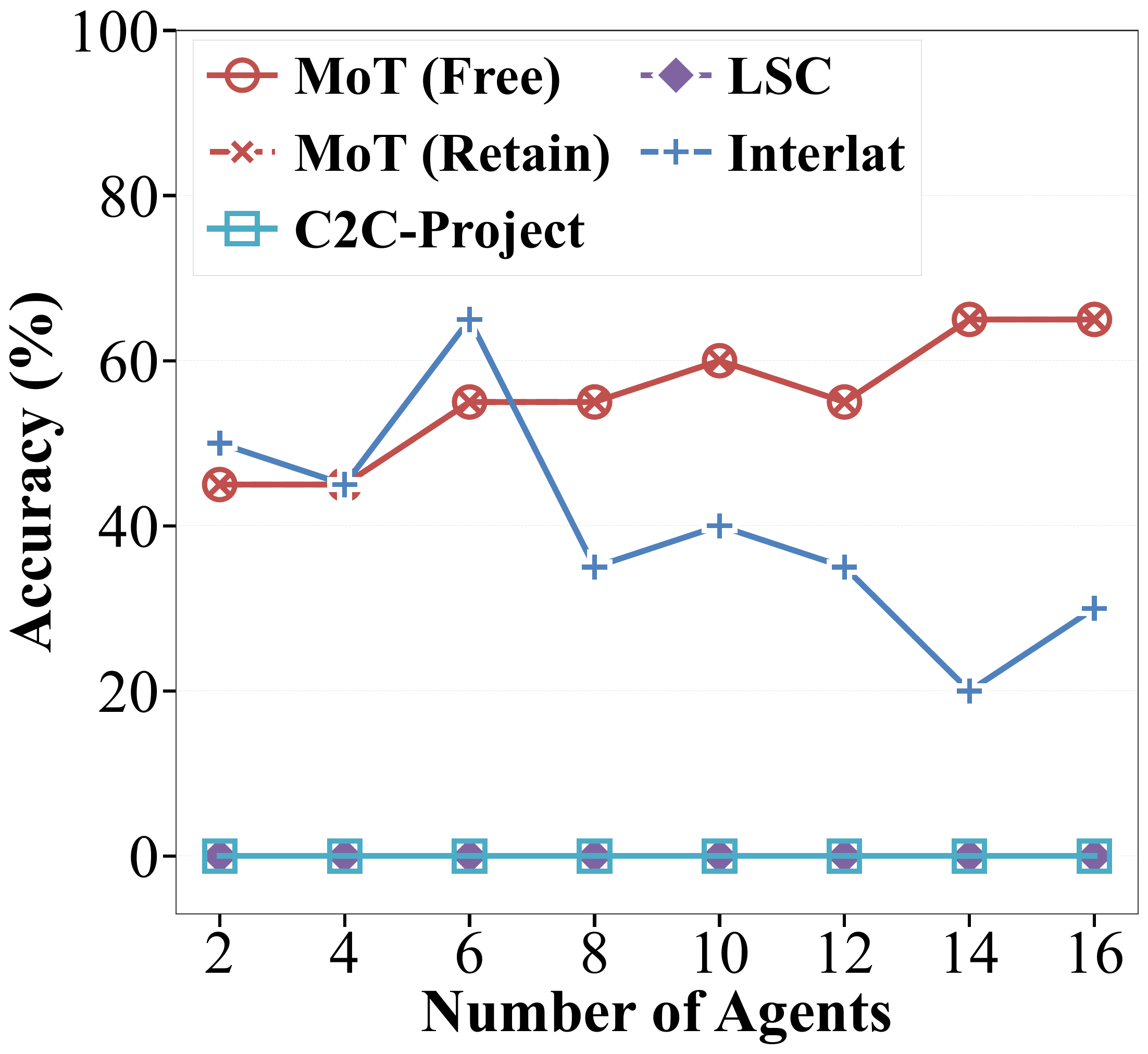}\\
            \small (a) Accuracy
        \end{minipage}
        \hfill
        \begin{minipage}{0.48\textwidth}
            \centering
            \includegraphics[width=\textwidth,height=3cm]{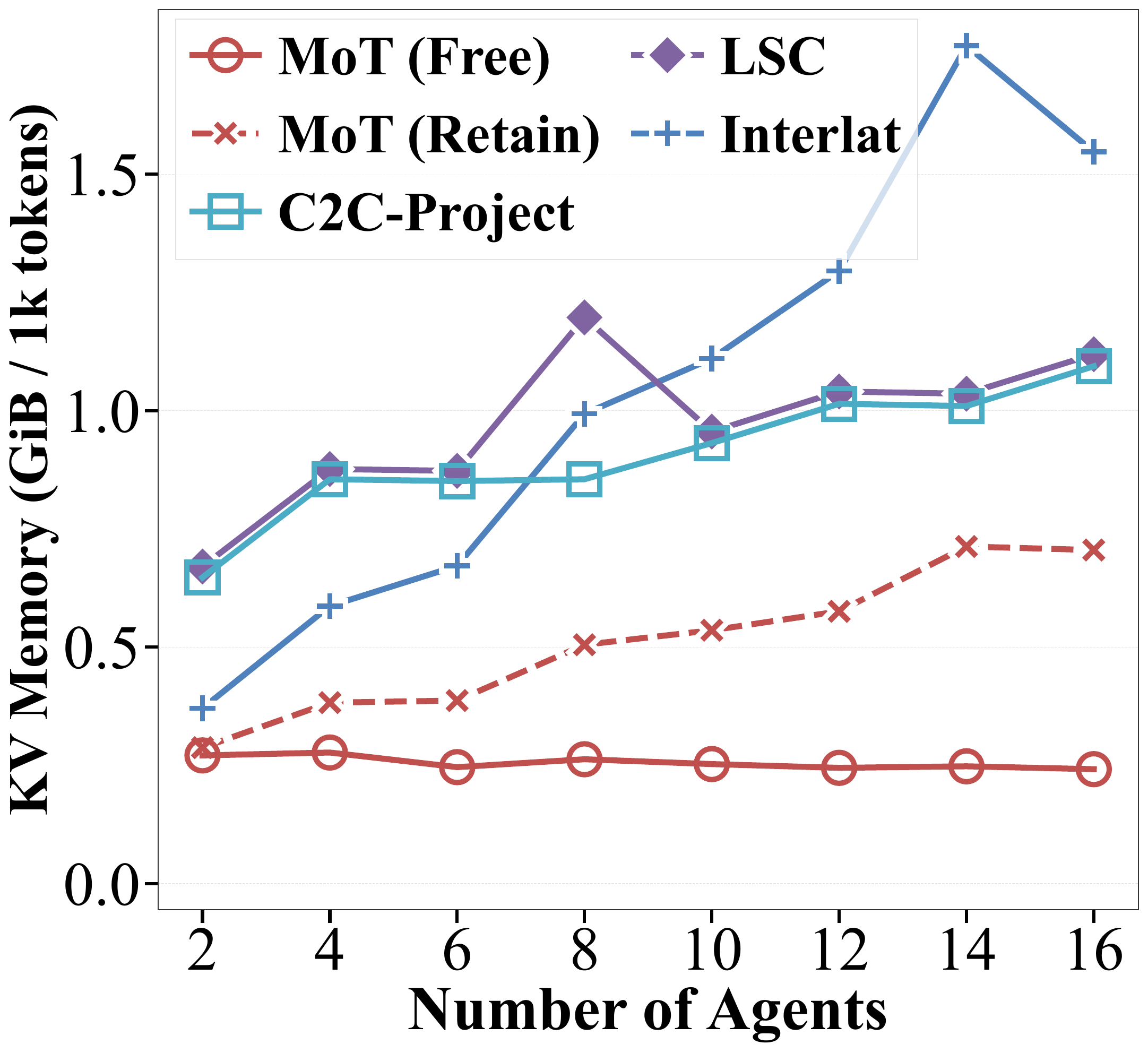}\\
            \small (b) KV Memory
        \end{minipage}
        \caption{Agent scalability.}
        \label{fig:agents}
    \end{minipage}
    \hfill
    \begin{minipage}{0.49\textwidth}
        \centering
        \begin{minipage}[t]{0.48\textwidth}
            \centering
            \includegraphics[width=\textwidth,height=3cm]{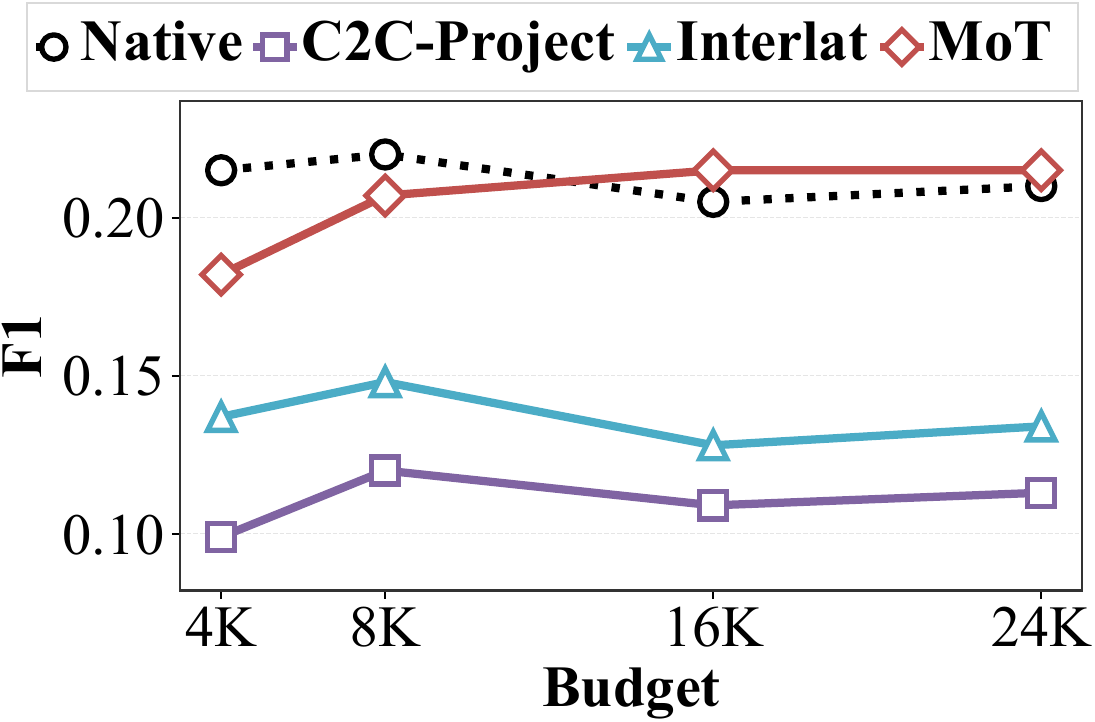}
            \subcaption{F1 by Budget}
            \label{fig:f1_by_budget}
        \end{minipage}
        \hfill
        \begin{minipage}[t]{0.48\textwidth}
            \centering
            \includegraphics[width=\textwidth,height=3cm]{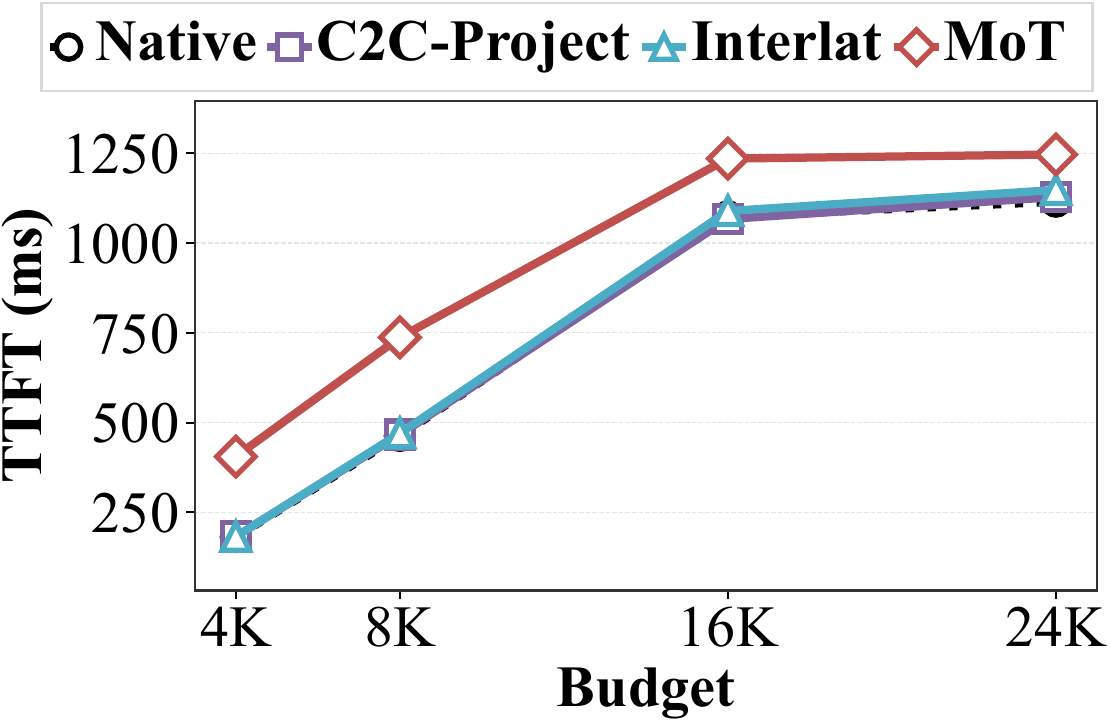}
            \subcaption{TTFT by Budget}
            \label{fig:ttft_by_budget}
        \end{minipage}
        \caption{Long-context scalability.}
        \label{fig:budget_f1_ttft_analysis}
    \end{minipage}

\end{figure*}

For multi-agent reasoning, we evaluate StrategyQA using a fixed MALLM
configuration with \textsf{Expert} persona generation, \textsf{Simple} response
generation, the \textsf{Memory} discussion paradigm, and
\textsf{Supermajority Consensus}~\citep{becker2025mallm}. We define the first
agent as the \textsf{Hub Agent} and compare two cache-management modes:
\textsf{Retain}, which keeps all resident agent caches, and \textsf{Free},
which retains the Hub Agent cache while releasing non-hub caches after offload to support scalable memory reuse.

We consider two complementary settings: heterogeneous 10-agent reasoning as a
practical cross-model case study, and homogeneous agent-count scaling to isolate memory scalability. Figure~\ref{fig:performance_landscape} shows the
heterogeneous 10-agent result, where \mot exhibits a clear quality--efficiency
advantage by preserving accurate cross-model discussion while substantially
reducing KV-memory usage.

Figure~\ref{fig:agents} evaluates homogeneous scalability from 2 to 16 agents.
Because unsuccessful discussions can terminate early and produce shorter
histories, we normalize KV usage by the tokens generated during successfully
verified turns, preventing early termination from appearing artificially
memory-efficient. As the number of agents increases, \mot shows an overall
improvement in reasoning accuracy while \mot(\textsf{Free}) keeps normalized
KV memory nearly constant. In contrast, \textsf{Retain} and the baselines
accumulate memory as the interaction grows. Together, these results demonstrate
\textbf{scale-invariant KV-memory behavior}: cache translation preserves useful
discussion history while cache eviction suppresses redundant model-local state.

Detailed experimental configurations, cache-management procedures,
normalization, training costs, and memory breakdowns are provided in
Appendix~\ref{sec:multi_agent_reasoning}.

\paragraph{Storage Reuse: Long-Context CAG.}
\label{sec:case_storage_reuse}
This case study evaluates long-context CAG\,(Cache-Augmented Generation), where long contexts are stored as KV caches and reused. This setting tests whether translated caches can preserve useful long-context information as the reusable context grows. On \texttt{HotpotQA-E}, we vary the context budget over 4K, 8K, 16K, and 24K. \native represents the RAG setting that prefills the original context for every query, while the other methods reuse stored caches in the CAG setting.
Figure~\ref{fig:budget_f1_ttft_analysis}(a) shows that \ctoc and \interlat remain in a low-F1 regime even as the context budget increases. In contrast, \mot improves substantially with larger budgets, especially at 16K and 24K, indicating that it can effectively reuse long-context information stored in the KV cache. Figure~\ref{fig:budget_f1_ttft_analysis}(b) shows that TTFT\,(Time to First Token) increases with context budget. \native incurs larger prefill cost, while CAG methods incur cache conversion and injection overhead. \mot shows only marginally higher TTFT due to the additional use of MoT.
Overall, \mot achieves F1 close to \native under large budgets, showing stronger cache-reuse quality for long-context CAG. Details are provided in Appendix~\ref{app:long-context}.

\section{Conclusion}
We introduced \textsf{MoT (Mixture-of-Translators)}, a cache translation
framework for reusing KV caches across heterogeneous LLMs. MoT combines multiple
translator modules with token-level routing and a Context Correction Loss that
aligns the replayed target trajectory with the native target trajectory. Our
analysis identified two competing error sources, propagation error and
correction-deficit error, which motivate this joint design.
Experiments on closed-set QA and extractive QA demonstrate that MoT preserves
downstream quality under heterogeneous cache translation among
\Llama, \Gemma, and \Qwen, achieving \(57.6\%\) average accuracy and \(0.42\)
average F1. These scores recover \(95.0\%\) and \(120.6\%\) of native target
performance and reach \(1.5\times\) and \(3.5\times\) the strongest-baseline
scores, respectively. The practical case studies further show that translated
caches enable accurate heterogeneous-agent discussion with scale-invariant
KV-memory behavior as agent interactions grow, while preserving near-native
quality for long-context CAG. Together, these results establish heterogeneous
KV-cache translation as a practical foundation for scalable multi-model LLM
systems.

\subsection*{Reproducibility statement}
We provide the assumptions and complete proofs for the theoretical results in
Appendix Section~\ref{app:theoretical_analysis}. Detailed descriptions of the
translator architecture, Context Correction Loss, channel mapping, cache replay,
and the hidden-state translation variant are provided in Appendix
Section~\ref{app:proposed_method_details}. Benchmark prompt formats are specified in Appendix
Section~\ref{app:benchmark-prompts}. Training configurations,
additional benchmark results, ablations, and scalability analyses are reported
in Appendix Section~\ref{app:additional_results}, and the complete experimental
configurations for the multi-agent reasoning and long-context CAG case studies
are provided in Appendix Section~\ref{app:details_case_studies}. Dataset,
implementation, and model-license information is documented in Appendix
Section~\ref{app:license}.


\subsection*{AI use statement}
In this work, we used generative AI tools to aid and polish writing; retrieve
and discover relevant literature; support research ideation and execution; draft
portions of the paper; and assist with method implementation and software code.
We did not use generative AI tools to generate synthetic datasets or to prove
mathematical claims.

All AI-assisted outputs were reviewed by the authors. In particular, suggested
literature was checked against the cited sources, research suggestions were
independently evaluated before use, and AI-assisted implementation and code were
inspected and tested by the authors. Mathematical claims and proofs were
independently checked by the authors. The authors take responsibility for the
final content of this work, including its text, claims, code, and other artifacts
produced with the aid of generative AI.

\bibliography{references}
\bibliographystyle{iclr2027_conference}

\newpage

\appendix

\section*{Appendix}

\definecolor{appendixback}{HTML}{E8EDF2}
\definecolor{appendixframe}{HTML}{2F8FA9}

\begin{tcolorbox}[
  enhanced,
  breakable,
  colback=appendixback,
  colframe=appendixframe,
  arc=3pt,
  boxrule=0.6pt,
  title={\small\textbf{Appendix Contents}},
  fonttitle=\small,
  left=6pt,
  right=6pt,
  top=4pt,
  bottom=5pt
]
\footnotesize
\setlength{\parindent}{0pt}%
\setlength{\parfillskip}{0pt}%
\setlength{\rightskip}{0pt plus 1fil}%
\setlength{\parskip}{2pt}%

\newcommand{\apptocentry}[3]{%
  \vspace{1pt}\hangindent=2.8em\hangafter=1%
  \hyperref[#3]{\makebox[2.8em][l]{\textbf{#1}}\textbf{#2}}%
  \dotfill\textbf{\pageref{#3}}\par
}

\apptocentry{A}{Related Work}{app:related-work-page}
\apptocentry{B}{Theoretical Analysis}{app:theoretical_analysis}
\apptocentry{C}{Proposed Method Details}{app:proposed_method_details}
\apptocentry{D}{Benchmark Prompts}{app:benchmark-prompts}
\apptocentry{E}{Additional Results}{app:additional_results}
\apptocentry{F}{Details of Case Studies}{app:details_case_studies}
\apptocentry{G}{License}{app:license}
\apptocentry{H}{Limitations and Broader Impacts}{app:limitations_broader_impacts}
\end{tcolorbox}

\phantomsection
\label{app:related-work-page}

\section{Related Work}
\subsection{Communication between Models}

Model communication predates LLMs and spans symbolic messages, parameter or model-update sharing, logit distillation, and learned vector protocols
\citep{finin1994kqml,mcmahan2017communication,jeong2018federated,li2019fedmd,foerster2016learning,lazaridou2020emergent}.
These paradigms primarily target training-time aggregation, output-level transfer, or task-specific protocols rather than transmitting context-conditioned internal states.
LLM systems instead commonly communicate through text via role-based collaboration, mixture-of-agents, debate, routing, and collaborative decoding
\citep{li2023camel,wu2024autogen,hong2024metagpt,wang2025mixtureofagents,du2024improving,ong2025routellm,shen2024collaborative_decoding}.
Although model-agnostic, text communication discretizes internal representations, introducing a semantic bottleneck and decoding overhead, while routing does not directly transfer internal understanding.
Embedding-, activation-, and latent-state methods
\citep{pham2024cipher,ramesh2025communicating,tang2025state_delta}
provide continuous alternatives but remain tied to particular representations or language trajectories; heterogeneous federated methods such as HeteroTune \citep{jia2024heterotune} instead focus on training-time knowledge sharing.

Recent work therefore communicates context-specific internal states directly.
KVComm \citep{kvcomm} approximates KV-cache offsets for training-free cross-context reuse;
LSC \citep{lsc} aligns model caches in a global shared latent space;
C2C \citep{c2c} directly projects and fuses source caches into the receiver space;
and Interlat \citep{interlat} compresses last-hidden-state ``thought'' representations for language-free communication.
Related studies further explore semantic compression and aligned latent communication
\citep{chen-etal-2026-condenseflow,rossi2026latent,chen2026see}.
These methods extend communication beyond text, logits, and parameters, but KVComm mainly targets RoPE-based cross-context reuse, LSC requires shared-space alignment, C2C can incur pairwise fuser or translator costs as the model pool grows, and Interlat does not explicitly resolve KV-format mismatch.
\mot instead uses a gated mixture of translators to increase translation capacity for heterogeneous KV caches without relying on a single translator or fixed shared space.

\subsection{Memory Reuse for Agentic Reasoning}

Agentic workflows repeatedly reuse system prompts, tool instructions, retrieved contexts, and shared workspaces across agents or branches \citep{du2024improving}.
Prefix-cache reuse removes redundant prefilling for identical prefixes, but model-specific KV formats force heterogeneous models to maintain separate copies.
Cache reuse can also affect downstream judging when cross-candidate interaction is required \citep{liang-etal-2026-kv}, whereas we focus on memory redundancy.
We classify compatibility as \textbf{Homogeneous} (same model), \textbf{Intra-Family} (same family), or \textbf{Cross-Family} (different families); Appendix~\ref{app:translation-taxonomy} gives the full taxonomy.

Same-format reuse ranges from token-exact or structured prefix sharing (ChunkAttention, PromptCache) to workflow-aware or multi-tier caching and semantic/segment-level reuse (KVFlow, IMPRESS, SemShareKV, KVShare, CrossKV)
\citep{ye2024chunkattention,gim2024prompt_cache,pan2025kvflow,chen2025impress,zhao-mastorakis-2025-semsharekv,yang2025kvshare,anonymous2026crosskv};
all assume a common model or KV format and thus retain multi-model redundancy.
Cross-model methods such as C2C, LSC, and FedRefine, together with closed-form or calibrated mappings
\citep{c2c,lsc,chen2026fedrefine,heo2026crossmodelkvcachetransfer,qu2026cachebridge},
enable state transfer across models or devices but introduce pairwise mapping, shared-space alignment, or federated communication costs.
Under stronger compatibility assumptions, DroidSpeak selectively recomputes layers across same-architecture variants, LRAgent separates shared-backbone and LoRA cache components, and ICaRus freezes a common cache-generating encoder to enable identical-cache reuse
\citep{liu2024droidspeak,jeon2026lragent,woo2026icarus};
these mechanisms are difficult to extend to arbitrary cross-family models.
Vision Wormhole \citep{liu2026vision_wormhole} instead communicates visual latent tokens through VLM interfaces and is therefore related but not a prefix-KV reuse method.

\mot stores shared prefixes once in the hub-model format and translates them to target formats on demand, reducing cache multiplicity to $O(1)$ rather than scaling with the number of participating models.

\begin{table}[t]
\caption{
Comparison of shared-prefix cache reuse across $M$ participating models or agents.
}
\label{tab:rw-mem-dup}
\centering
\resizebox{\textwidth}{!}{
\begin{tabular}{lccc}
\toprule
\textbf{Category} &
\textbf{Heterogeneity Level} &
\textbf{Cache Format} &
\textbf{Memory Redundancy} \\
\midrule

\makecell[l]{%
ChunkAttention\citep{ye2024chunkattention}, PromptCache\citep{gim2024prompt_cache},\\
KVFlow\citep{pan2025kvflow}, SemShareKV\citep{zhao-mastorakis-2025-semsharekv},\\
KVShare\citep{yang2025kvshare}, CrossKV\citep{anonymous2026crosskv}%
} &
Homogeneous &
KV Cache &
$O(M)$ \\

\makecell[l]{%
DroidSpeak\citep{liu2024droidspeak},\\
LRAgent\citep{jeon2026lragent}%
} &
Intra-Family &
KV Cache &
$O(M)$ \\

ICaRus\citep{woo2026icarus} &
Intra-Family &
KV Cache &
$O(1)$ \\

C2C\citep{c2c}, LSC\citep{lsc} &
Cross-Family &
KV Cache &
$O(M)$ \\

Vision Wormhole\citep{liu2026vision_wormhole} &
Cross-Family &
Visual Tokens (Non-KV) &
N/A \\

\midrule

\textbf{\mot (Proposed)} &
\textbf{Cross-Family} &
\textbf{KV Cache} &
\textbf{$O(1)$} \\

\bottomrule
\end{tabular}
}
\end{table}

Table~\ref{tab:rw-mem-dup} separates cross-model compatibility from cache multiplicity:
cross-model transfer can still retain $O(M)$ model-local caches, whereas \mot retains only the Hub Agent cache.

\subsection{Storage Reuse for Long-Context RAG}

RAG retrieves query-relevant documents and shares model-agnostic text across heterogeneous models without additional training \citep{lewis2020retrieval}, but each request still requires prefilling and KV-cache construction.
Internal-state reuse avoids this repeated work by precomputing representations offline:
TurboRAG stores KV caches for retrieved chunks \citep{lu2025turborag};
CAG preloads a collection into offline caches for query-only online generation \citep{chan2024dontdorag};
HCache restores KV caches from stored hidden states \citep{hcache};
KVLink uses trainable special tokens and training for chunk-cache reuse \citep{yang2025kvlink};
and SpecCache selectively reuses cached states while recomputing critical tokens identified by a lightweight speculative model \citep{wen-etal-2026-speccache}.
However, these stored KV or hidden states remain architecture- and tokenizer-specific, so the same item must be stored separately when the model changes
\citep{lu2025turborag,chan2024dontdorag,hcache,wen-etal-2026-speccache}.

\mot instead stores each shared prefix once as a hub-format KV cache and performs one-hop hub-to-target translation at inference, avoiding per-model prefix-state stores.

\begin{table}[t]
\caption{
Comparison of long-context storage reuse across $M$ models.
}
\label{tab:rw-storage-reuse}
\centering\small
\begin{tabular}{lcc}
\toprule
\textbf{Category} &
\textbf{Heterogeneity Level} &
\textbf{Storage Redundancy} \\
\midrule

RAG\citep{lewis2020retrieval} &
Cross-Family\textsuperscript{\dag} &
$O(1)$ (Text) \\

\makecell[l]{%
CAG\citep{chan2024dontdorag}, HCache\citep{hcache},\\
TurboRAG\citep{lu2025turborag}, KVLink\citep{yang2025kvlink},\\
SpecCache\citep{wen-etal-2026-speccache}%
} &
Homogeneous &
$O(M)$ \\

\midrule

\textbf{\mot (Proposed)} &
\textbf{Cross-Family} &
\textbf{$O(1)$} \\

\bottomrule
\end{tabular}
\vspace{1mm}

{\footnotesize
\raggedright
\textsuperscript{\dag}RAG shares model-agnostic text rather than stored
model-dependent internal states.\par
}
\end{table}

Table~\ref{tab:rw-storage-reuse} distinguishes model-agnostic text sharing from model-dependent internal-state reuse:
existing cache-based methods require per-model states, whereas \mot retains a single hub-format prefix cache across heterogeneous models.

\section{Theoretical Analysis}
\label{app:theoretical_analysis}

\subsection{Assumptions}

The following assumption specifies the local regularity condition used in our
analysis of translation-error propagation and correction dynamics.

\begin{assumption}
\label{assump:residual_lipschitz}
(Local Lipschitz Residual Branch) \\
There exists a constant \(\delta\ge 0\) such that, for every layer \(\ell\), the
residual branch \(F_\ell\) is locally Lipschitz on the hidden-state region. That is, for any two hidden states \(\bx\) and \(\bx'\) in this
region,
\[
\|F_\ell(\bx)-F_\ell(\bx')\|
\le
\delta\|\bx-\bx'\|.
\]
We call \(\delta\) the uniform residual-branch Lipschitz constant.
\end{assumption}

\subsection{Shift Delta and Correction Deficit Coefficient}

\begin{definition}
\label{def:shift_delta_and_correction_deficit}
(Shift Delta and Correction Deficit Coefficient) \\
For the highest target translation layer \(T\), let \(\bs_T\) denote the translation shift accumulated within the channel set, and let \(\bs_L\) denote the final shift observed at the last hidden state. We define the shift delta additionally induced by the remaining upper layers after \(T\) as
\(\Delta \bs_{T:L}:=\bs_L-\bs_T\).
That is, the shift delta is a vector that describes how the residual dynamics after \(T\) modify the translation shift.

When \(\bs_T\neq 0\), we define the anti-shift correction coefficient as
\[
\alpha_{T:L}
:=
-\left\langle
\Delta \bs_{T:L},
\bu^{\parallel}
\right\rangle,
\qquad
\bu^{\parallel}:=\frac{\bs_T}{\|\bs_T\|}.
\]
Here, \(\bu^{\parallel}\) denotes the unit vector in the translation-shift direction.
We also define the orthogonal shift coefficient as
\[
\beta_{T:L}
:=
\left\|
\Delta \bs_{T:L}
-
\left\langle
\Delta \bs_{T:L},
\bu^{\parallel}
\right\rangle
\bu^{\parallel}
\right\|.
\]
Finally, the correction deficit coefficient is defined as
\[
d_{T:L}
:=
\frac{\langle \bs_L,\bu^{\parallel}\rangle}{\|\bs_T\|}.
\]
This value is a normalized projection coefficient that quantifies how much of the translation-shift direction accumulated within the channel set remains in the last hidden state.
\end{definition}

\subsection{Proof of Proposition~\ref{prop:source_propagation_error}}
\label{app:proof:source_propagation_error}

\begin{proof}
Let \(\bh_\ell\) and \(\widehat{\bh}_\ell\) denote the hidden states after the \(\ell\)-th layer induced by the native target cache and the translated target cache, respectively, and define
\(\bs_\ell := \bh_\ell-\widehat{\bh}_\ell\).
The shift formed at \(T_{\mathrm{Start}}\) is
\(\bs_{T_{\mathrm{Start}}}
=
\bh_{T_{\mathrm{Start}}}-\widehat{\bh}_{T_{\mathrm{Start}}}\).

This proposition tracks how much the translation-induced shift component generated at \(T_{\mathrm{Start}}\) can be amplified as it passes through subsequent layers. For each layer \(\ell=T_{\mathrm{Start}}+1,\dots,L\), assuming that the two hidden-state trajectories propagate through the same target residual recurrence, we can write
\[
\bh_\ell
=
\bh_{\ell-1}+F_\ell(\bh_{\ell-1}),
\qquad
\widehat{\bh}_\ell
=
\widehat{\bh}_{\ell-1}+F_\ell(\widehat{\bh}_{\ell-1})
\]
Therefore,
\[
\bs_\ell
=
\bs_{\ell-1}
+
\left(
F_\ell(\bh_{\ell-1})-F_\ell(\widehat{\bh}_{\ell-1})
\right).
\]
Taking the norm on both sides and applying the triangle inequality together with Assumption~\ref{assump:residual_lipschitz}, we obtain
\[
\|\bs_\ell\|
\le
\|\bs_{\ell-1}\|
+
\delta\|\bs_{\ell-1}\|
=
(1+\delta)\|\bs_{\ell-1}\|.
\]
Applying this recursively for \(\ell=T_{\mathrm{Start}}+1,\dots,L\) gives
\[
\boxed{
\|\bs_L\|
\le
(1+\delta)^{L-T_{\mathrm{Start}}}
\|\bs_{T_{\mathrm{Start}}}\|.
}
\]
This proves Eq.~\eqref{eq:combined_shift_bound}. Since \(L\) and \(\|\bs_{T_{\mathrm{Start}}}\|\) are fixed, increasing \(T_{\mathrm{Start}}\) reduces the exponent \(L-T_{\mathrm{Start}}\), and hence the bound on the propagated last-state shift decreases exponentially.
\end{proof}

\subsection{Proof of Proposition~\ref{prop:correction_deficit_coefficient_bound}}
\label{app:proof:correction_deficit_coefficient_bound}

This appendix gives the detailed derivation behind
Proposition~\ref{prop:correction_deficit_coefficient_bound}. The key point is
that, as the highest target translation layer \(T\) moves toward the final layer
\(L\), fewer upper layers remain available to correct the translation shift.
Consequently, the correction-deficit coefficient approaches its terminal value.

\begin{proof}
Let \(\bh_\ell\) and \(\widehat{\bh}_\ell\) denote the hidden states after the
\(\ell\)-th layer induced by the native target cache and the translated target
cache, respectively, and define
\(\bs_\ell := \bh_\ell-\widehat{\bh}_\ell\).
For a channel set \(\mathcal C\), define the target translation layer set and its boundary layers as
\[
\mathcal I^{\mathrm{Tgt}}(\mathcal C)
:=
\{j \mid (i,j)\in\mathcal C\},
\qquad
T_{\mathrm{Start}}
:=
\min \mathcal I^{\mathrm{Tgt}}(\mathcal C),
\qquad
T
:=
\max \mathcal I^{\mathrm{Tgt}}(\mathcal C).
\]
The shift formed at \(T_{\mathrm{Start}}\) is
\(\bs_{T_{\mathrm{Start}}}
=
\bh_{T_{\mathrm{Start}}}-\widehat{\bh}_{T_{\mathrm{Start}}}\).

For each layer \(\ell=T_{\mathrm{Start}}+1,\dots,T\), assuming that the two
hidden-state trajectories propagate through the same target residual recurrence,
we can write
\[
\bh_\ell
=
\bh_{\ell-1}+F_\ell(\bh_{\ell-1}),
\qquad
\widehat{\bh}_\ell
=
\widehat{\bh}_{\ell-1}+F_\ell(\widehat{\bh}_{\ell-1}).
\]
Therefore,
\[
\bs_\ell
=
\bs_{\ell-1}
+
\left(
F_\ell(\bh_{\ell-1})-F_\ell(\widehat{\bh}_{\ell-1})
\right).
\]
Taking the norm and applying the triangle inequality together with
Assumption~\ref{assump:residual_lipschitz}, we obtain
\[
\|\bs_\ell\|
\le
\|\bs_{\ell-1}\|
+
\delta\|\bs_{\ell-1}\|
=
(1+\delta)\|\bs_{\ell-1}\|.
\]
Applying this recursively for \(\ell=T_{\mathrm{Start}}+1,\dots,T\) gives
\begin{equation}
\|\bs_T\|
\le
(1+\delta)^{T-T_{\mathrm{Start}}}
\|\bs_{T_{\mathrm{Start}}}\|.
\label{eq:translation_shift_bound_reintroduced}
\end{equation}
Thus, when \(T_{\mathrm{Start}}\) and
\(\|\bs_{T_{\mathrm{Start}}}\|\) are fixed, the upper bound on the translation
shift \(\|\bs_T\|\) can increase as \(T\) increases.

We now bound the correction opportunity provided by the remaining upper layers
after \(T\). By Definition~\ref{def:shift_delta_and_correction_deficit},
\(\Delta \bs_{T:L}=\bs_L-\bs_T\).
Also,
\[
\bs_L-\bs_T
=
\sum_{\ell=T+1}^{L}
(\bs_\ell-\bs_{\ell-1}),
\]
and therefore
\[
\|\Delta \bs_{T:L}\|
\le
\sum_{\ell=T+1}^{L}
\|\bs_\ell-\bs_{\ell-1}\|.
\]
For each term,
\[
\bs_\ell-\bs_{\ell-1}
=
F_\ell(\bh_{\ell-1})-F_\ell(\widehat{\bh}_{\ell-1}),
\]
so by Assumption~\ref{assump:residual_lipschitz},
\[
\|\bs_\ell-\bs_{\ell-1}\|
\le
\delta\|\bs_{\ell-1}\|.
\]
Moreover, applying the recurrence bound starting from \(T\), we have
\[
\|\bs_{\ell-1}\|
\le
(1+\delta)^{\ell-1-T}\|\bs_T\|.
\]
Therefore,
\[
\|\Delta \bs_{T:L}\|
\le
\sum_{\ell=T+1}^{L}
\delta(1+\delta)^{\ell-1-T}\|\bs_T\|.
\]
Computing the geometric sum yields
\[
\|\Delta \bs_{T:L}\|
\le
\Bigl((1+\delta)^{L-T}-1\Bigr)\|\bs_T\|.
\]

By Definition~\ref{def:shift_delta_and_correction_deficit}, the anti-shift
correction coefficient is
\[
\alpha_{T:L}
=
-\left\langle
\Delta \bs_{T:L},
\bu^{\parallel}
\right\rangle.
\]
Thus, by the Cauchy--Schwarz inequality,
\[
\alpha_{T:L}
\le
\left|
\left\langle
\Delta \bs_{T:L},
\bu^{\parallel}
\right\rangle
\right|
\le
\|\Delta \bs_{T:L}\|.
\]
Hence,
\begin{equation}
\alpha_{T:L}
\le
\Bigl((1+\delta)^{L-T}-1\Bigr)\|\bs_T\|.
\label{eq:anti_shift_correction_bound_local}
\end{equation}
Substituting Eq.~\eqref{eq:translation_shift_bound_reintroduced} into this bound
gives
\[
\alpha_{T:L}
\le
\Bigl(
(1+\delta)^{L-T_{\mathrm{Start}}}
-
(1+\delta)^{T-T_{\mathrm{Start}}}
\Bigr)
\|\bs_{T_{\mathrm{Start}}}\|.
\]
Therefore, when \(T_{\mathrm{Start}}\) and
\(\|\bs_{T_{\mathrm{Start}}}\|\) are fixed, this upper bound decreases to \(0\)
as \(T\to L\). This formalizes the intuition that late translation leaves fewer
remaining upper layers for anti-shift correction.

Finally, we derive the correction-deficit bound. Since
\(\bs_L=\bs_T+\Delta \bs_{T:L}\) and
\(\bs_T=\|\bs_T\|\bu^{\parallel}\), we have
\[
\langle \bs_L,\bu^{\parallel}\rangle
=
\|\bs_T\|
+
\left\langle
\Delta \bs_{T:L},
\bu^{\parallel}
\right\rangle
=
\|\bs_T\|-\alpha_{T:L}.
\]
Dividing both sides by \(\|\bs_T\|\), we obtain
\(d_{T:L}=1-\alpha_{T:L}/\|\bs_T\|\).
Using Eq.~\eqref{eq:anti_shift_correction_bound_local}, we get
\[
d_{T:L}
\ge
1-
\Bigl((1+\delta)^{L-T}-1\Bigr)
=
2-(1+\delta)^{L-T}.
\]
Thus, when \(L\) is fixed, the lower bound on \(d_{T:L}\) increases as \(T\)
increases.

When \(T=L\), there are no remaining upper layers, so
\(\Delta \bs_{L:L}=\bs_L-\bs_L=0\).
Therefore, \(\alpha_{L:L}=0\) and \(\beta_{L:L}=0\).
Moreover, if \(\bs_L\neq0\), then
\(\bu^{\parallel}=\bs_L/\|\bs_L\|\).
Hence,
\(d_{L:L}
=
\langle \bs_L,\bs_L/\|\bs_L\|\rangle/\|\bs_L\|
=
1\).
Thus, at the final hidden state, no upper-layer correction opportunity remains,
and the correction-deficit coefficient reaches its terminal value.
\end{proof}

\subsection{Proof of Proposition~\ref{prop:last_state_shift_decomposition}}
\label{app:proof:last_state_shift_decomposition}

\begin{proof}
By Definition~\ref{def:shift_delta_and_correction_deficit},
\[
\alpha_{T:L}
=
-\left\langle
\Delta \bs_{T:L},
\bu^{\parallel}
\right\rangle
\]
and
\[
\beta_{T:L}
=
\left\|
\Delta \bs_{T:L}
-
\left\langle
\Delta \bs_{T:L},
\bu^{\parallel}
\right\rangle
\bu^{\parallel}
\right\|.
\]
Therefore, when \(\beta_{T:L}>0\), for some \(\bu^{\perp}\perp \bu^{\parallel}\), we can write
\[
\Delta \bs_{T:L}
=
-\alpha_{T:L}\bu^{\parallel}
+
\beta_{T:L}\bu^{\perp}
\]
When \(\beta_{T:L}=0\), the same expression holds by choosing any \(\bu^{\perp}\perp \bu^{\parallel}\).

By definition,
\(\bs_L=\bs_T+\Delta \bs_{T:L}\) and
\(\bs_T=\|\bs_T\|\bu^{\parallel}\).
Thus,
\[
\bs_L
=
(\|\bs_T\|-\alpha_{T:L})\bu^{\parallel}
+
\beta_{T:L}\bu^{\perp}.
\]
Since \(\bu^{\parallel}\perp \bu^{\perp}\),
\[
\|\bs_L\|^2
=
(\|\bs_T\|-\alpha_{T:L})^2+\beta_{T:L}^2.
\]
Moreover, by Definition~\ref{def:shift_delta_and_correction_deficit},
\[
d_{T:L}
=
\frac{\langle \bs_L,\bu^{\parallel}\rangle}{\|\bs_T\|}
=
\frac{\|\bs_T\|-\alpha_{T:L}}{\|\bs_T\|}.
\]
Therefore,
\(\|\bs_T\|-\alpha_{T:L}=d_{T:L}\|\bs_T\|\).
Substituting this into the previous expression gives Eq.~\eqref{eq:last_state_shift_decomposition}.
\end{proof}

\section{Proposed Method Details}
\label{app:proposed_method_details}

\subsection{Mixture-of-Translators}
\label{app:mot}

Prior work on mixture-of-experts has shown that input-dependent specialization
and learned routing can improve over applying a single expert uniformly across
inputs
\citep{jacobs1991adaptive,chen2022towards,dikkala2023benefits}.
Following this perspective, we analyze how token-level translator routing affects
cache translation error and, in turn, the translation shift studied in the main
text.

For a fixed channel set \(\mathcal C\), let \(\by_u\) denote the native target-cache
block for token \(u\), and let
\(\widehat{\by}_{u,m}:=\operatorname{Tr}_m(\bz_u)\) be the output of translator
\(m\). We write the corresponding error as
\(e_{u,m}:=\|\widehat{\by}_{u,m}-\by_u\|^2\).

A single translator must use the same translator for every token, giving the
best achievable error
\(\mathcal E_{\mathrm{Single}}:=\min_m\sum_u e_{u,m}\).
In contrast, token-wise selection can choose the best translator separately for
each token. We define the resulting specialization gain as
\[
\Gamma_{\mathrm{Spec}}
:=
\mathcal E_{\mathrm{Single}}
-
\sum_u \min_m e_{u,m}
\ge 0.
\]
Thus, \(\Gamma_{\mathrm{Spec}}>0\) whenever different tokens benefit from
different translators.

Let \(w_{u,m}\) denote the normalized routing weight assigned by \mot to
translator \(m\) for token \(u\), with zero weight for unselected translators.
The learned router need not always select the best translator. We therefore
define its routing gap as
\[
\mathcal R_{\mathrm{Route}}
:=
\sum_u
\left(
\sum_m w_{u,m}e_{u,m}
-
\min_m e_{u,m}
\right)
\ge 0.
\]

\begin{proposition}[Token-Wise Translator Specialization]
\label{prop:mot_reduces_translation_shift}
For a fixed translator pool, the cache mismatch produced by
\textsf{Mixture-of-Translators} satisfies
\begin{equation}
\left\|
\widehat{\bY}_{\mathrm{MoT}}-\bY
\right\|_F^2
\le
\mathcal E_{\mathrm{Single}}
-
\Gamma_{\mathrm{Spec}}
+
\mathcal R_{\mathrm{Route}}.
\label{eq:mot_specialization_bound}
\end{equation}
Consequently, if
\(\mathcal R_{\mathrm{Route}}<\Gamma_{\mathrm{Spec}}\), token-level routing
achieves lower cache mismatch than the best single translator from the same
translator pool.
\end{proposition}

\begin{proof}
For each token,
\(\widehat{\by}^{\mathrm{MoT}}_u-\by_u
=
\sum_m w_{u,m}(\widehat{\by}_{u,m}-\by_u)\).
Since the squared Euclidean norm is convex,
\[
\left\|
\widehat{\by}^{\mathrm{MoT}}_u-\by_u
\right\|^2
\le
\sum_m w_{u,m}e_{u,m}.
\]
Summing over tokens and using the definitions of
\(\Gamma_{\mathrm{Spec}}\) and \(\mathcal R_{\mathrm{Route}}\) gives
Eq.~\eqref{eq:mot_specialization_bound}. If
\(\mathcal R_{\mathrm{Route}}<\Gamma_{\mathrm{Spec}}\), its right-hand side is
strictly smaller than \(\mathcal E_{\mathrm{Single}}\).
\end{proof}

We next connect this cache-space result to the translation shift
\(\bs_T\) used in the main analysis. Let \(\Phi_T\) denote the fixed target
computation that maps the translated cache to the hidden state at the highest
target translation layer
\(T:=\max\{j\mid(i,j)\in\mathcal C\}\).
Let \(\kappa_T\) denote its local sensitivity, such that
\(\|\Phi_T(\bY_1)-\Phi_T(\bY_2)\|
\le
\kappa_T\|\bY_1-\bY_2\|_F\)
over the relevant cache region.

\begin{corollary}[Translation-Shift Bound for \textsf{Mixture-of-Translators}]
\label{cor:mot_translation_shift_bound}
The translation shift induced by \mot satisfies
\begin{equation}
\|\bs_T\|^2
\le
\kappa_T^2
\left(
\mathcal E_{\mathrm{Single}}
-
\Gamma_{\mathrm{Spec}}
+
\mathcal R_{\mathrm{Route}}
\right).
\label{eq:mot_translation_shift_bound}
\end{equation}
Under Assumption~\ref{assump:residual_lipschitz}, this shift further propagates as
\begin{equation}
\|\bs_L\|
\le
(1+\delta)^{L-T}
\kappa_T
\sqrt{
\mathcal E_{\mathrm{Single}}
-
\Gamma_{\mathrm{Spec}}
+
\mathcal R_{\mathrm{Route}}
}.
\label{eq:mot_last_state_shift_bound}
\end{equation}
\end{corollary}

\begin{proof}
By the local sensitivity of \(\Phi_T\),
\(\|\bs_T\|
\le
\kappa_T
\|\widehat{\bY}_{\mathrm{MoT}}-\bY\|_F\).
Combining this with
Proposition~\ref{prop:mot_reduces_translation_shift} gives
Eq.~\eqref{eq:mot_translation_shift_bound}.
After layer \(T\), the two trajectories follow the same target residual
recurrence, so Proposition~\ref{prop:source_propagation_error} gives
\(\|\bs_L\|\le(1+\delta)^{L-T}\|\bs_T\|\), which yields
Eq.~\eqref{eq:mot_last_state_shift_bound}.
\end{proof}

Combining Eq.~\eqref{eq:mot_translation_shift_bound} with
Proposition~\ref{prop:last_state_shift_decomposition} further gives
\begin{equation}
\|\bs_L\|^2
\le
d_{T:L}^2
\kappa_T^2
\left(
\mathcal E_{\mathrm{Single}}
-
\Gamma_{\mathrm{Spec}}
+
\mathcal R_{\mathrm{Route}}
\right)
+
\beta_{T:L}^2.
\label{eq:mot_shift_decomposition_bound}
\end{equation}
Thus, token-level routing reduces the translation-shift contribution when its
specialization gain exceeds its routing gap. The remaining error depends on the
upper-layer correction dynamics, motivating the separate
\textsf{Context Correction Loss}.

\subsection{Context Correction Loss}
\label{app:context_correction_loss}

In practice, instead of explicitly storing and comparing hidden states, we use a heuristic that compares the KV cache formed during replay with the native KV cache. Thus, in implementation, we use the following KV-based form:
\begin{equation}
\sum_{\ell=T_{\mathrm{Start}}+1}^{L}
\left(
\left\|
\widehat{\bK}_\ell^{\mathrm{Tgt}}
-
\bK_\ell^{\mathrm{Tgt}}
\right\|^2
+
\left\|
\widehat{\bV}_\ell^{\mathrm{Tgt}}
-
\bV_\ell^{\mathrm{Tgt}}
\right\|^2
\right).
\label{eq:context_correction_kv}
\end{equation}

\begin{proposition}[Last-State Shift Reduction by \textsf{Context Correction Loss}]
\label{prop:context_correction_kv_bound}
Suppose that the stacked final-layer key/value projection of the target model has full column rank. Then, reducing the final-layer term of Eq.~\eqref{eq:context_correction_kv} provides a training signal that reduces an upper bound on the last-state shift \(\|\bs_L\|\).
\end{proposition}

\begin{proof}
Let the final-layer key and value projection matrices be denoted by
\(\bW_L^K\) and \(\bW_L^V\), respectively, and define the stacked final-layer
key/value projection as
\[
\bW_L^{KV}
:=
\begin{bmatrix}
\bW_L^K\\
\bW_L^V
\end{bmatrix}.
\]
The assumption in Proposition~\ref{prop:context_correction_kv_bound}
means that \(\bW_L^{KV}\) has full column rank.

At each layer, the native key and value caches are fixed linear projections of the native hidden state:
\[
\bK_\ell^{\mathrm{Tgt}}
=
\bW_\ell^K \bh_\ell^{\mathrm{Tgt}},
\qquad
\bV_\ell^{\mathrm{Tgt}}
=
\bW_\ell^V \bh_\ell^{\mathrm{Tgt}}.
\]
Similarly, the replayed key and value caches are given by
\[
\widehat{\bK}_\ell^{\mathrm{Tgt}}
=
\bW_\ell^K \widehat{\bh}_\ell^{\mathrm{Tgt}},
\qquad
\widehat{\bV}_\ell^{\mathrm{Tgt}}
=
\bW_\ell^V \widehat{\bh}_\ell^{\mathrm{Tgt}}.
\]

By the definition of the last-state shift,
\(\bs_L=\widehat{\bh}_L^{\mathrm{Tgt}}-\bh_L^{\mathrm{Tgt}}\).
Therefore, at the final layer,
\[
\widehat{\bK}_L^{\mathrm{Tgt}}
-
\bK_L^{\mathrm{Tgt}}
=
\bW_L^K \bs_L,
\qquad
\widehat{\bV}_L^{\mathrm{Tgt}}
-
\bV_L^{\mathrm{Tgt}}
=
\bW_L^V \bs_L.
\]
Therefore, the final-layer KV matching error is
\[
\left\|
\bW_L^K \bs_L
\right\|^2
+
\left\|
\bW_L^V \bs_L
\right\|^2
=
\left\|
\begin{bmatrix}
\bW_L^K\\
\bW_L^V
\end{bmatrix}
\bs_L
\right\|^2.
\]

Since
\[
\begin{bmatrix}
\bW_L^K\\
\bW_L^V
\end{bmatrix}
\]
is fixed by the target model and has full column rank, any nonzero shift \(\bs_L\neq 0\) cannot vanish in the final-layer key/value space.

Furthermore, because this projection is a fixed linear map that does not change during training, the proportional relationship between the final-layer KV matching error and the last-state shift also remains fixed throughout training. Therefore, reducing the final-layer KV matching error is equivalent to reducing a fixed upper bound on \(\|\bs_L\|\). Since Eq.~\eqref{eq:context_correction_kv} includes this final-layer term, the KV-based implementation is justified as a heuristic for reducing the last-state shift.
\end{proof}

\subsection{Backbone Translator}
\label{app:backbone-translator}

The backbone translator directly maps the source context \(\bK\bV\) cache into
the target cache space for a selected channel set. We adopt the
cross-attention-based translation pattern of prior work~\citep{lsc}, but unlike
shared-latent-space approaches, we perform direct source-to-target cache
translation without introducing an additional shared latent space.

Keys and values are translated independently. The translator first projects
each source cache block into a common translator space, combines information
across the selected channels through recurrent cross-attention, and projects the
result into the target cache space. The same formulation applies to contiguous
and non-contiguous channel sets.

Given a fixed channel set \(\mathcal C\), let
\(\bX_{\mathcal C}^{\mathrm{Src},\mathrm{Ctx}}
=
(\bX_1^{\mathrm{Src},\mathrm{Ctx}},\dots,\bX_N^{\mathrm{Src},\mathrm{Ctx}})\)
denote the source-side cache block selected by \(\mathcal C\), where \(N=|\mathcal C|\). Here, the index \(r\in\{1,\dots,N\}\) refers only to the slot order inside the selected channel window, not to an absolute model layer index. Key and value caches are translated independently, so the following description applies to each \(X\in\{K,V\}\).

First, each source cache slot is normalized and projected into the translator hidden space:
\[
\bz_r^X
=
\phi\!\left(
\bW_{\mathrm{In}}^X
\operatorname{LN}_{\mathrm{In}}^X
\left(
\bX_r^{\mathrm{Src},\mathrm{Ctx}}
\right)
\right),
\qquad
r=1,\dots,N,
\]
where \(\phi\) denotes the GELU activation. The translator then performs recurrent cross-attention over the projected window. It maintains a running hidden state \(\bh^X\), initialized as
\(\bh_0^X=\bz_1^X\).
For each slot \(r\), the running hidden state attends to the corresponding projected source cache slot:
\[
\bh_r^X
=
\operatorname{CAB}_r^X
\left(
\bh_{r-1}^X,
\bz_r^X
\right),
\qquad
r=1,\dots,N.
\]
The cross-attention block is defined as
\[
\ba_r^X
=
\operatorname{CrossAttn}
\left(
\operatorname{LN}_{Q}(\bh_{r-1}^X),
\operatorname{LN}_{C}(\bz_r^X),
\operatorname{LN}_{C}(\bz_r^X)
\right),
\]
\[
\widetilde{\bh}_r^X
=
\bh_{r-1}^X+\ba_r^X,
\]
\[
\bh_r^X
=
\widetilde{\bh}_r^X
+
\operatorname{FFN}
\left(
\operatorname{LN}_{F}(\widetilde{\bh}_r^X)
\right).
\]
Thus, in the notation \(\operatorname{CrossAttn}(Q,K,V)\), the query \(Q\) is the normalized running hidden state, while both the key \(K\) and value \(V\) are the normalized projected source cache slot. Intuitively, the running hidden state accumulates information from previously visited slots, and each cross-attention step lets it retrieve the information needed from the current source slot. This cross-layer scan can be repeated for multiple translator stages, following the recurrent translator pattern of LSC~\citep{lsc}.

After the final scan, the hidden states collected across the channel window are concatenated and projected into the target cache space:
\[
\bF^X
=
\operatorname{Concat}
\left(
\bh_1^X,\dots,\bh_N^X
\right),
\]
\[
\widehat{\bX}_{\mathcal C}^{\mathrm{Tgt},\mathrm{Ctx}}
=
\phi\!\left(
\bW_{\mathrm{Out}}^X
\operatorname{LN}_{\mathrm{Out}}^X
\left(
\bF^X
\right)
\right),
\qquad
X\in\{K,V\}.
\]
The translated key and value blocks are then combined as
\[
\widehat{\bK\bV}_{\mathcal C}^{\mathrm{Tgt},\mathrm{Ctx}}
=
\left(
\widehat{\bK}_{\mathcal C}^{\mathrm{Tgt},\mathrm{Ctx}},
\widehat{\bV}_{\mathcal C}^{\mathrm{Tgt},\mathrm{Ctx}}
\right).
\]

This formulation abstracts away implementation-level details such as tensor layout changes and reshaping. The important point is that the backbone translator uses recurrent cross-attention to fuse information across the selected source cache window and directly emits the corresponding target cache block. It serves as the basic translation unit used by \mot: each expert in \mot has the same backbone structure, while the routing function determines how the outputs of multiple backbone translators are selected or combined for each token.

\subsection{Pre-Translation Phase: Channel Mapping}
\label{app:channel-mapping}

Channel mapping is the pre-translation stage that determines the set of layer pairs
\(\mathcal C\) over which cache translation is performed between the source model
and the target model. The main focus of this work is to learn accurate cache
translation for a given channel set, while the bi-objective problem of jointly
optimizing the size of the channel set is left outside the scope of this paper.
Therefore, we assume that \(|\mathcal C|\) is given, and discuss the corresponding
parameter search in the experimental section.

Before specifying the channel-mapping rule, we first conduct a preliminary
layer-mapping analysis to examine whether layers of heterogeneous models exhibit
a meaningful correspondence. Specifically, we prefill the same OpenWebText
samples into both models, compute the mean representations of the key and value
caches produced at each layer, and measure the layer-wise similarity for all
source-to-target layer combinations. Figure~\ref{fig:layer_mapping} visualizes the
resulting similarity matrices for several heterogeneous model pairs.

\begin{figure*}[t!]
    \centering

    \begin{minipage}[t]{0.19\textwidth}
        \centering
        \includegraphics[width=3cm,height=4cm]{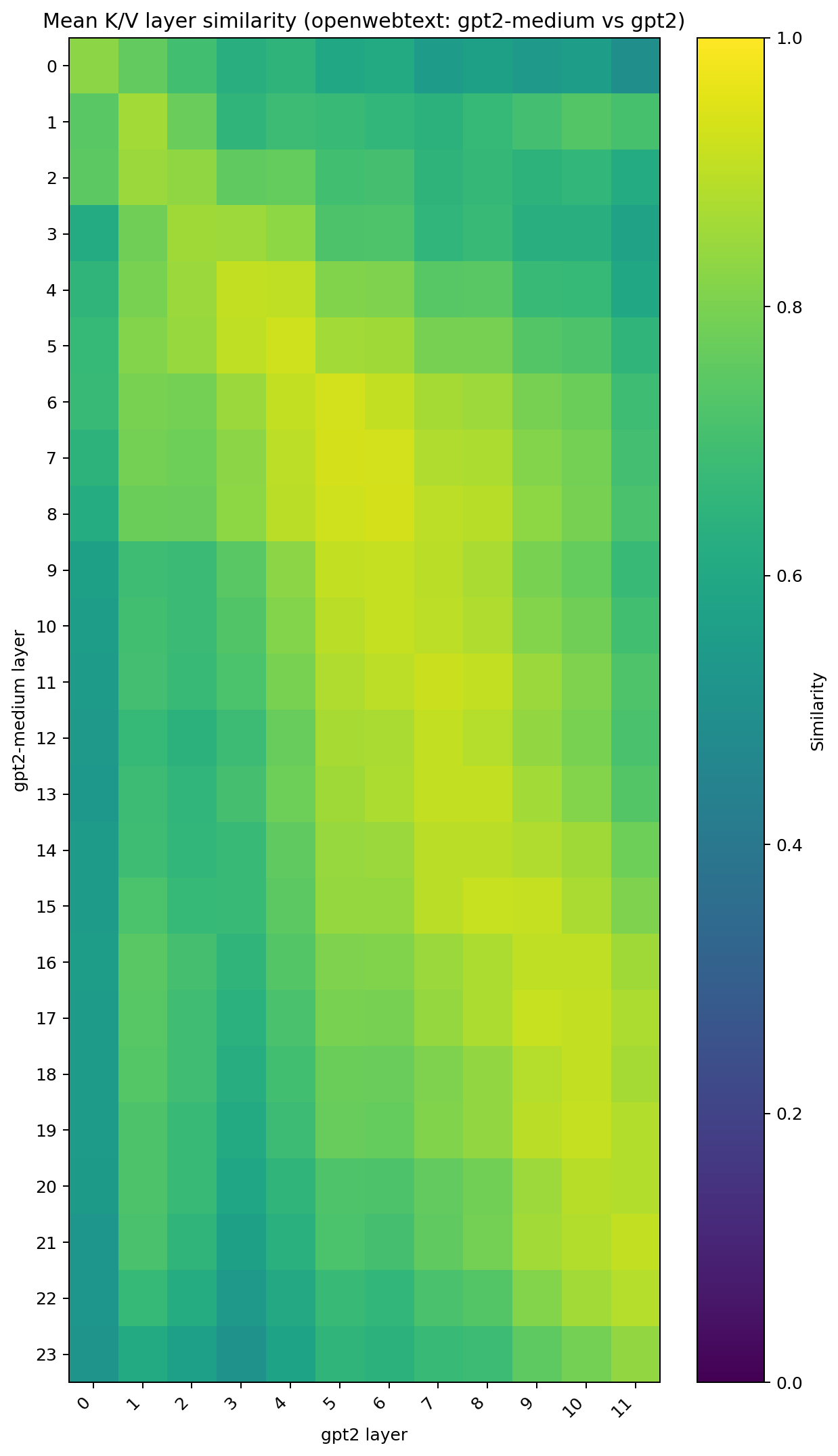}
        {\small (a) \texttt{gpt2-medium}\newline↔\,\texttt{gpt2}\par}
    \end{minipage}
    \hfill
    \begin{minipage}[t]{0.19\textwidth}
        \centering
        \includegraphics[width=3cm,height=4cm]{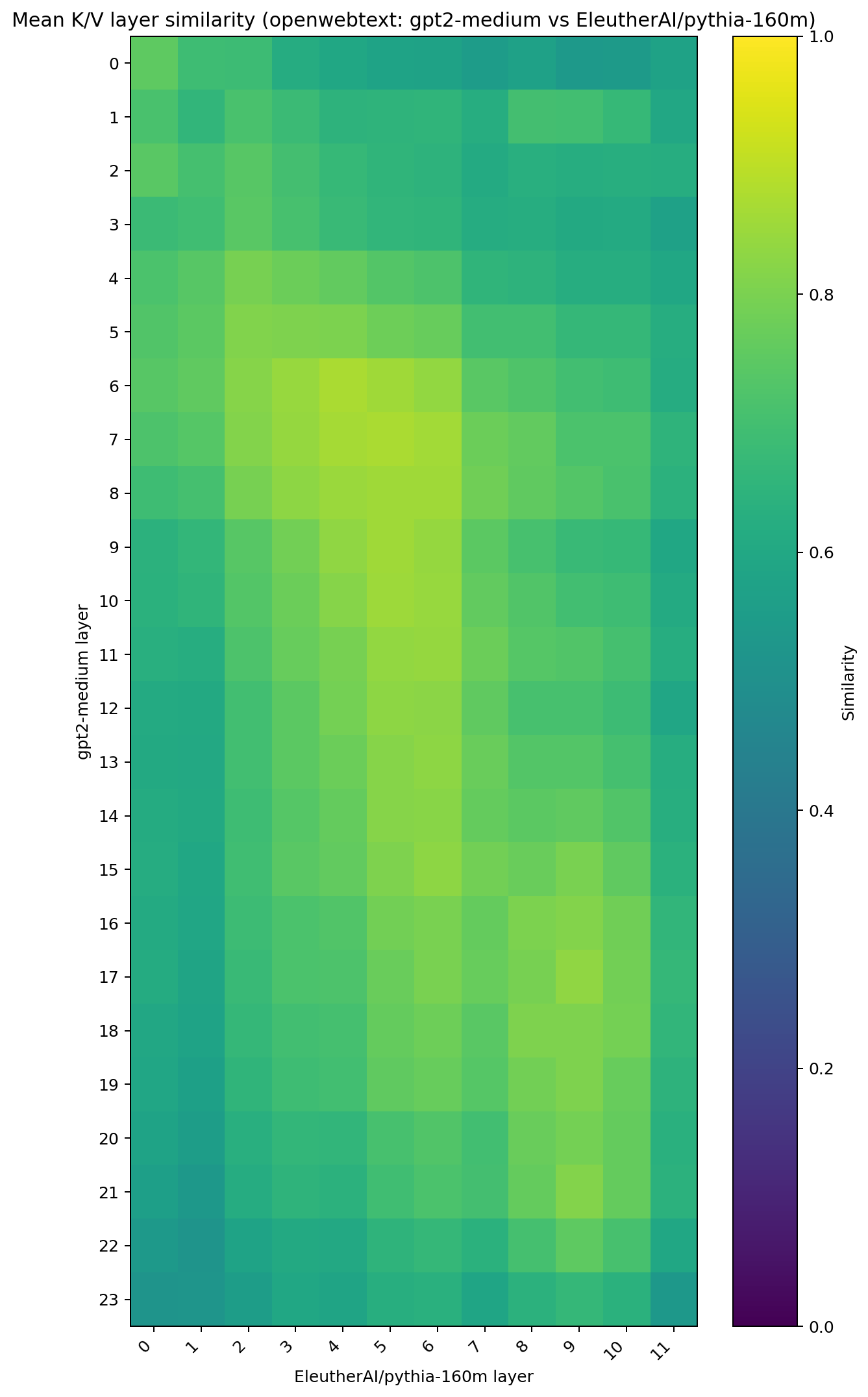}
        {\small (b) \texttt{gpt2-medium}\newline↔\,\texttt{pythia-160m}\par}
    \end{minipage}
    \hfill
    \begin{minipage}[t]{0.19\textwidth}
        \centering
        \includegraphics[width=3cm,height=4cm]{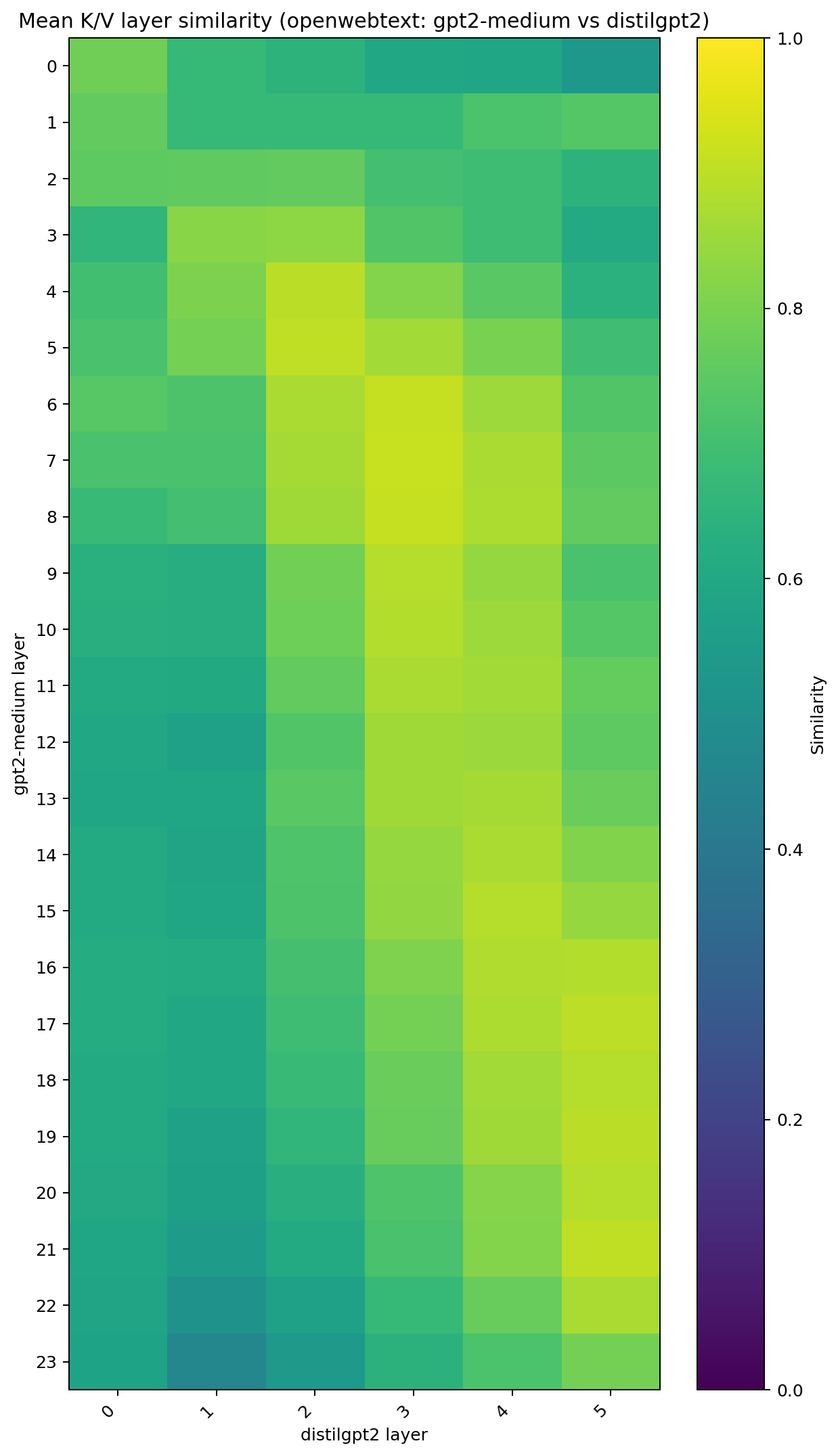}
        {\small (c) \texttt{gpt2-medium}\newline↔\,\texttt{distilgpt2}\par}
    \end{minipage}
    \hfill
    \begin{minipage}[t]{0.19\textwidth}
        \centering
        \includegraphics[width=3cm,height=4cm]{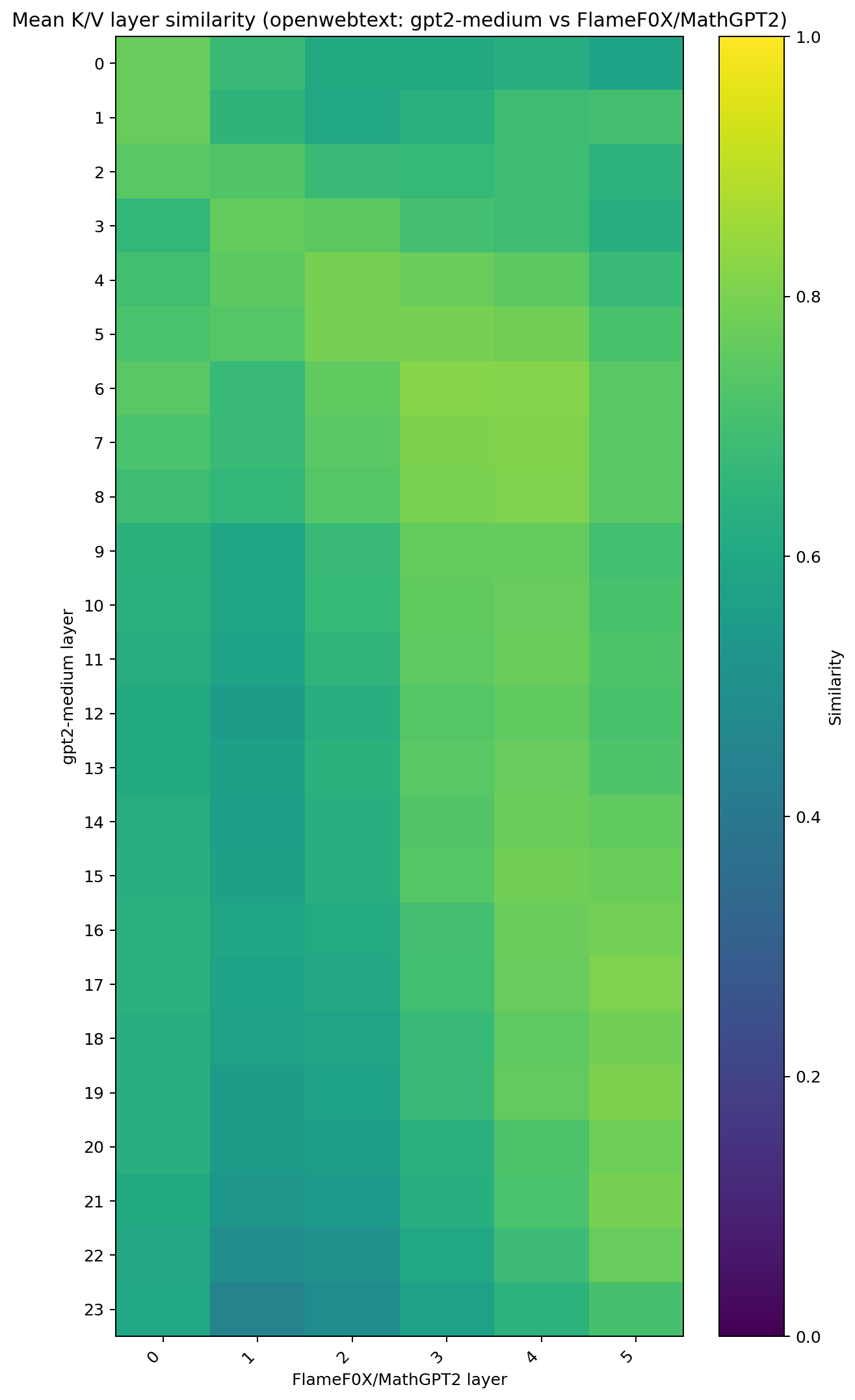}
        {\small (d) \texttt{gpt2-medium}\newline↔\,\texttt{MathGPT2}\par}
    \end{minipage}
    \hfill
    \begin{minipage}[t]{0.19\textwidth}
        \centering
        \includegraphics[width=3cm,height=4cm]{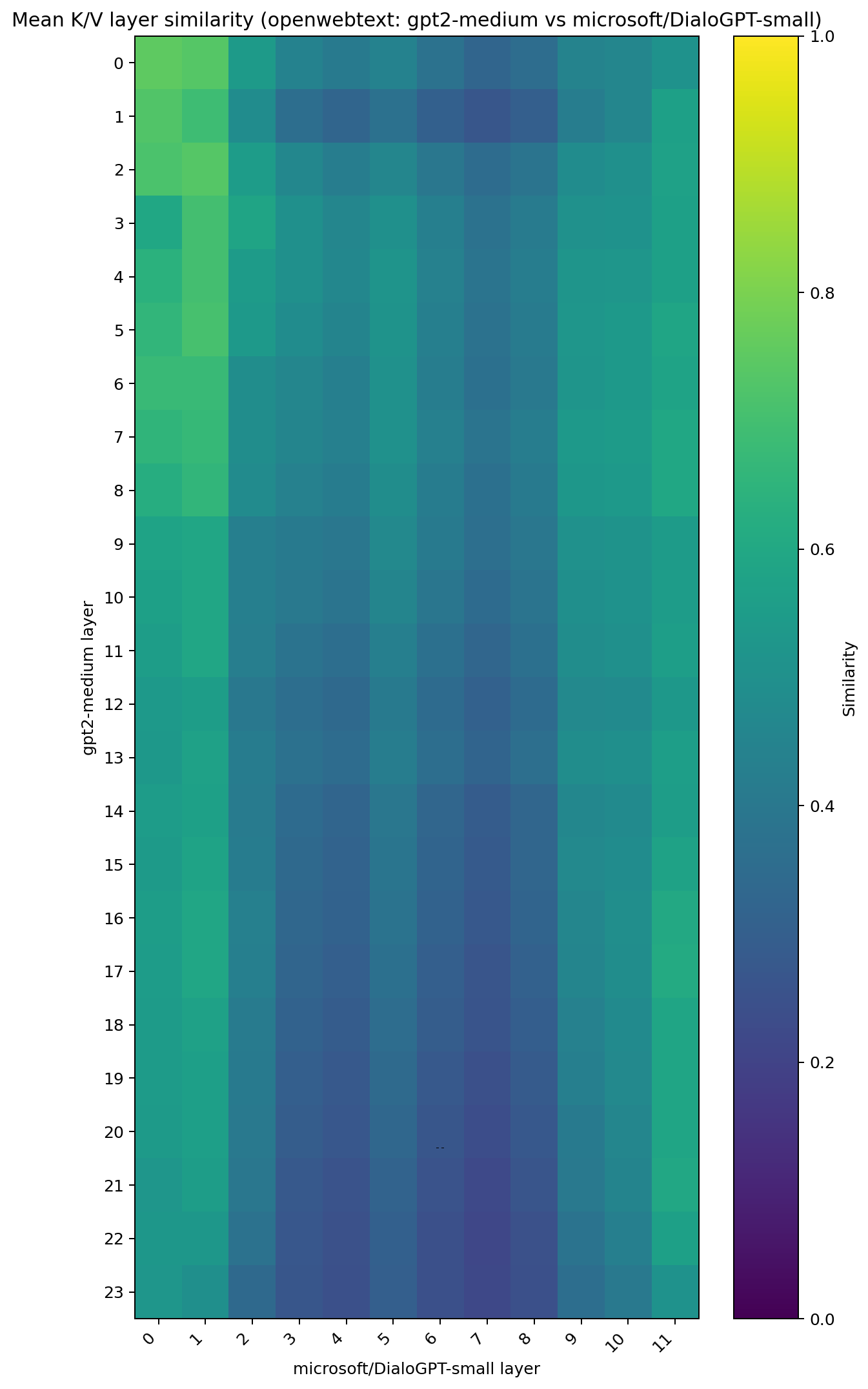}
        {\small (e) \texttt{gpt2-medium}\newline↔\,\texttt{DialoGPT-small}\par}
    \end{minipage}

    \caption{Preliminary layer-mapping analysis.
    Layer-wise similarity between heterogeneous model architectures, computed from
    the mean key/value cache representations obtained on the same OpenWebText
    samples.}
    \label{fig:layer_mapping}
\end{figure*}

Figures~\ref{fig:layer_mapping}(a)--(d) show that, even when the source and target
models have different numbers of layers, high similarity often appears between
layers with similar relative depths. In other words, lower-depth source layers
tend to form mean KV representations that are more similar to lower-depth target
layers, while higher-depth source layers are more similar to higher-depth target
layers. This trend appears as a spectral diagonal pattern in the similarity
matrix, suggesting that relative depth provides a useful first-order principle
for constructing channel mappings.

In particular, panel (b) of Figure~\ref{fig:layer_mapping} shows the result for
\texttt{gpt2-medium} and \texttt{pythia-160m}. Although the two models do not
share the same training data or architecture family, they still exhibit high
similarity between layer pairs with similar depth ratios. This suggests that the
depth-ratio relationship is not merely a phenomenon within the same model family,
but may also hold under more general forms of model heterogeneity.

By contrast, Figure~\ref{fig:layer_mapping}(e) illustrates an exceptional case.
When models differ substantially in their instruction-following behavior or
training data distribution, the spectral diagonal pattern can become less
pronounced. In such cases, a fixed depth-ratio rule may not always yield the
optimal channel mapping. We therefore treat depth-ratio mapping as a simple and
effective default strategy, while leaving adaptive channel mapping under stronger
data or instruction mismatch as future work.

Motivated by this preliminary analysis, we use Depth-Ratio Mapping~\citep{c2c} as
the default channel-mapping rule. Intuitively, even when two models differ, lower
layers tend to encode lower-level representations, whereas higher layers tend to
encode higher-level representations. Therefore, it is natural to connect layer
pairs that have similar relative depths.

Depth-Ratio Mapping connects source layers and target layers in order according
to their relative depth. That is, it aligns the source model from bottom to top
with the target model from bottom to top, so that lower-depth source layers are
connected to lower-depth target layers, and higher-depth source layers are
connected to higher-depth target layers. Let the resulting ordered channel sequence be
\(\mathcal C_{\mathrm{DR}}=(c_1,c_2,\dots,c_m)\).
Here, each \(c_r\) denotes the \(r\)-th depth-ratio layer pair, and
\(m=\min(L^{\mathrm{Src}},L^{\mathrm{Tgt}})\)
is the length of the sequence.

Given \(|\mathcal C|\), we do not search over all possible subsets of
\(\mathcal C_{\mathrm{DR}}\). Instead, we use only contiguous channel windows as
candidates:
\[
\mathfrak C_{\mathrm{DR}}(|\mathcal C|)
=
\left\{
\{c_q,c_{q+1},\dots,c_{q+|\mathcal C|-1}\}
\,\middle|\,
1\le q\le m-|\mathcal C|+1
\right\}.
\]
This design is motivated by Proposition~\ref{prop:source_propagation_error} and
Proposition~\ref{prop:correction_deficit_coefficient_bound}, which show that
translating at too early layers can increase propagation error, whereas
translating at too late layers can reduce correction opportunity. Therefore,
instead of selecting sparse and disconnected channels, choosing a single
contiguous window allows us to move the start and end points of the window
together and more stably find a balance between the two error dynamics.

The final channel set is selected as the one that minimizes the validation loss
among the sliding-window candidates. However, training the full \mot for every
candidate channel window is computationally expensive. Therefore, during the
channel optimization stage, we use a single backbone translator as a proxy to
quickly evaluate the validation loss of each candidate. Since this proxy uses the
same backbone structure as the translators inside \mot, it is suitable for
measuring the translation difficulty induced by channel location. The detailed
structure of the backbone translator is provided in
Appendix~\ref{app:backbone-translator}:
\[
\mathcal C_{\star}^{\mathrm{Src}\to\mathrm{Tgt}}
=
\arg\min_{\mathcal C'\in\mathfrak C_{\mathrm{DR}}(|\mathcal C|)}
\mathcal L_{\mathrm{Val}}^{\mathrm{Src}\to\mathrm{Tgt}}(\mathcal C').
\]

\subsection{Post-Translation Phase: Cache Replay}
\label{app:cache-replay}

Immediately after translation, we need to construct a cache replay procedure so that the selected translation result can be used in the actual long-context setting. After the translation phase, the source model's context cache has been converted into the cache space of the target model, and this translated result should be usable as the actual context cache of the target model. However, the translated cache provides only the context \(\bK\bV\) cache corresponding to the target translation layer set \(\mathcal I^{\mathrm{Tgt}}\). Therefore, in order to obtain a target-side context \(\bK\bV\) cache that can be used consistently during subsequent prompt prefill and completion generation, a cache replay procedure that passes the context through the target model again is required.

The most direct approach is to perform full attention at each target layer \(\ell\), while injecting the translated \(\bK\bV\) cache instead of the native \(\bK\bV\) cache only at the translated layers. Let \(N_{\mathrm{Ctx}}\) denote the context length, \(L^{\mathrm{Tgt}}\) the number of target layers, and \(d\) the attention dimension. Since full-attention-based replay allows each context query token to attend to all previous context key-value tokens, its computational complexity is approximately
\(\mathcal O(L^{\mathrm{Tgt}}N_{\mathrm{Ctx}}^2d)\).
In particular, replay is performed before prompt prefill in order to reconstruct the target-side context cache. Thus, for long contexts, the \(N_{\mathrm{Ctx}}^2\) term dominates the cost, and the need to pass through multiple target layers makes this procedure expensive.

To reduce this cost, we use \textsf{Top-\(S\) Source-Guided Sparse Attention}. Here, \(S_{\mathrm{Sparse}}\) denotes the number of source-guided key positions to be attended to at each context query position. The key idea is that, during target replay, we do not attend to all context key-value tokens again. Instead, we select only the token positions that the source model regarded as important in the given context and reuse them for target attention. In other words, the selection criterion for sparse attention is not newly learned inside the target model, but is derived from the attention pattern of the source model.

Let the context \(\bK\bV\) cache used for replay at target layer \(\ell\) be denoted by
\(\overline{\bK\bV}^{\mathrm{Tgt},\mathrm{Ctx}}_{\ell}
=
(\bar{\bK}^{\mathrm{Tgt},\mathrm{Ctx}}_{\ell},
\bar{\bV}^{\mathrm{Tgt},\mathrm{Ctx}}_{\ell})\).
It is defined as
\[
\overline{\bK\bV}^{\mathrm{Tgt},\mathrm{Ctx}}_{\ell}
=
\begin{dcases}
\widehat{\bK\bV}^{\mathrm{Tgt},\mathrm{Ctx}}_{\ell},
&
\ell \in \mathcal I^{\mathrm{Tgt}},
\\
\bK\bV^{\mathrm{Tgt},\mathrm{Ctx}}_{\ell},
&
\ell \notin \mathcal I^{\mathrm{Tgt}}.
\end{dcases}
\]
That is, at the target translation layers, we inject the translated context \(\bK\bV\) cache, while at the remaining layers, we use the context \(\bK\bV\) cache computed by the native replay process of the target model.

Full-attention replay computes, for each context token position \(u\in\{1,\dots,N_{\mathrm{Ctx}}\}\),
\[
\bo^{\mathrm{Full}}_{\ell,u}
=
\operatorname{Attn}
\left(
\bQ^{\mathrm{Tgt}}_{\ell,u},
\bar{\bK}^{\mathrm{Tgt},\mathrm{Ctx}}_{\ell,1:u},
\bar{\bV}^{\mathrm{Tgt},\mathrm{Ctx}}_{\ell,1:u}
\right).
\]
Thus, each context query attends to the entire causal context
\(\{1,\dots,u\}\).

In contrast, Top-\(S\) Source-Guided Sparse Attention uses
\(\pi(\ell)\) to denote the source guidance layer corresponding to target layer \(\ell\).
Here, \(\pi(\ell)\) is a source layer map used to guide the sparse attention pattern of target layer \(\ell\). Let the attention score used as context replay guidance at source layer \(\pi(\ell)\) be
\(\bA^{\mathrm{Src}}_{\pi(\ell)}
\in\mathbb R^{N_{\mathrm{Ctx}}\times N_{\mathrm{Ctx}}}\).
We omit the head index, and when necessary, interpret this as an attention map averaged over heads. For each context query position \(u\), we define the \(S_{\mathrm{Sparse}}\) context key positions most strongly attended to by the source attention as
\[
\mathcal N_{\ell}(u)
=
\operatorname{Top-S}
\left(
\bA^{\mathrm{Src}}_{\pi(\ell),u,1:u},
S_{\mathrm{Sparse}}
\right).
\]
Then, the attention in target replay attends only to the selected source-guided sparse context token set, rather than to the entire context:
\[
\bo^{\mathrm{Sparse}}_{\ell,u}
=
\operatorname{Attn}
\left(
\bQ^{\mathrm{Tgt}}_{\ell,u},
\bar{\bK}^{\mathrm{Tgt},\mathrm{Ctx}}_{\ell,\mathcal N_{\ell}(u)},
\bar{\bV}^{\mathrm{Tgt},\mathrm{Ctx}}_{\ell,\mathcal N_{\ell}(u)}
\right).
\]
In other words, the full-attention reference set \(\{1,\dots,u\}\)
is replaced by the source-guided Top-\(S\) context position set
\(\mathcal N_{\ell}(u)\).

When \(S_{\mathrm{Sparse}}\ll N_{\mathrm{Ctx}}\), the attention cost of each target layer is reduced from
\(\mathcal O(N_{\mathrm{Ctx}}^2d)\) to
\(\mathcal O(N_{\mathrm{Ctx}}S_{\mathrm{Sparse}}d)\).
Accordingly, the overall replay cost is reduced approximately as
\[
\mathcal O\!\left(L^{\mathrm{Tgt}}N_{\mathrm{Ctx}}^2d\right)
\quad\longrightarrow\quad
\mathcal O\!\left(L^{\mathrm{Tgt}}N_{\mathrm{Ctx}}S_{\mathrm{Sparse}}d\right).
\]
In implementation, for stability, some bottom layers can be kept as full-attention layers, while source-guided sparse attention is applied only to the remaining layers. Let \(L_{\mathrm{Full}}\) denote the number of bottom full-attention layers. Then, the attention reference set is defined as
\[
\mathcal N_{\ell}(u)
=
\begin{dcases}
\{1,\dots,u\},
&
\ell < L_{\mathrm{Full}},
\\
\operatorname{Top-S}
\left(
\bA^{\mathrm{Src}}_{\pi(\ell),u,1:u},
S_{\mathrm{Sparse}}
\right),
&
\ell \ge L_{\mathrm{Full}}.
\end{dcases}
\]

Overall, the cache replay procedure in the post-translation phase can be summarized as follows. First, the translated target context \(\bK\bV\) cache is injected into the target translation layers, while the remaining layers are reconstructed through the target model's native computation. Standard full-attention replay attends to all context tokens and therefore incurs a quadratic cost in \(N_{\mathrm{Ctx}}\). In contrast, the proposed Top-\(S\) Source-Guided Sparse Attention reduces the replay cost to nearly linear in the context length by attending only to the \(S_{\mathrm{Sparse}}\) context tokens selected by the source attention pattern. The target model then performs prompt prefill conditioned on the target-side context cache obtained through replay and finally generates the completion.

\section{Benchmark Prompts}
\label{app:benchmark-prompts}

This section summarizes the prompt format used for each benchmark.
The main goal of our evaluation is to measure whether a target model can use a long input through a translated cache, instead of directly prefilling the full context.
Therefore, each prompt is designed to clearly separate long information from the short query.
Long and difficult input, such as a passage, abstract, article, or multi-document evidence, is placed in the context region.
The short question, answer cue, or choice instruction required by the target model is placed in the prompt region.
In the boxes below, the red context region denotes the part subject to cache translation, while the blue prompt and completion regions denote the part where the target model receives a short query and generates the answer.

The benchmarks are divided into two groups: closed-set QA and extractive QA.
Closed-set QA benchmarks have a fixed answer space, and the generated answer is matched to the valid answer set to compute accuracy.
Extractive QA benchmarks require the model to generate a short answer grounded in the passage, and the output is evaluated by token-level F1 against gold answers.

\subsection{Closed-set QA benchmarks}

We use BoolQ, PubMedQA, and MMLU-Redux as closed-set QA benchmarks.
Each dataset is represented with three regions: context, prompt, and completion.
The context contains the long input information, the prompt contains the short question or choice instruction, and the completion contains the generated answer expected to match one of the valid candidates.
The model is evaluated by accuracy after answer matching.

\begin{bluebox}[BoolQ prompt with sample]{width=\linewidth, valign=top, breakable}
\label{text:prompt_boolq_sample}

\noindent\textbf{Dataset.}
BoolQ is a binary QA benchmark where the model reads a passage and answers with either yes or no.

\vspace{0.8em}

\begin{translatedcontextbox}
{\ttfamily\small
Read the passage and answer the question with yes or no only.

\vspace{0.6em}
Passage: \(\langle\)Passage\(\rangle\)
}
\end{translatedcontextbox}

\begin{promptbox}
{\ttfamily\small
Question: \(\langle\)Question\(\rangle\)\\
Answer:
}
\end{promptbox}

\begin{completionbox}
{\ttfamily\small
yes \quad or \quad no
}
\end{completionbox}

\tcbline
\noindent\textbf{Sample.}

\vspace{0.6em}

\begin{translatedcontextbox}
{\ttfamily\small
Read the passage and answer the question with yes or no only.

\vspace{0.6em}
Passage: Windows Movie Maker is a discontinued video editing software by Microsoft. It is a part of Windows Essentials software suite and offers the ability to create and edit videos as well as to publish them on OneDrive, Facebook, Vimeo, YouTube, and Flickr.
}
\end{translatedcontextbox}

\begin{promptbox}
{\ttfamily\small
Question: is windows movie maker part of windows essentials\\
Answer:
}
\end{promptbox}

\begin{completionbox}
{\ttfamily\small
yes
}
\end{completionbox}

\end{bluebox}

\begin{bluebox}[PubMedQA prompt with sample]{width=\linewidth, valign=top, breakable}
\label{text:prompt_pubmedqa_sample}

\noindent\textbf{Dataset.}
PubMedQA is a biomedical QA benchmark where the model reads an abstract and answers with yes, no, or maybe.

\vspace{0.8em}

\begin{translatedcontextbox}
{\ttfamily\small
Read the abstract and answer the biomedical research question with yes, no, or maybe only.

\vspace{0.6em}
Abstract: \(\langle\)Abstract\(\rangle\)
}
\end{translatedcontextbox}

\begin{promptbox}
{\ttfamily\small
Question: \(\langle\)Question\(\rangle\)\\
Answer:
}
\end{promptbox}

\begin{completionbox}
{\ttfamily\small
yes \quad or \quad no \quad or \quad maybe
}
\end{completionbox}

\tcbline
\noindent\textbf{Sample.}

\vspace{0.6em}

\begin{translatedcontextbox}
{\ttfamily\small
Read the abstract and answer the biomedical research question with yes, no, or maybe only.

\vspace{0.6em}
Abstract: Halofantrine is a newly developed antimalarial drug used for the treatment of Plasmodium falciparum malaria. The introduction of this drug has been delayed because of its possible side effects, and due to insufficient studies on adverse reactions in humans. \(\ldots\) Halofantrine has mild to moderate pathological effects on cochlea histology, and can be considered an ototoxic drug.
}
\end{translatedcontextbox}

\begin{promptbox}
{\ttfamily\small
Question: Is halofantrine ototoxic?\\
Answer:
}
\end{promptbox}

\begin{completionbox}
{\ttfamily\small
yes
}
\end{completionbox}

\end{bluebox}

\begin{bluebox}[MMLU-Redux prompt with sample]{width=\linewidth, valign=top, breakable}
\label{text:prompt_mmlu_redux_sample}

\noindent\textbf{Dataset.}
MMLU-Redux is a multiple-choice QA benchmark with subject-specific questions.
The model generates one answer label from four candidates.

\vspace{0.8em}

\begin{translatedcontextbox}
{\ttfamily\small
Subject: \(\langle\)Subject\(\rangle\)\\
Question: \(\langle\)Question\(\rangle\)
}
\end{translatedcontextbox}

\begin{promptbox}
{\ttfamily\small
Choices:\\
(A) \(\langle\)Choice 1\(\rangle\)\\
(B) \(\langle\)Choice 2\(\rangle\)\\
(C) \(\langle\)Choice 3\(\rangle\)\\
(D) \(\langle\)Choice 4\(\rangle\)\\
Answer:
}
\end{promptbox}

\begin{completionbox}
{\ttfamily\small
(A) \quad or \quad (B) \quad or \quad (C) \quad or \quad (D)
}
\end{completionbox}

\tcbline
\noindent\textbf{Sample.}

\vspace{0.6em}

\begin{translatedcontextbox}
{\ttfamily\small
Subject: abstract algebra\\
Question: Find all \(c\) in \(\mathbb{Z}_3\) such that \(\mathbb{Z}_3[x]/(x^3 + cx^2 + 1)\) is a field.
}
\end{translatedcontextbox}

\begin{promptbox}
{\ttfamily\small
Choices:\\
(A) 0\\
(B) 2\\
(C) 1\\
(D) 3\\
Answer:
}
\end{promptbox}

\begin{completionbox}
{\ttfamily\small
(B)
}
\end{completionbox}

\end{bluebox}

\subsection{Extractive QA benchmarks}

The extractive QA benchmarks include SQuAD-v1.1 and NewsQA.
Both datasets use the same context--prompt--completion structure.
A passage is placed in the context region, while a short question is placed in the prompt region.
The model generates a short answer, and the output is evaluated by token-level F1 against the gold answers.

Under this abstraction, SQuAD-v1.1 and NewsQA evaluate whether the translated context cache preserves the evidence needed for short-answer extraction.
Although the model produces the answer autoregressively, the task is not treated as open-ended generation; instead, it is evaluated as extractive QA.

\begin{bluebox}[Unified extractive QA prompt]{width=\linewidth, valign=top, breakable}
\label{text:prompt_unified_extractive_qa}

\noindent\textbf{Datasets.}
SQuAD-v1.1 and NewsQA use the same short-answer extractive QA prompt format.
HotpotQA-E in the long-context CAG case study follows the same format with a multi-document context.

\vspace{0.8em}

\begin{translatedcontextbox}
{\ttfamily\small
Read the passage and answer the question briefly. Use a short phrase from the passage when possible.

\vspace{0.6em}
Passage: \(\langle\)Passage / article / multi-document evidence\(\rangle\)
}
\end{translatedcontextbox}

\begin{promptbox}
{\ttfamily\small
Question: \(\langle\)Question\(\rangle\)\\
Answer:
}
\end{promptbox}

\begin{completionbox}
{\ttfamily\small
\(\langle\)Short answer\(\rangle\)
}
\end{completionbox}

\end{bluebox}

\begin{bluebox}[SQuAD-v1.1 sample]{width=\linewidth, valign=top, breakable}
\label{text:prompt_squad_sample}

\noindent\textbf{Dataset.}
SQuAD-v1.1 is an extractive QA benchmark where the model reads a single Wikipedia passage and generates a short answer grounded in the passage.

\vspace{0.8em}

\begin{translatedcontextbox}
{\ttfamily\small
Read the passage and answer the question briefly. Use a short phrase from the passage when possible.

\vspace{0.6em}
Passage: Architecturally, the school has a Catholic character. Atop the Main Building's gold dome is a golden statue of the Virgin Mary. Immediately in front of the Main Building and facing it, is a copper statue of Christ with arms upraised with the legend ``Venite Ad Me Omnes''. Next to the Main Building is the Basilica of the Sacred Heart. Immediately behind the basilica is the Grotto, a Marian place of prayer and reflection. \(\ldots\)
}
\end{translatedcontextbox}

\begin{promptbox}
{\ttfamily\small
Question: To whom did the Virgin Mary allegedly appear in 1858 in Lourdes France?\\
Answer:
}
\end{promptbox}

\begin{completionbox}
{\ttfamily\small
Saint Bernadette Soubirous
}
\end{completionbox}

\end{bluebox}

\begin{bluebox}[NewsQA sample]{width=\linewidth, valign=top, breakable}
\label{text:prompt_newsqa_sample}

\noindent\textbf{Dataset.}
NewsQA is an extractive QA benchmark where the model reads a news article and generates a short answer grounded in the article.

\vspace{0.8em}

\begin{translatedcontextbox}
{\ttfamily\small
Read the passage and answer the question briefly. Use a short phrase from the passage when possible.

\vspace{0.6em}
Passage: City officials said the new transportation plan will expand bus routes, add protected bike lanes, and increase late-night service. The first phase of construction is expected to begin next spring and will focus on downtown corridors. Supporters argue that the plan will reduce congestion, while critics are concerned about construction delays and costs.
}
\end{translatedcontextbox}

\begin{promptbox}
{\ttfamily\small
Question: When is the first phase of construction expected to begin?\\
Answer:
}
\end{promptbox}

\begin{completionbox}
{\ttfamily\small
next spring
}
\end{completionbox}

\end{bluebox}

\section{Additional Results}
\label{app:additional_results}

\subsection{Translation Taxonomy}
\label{app:translation-taxonomy}

We organize cache translation using a three-level \textbf{translation taxonomy}
based on the relationship between the source and target models. We use
\textbf{homogeneous translation} when the source and target correspond to the
same model configuration, and therefore share the same model architecture and
native KV-cache format. We use \textbf{intra-family heterogeneous translation}
when the source and target are distinct model configurations within the same
model family, such as different model scales whose depth, hidden dimension,
attention configuration, or KV-cache shape may differ. Finally, we use
\textbf{cross-family heterogeneous translation} when the source and target
belong to different model families.

Accordingly, heterogeneous translation is an umbrella term that includes both
intra-family and cross-family heterogeneous translations. For brevity, the main
paper uses \emph{heterogeneous translation} when discussing its primary
evaluation across \Llama, \Gemma, and \Qwen, while this appendix uses the full
translation taxonomy when distinguishing the additional experiments. For example,
\qwen~\textsf{(7B)}$\to$\qwen~\textsf{(0.5B)} and
\gpt~\textsf{(XL)}$\to$\gpt~\textsf{(Large)} are intra-family heterogeneous
translations, whereas translations between \Llama, \Gemma, and \Qwen are
cross-family heterogeneous translations.

\subsection{Implementation of Compared Methods}
\label{app:baseline-implementations}

For \ctoc \footnote{\url{https://github.com/thu-nics/C2C}}, \interlat \footnote{\url{https://github.com/XiaoDu-flying/Interlat}}, and \kvcomm \footnote{\url{https://github.com/Zephyroam/KVComm}}, we use their official open-source
implementations.
We integrate these implementations into a common experimental interface for
training and evaluation parameters and, where applicable, multi-agent cache
management. Since no official open-source implementation of \lsc was available,
we implemented \lsc based on the corresponding paper.

For the multi-agent case study, we use the official MALLM implementation%
\footnote{\url{https://github.com/Multi-Agent-LLMs/mallm}} and integrate our cache
translation and cache-management modes into its discussion pipeline.

\subsection{Translator Training Configuration}
\label{app:translator-training-configuration}

All trainable translator-based methods are trained on OpenWebText using the same common optimization setup. We use sequences of 128 tokens, with a 64-token context region followed by a 64-token suffix used for prompt-side language-model supervision; method-specific auxiliary objectives may additionally supervise the translated context or communicated representation. Unless otherwise specified, translators are trained for 500 optimizer steps with a per-micro-batch size of 4 and gradient accumulation over 4 micro-batches, corresponding to an effective batch size of 16 per model stream. We use AdamW with learning rate \(10^{-4}\), weight decay \(0.01\), 50 warmup steps followed by cosine decay, gradient clipping at \(1.0\), an OpenWebText shuffle buffer of 50K examples, random seed 42, and \texttt{bfloat16} training. These shared hyperparameters are held fixed across the trainable methods. For the baselines, method-specific settings follow the corresponding official implementations where available, or the original paper when no official implementation is available. MoT-specific hyperparameters are selected through the ablations reported in Section~\ref{sec:ablation}.

\paragraph{Compute Worker.}
All experiments were conducted on a compute worker equipped with two NVIDIA H100 80GB GPUs. Each experiment used only a single H100 GPU; the two GPUs were not combined for any experiment. The same GPU configuration was used across all experiments. Runtime and GPU memory measurements were collected with each experiment running as the sole experiment process on the GPU.

\subsection{Homogeneous and Intra-Family Heterogeneous Translations}
\label{app:homogeneous_intra_family}

\begin{table}[t]
\caption{
Homogeneous and intra-family heterogeneous \qwen translation results.
The unshaded \native row denotes the target model without translation; shaded rows denote translation methods. Darker green indicates higher recovery relative to \native.
}
\label{tab:qwen25_combined}
\centering
\resizebox{\textwidth}{!}{%
\begin{tabular}{c|c|c|c c c c|c c c|c}
\toprule
\multirow[c]{2}{*}{\textbf{Translation Category}} &
\multirow[c]{2}{*}{\textbf{Method}} &
\multirow[c]{2}{*}{\shortstack{\textbf{Cosine}\\\textbf{Sim}}} &
\multicolumn{4}{c|}{\textbf{Closed-set QA (Accuracy)}} &
\multicolumn{3}{c|}{\textbf{Extractive QA (F1)}} &
\multirow[c]{2}{*}{\shortstack{\textbf{Peak Mem.}\\\textbf{(GiB)}}} \\
\cmidrule(lr){4-7}
\cmidrule(lr){8-10}
&
&
&
\textbf{BoolQ} &
\textbf{PubMedQA} &
\textbf{MMLU-Redux} &
\textbf{Avg} &
\textbf{SQuAD} &
\textbf{NewsQA} &
\textbf{Avg} &
\\
\midrule
\multirow{6}{*}{\shortstack{\textbf{Homogeneous}\\\texttt{0.5B} $\to$ \texttt{0.5B}}}
& \nativecolor \native & \nativecolor N/A & \nativecolor 66.0 & \nativecolor 56.0 & \nativecolor 35.0 & \nativecolor 52.0 & \nativecolor 0.56 & \nativecolor 0.34 & \nativecolor 0.45 & \nativecolor 1.92 \\
& \transbg \lsc & \transbg 0.01 & \transbg 6.0 & \transbg 20.0 & \transbg 23.0 & \transbg 16.0 & \transbg 0.06 & \transbg 0.01 & \cellcolor[HTML]{E7EDF2}0.04 & \transbg 4.09 \\
& \transbg \ctoc & \transbg 0.02 & \transbg 39.0 & \transbg 22.0 & \transbg 23.0 & \cellcolor[HTML]{DAE6EC}28.0 & \transbg 0.04 & \transbg 0.01 & \transbg 0.03 & \transbg 1.92 \\
& \transbg \interlat & \transbg 0.98 & \transbg 43.0 & \transbg 42.0 & \transbg 19.0 & \cellcolor[HTML]{A8CCD9}35.0 & \transbg 0.29 & \transbg 0.25 & \cellcolor[HTML]{80B8C9}0.27 & \transbg 1.95 \\
& \transbg \kvcomm & \transbg 0.98 & \transbg \textbf{62.0} & \transbg 46.0 & \transbg 27.0 & \cellcolor[HTML]{61A8BD}45.0 & \transbg 0.05 & \transbg 0.04 & \cellcolor[HTML]{E7EDF2}0.04 & \transbg \textbf{1.91} \\
\cmidrule(lr){2-11}
& \transbg \textbf{\mot} & \transbg \textbf{1.00} & \transbg 60.0 & \transbg \textbf{52.0} & \transbg \textbf{35.0} & \cellcolor[HTML]{449AB1}\textbf{49.0} & \transbg \textbf{0.54} & \transbg \textbf{0.30} & \cellcolor[HTML]{3C96AE}\textbf{0.42} & \transbg 3.29 \\
\midrule
\multirow{5}{*}{\shortstack{\textbf{Intra-Family}\\\textbf{Heterogeneous}\\\texttt{7B} $\to$ \texttt{0.5B}}}
& \nativecolor \native & \nativecolor N/A & \nativecolor 66.0 & \nativecolor 56.0 & \nativecolor 35.0 & \nativecolor 52.0 & \nativecolor 0.56 & \nativecolor 0.34 & \nativecolor 0.45 & \nativecolor 15.25 \\
& \transbg \lsc & \transbg 0.01 & \transbg 22.0 & \transbg 9.0 & \transbg 14.0 & \transbg 15.0 & \transbg 0.05 & \transbg 0.01 & \transbg 0.03 & \transbg 18.21 \\
& \transbg \ctoc & \transbg 0.02 & \transbg 47.0 & \transbg 18.0 & \transbg 24.0 & \cellcolor[HTML]{D5E3EB}30.0 & \transbg 0.05 & \transbg 0.01 & \cellcolor[HTML]{E8EDF2}0.03 & \transbg \textbf{15.25} \\
& \transbg \interlat & \transbg 0.98 & \transbg 41.0 & \transbg 44.0 & \transbg 20.0 & \cellcolor[HTML]{AFD0DC}35.0 & \transbg 0.29 & \transbg 0.25 & \cellcolor[HTML]{7EB7C8}0.27 & \transbg 15.27 \\
\cmidrule(lr){2-11}
& \transbg \textbf{\mot} & \transbg \textbf{1.00} & \transbg \textbf{65.0} & \transbg \textbf{53.0} & \transbg \textbf{35.0} & \cellcolor[HTML]{3793AC}\textbf{51.0} & \transbg \textbf{0.53} & \transbg \textbf{0.32} & \cellcolor[HTML]{3893AC}\textbf{0.43} & \transbg 16.62 \\
\bottomrule
\end{tabular}
}
\end{table}

Table~\ref{tab:qwen25_combined} provides supplementary results for homogeneous
and intra-family heterogeneous translation within \qwen.
We evaluate homogeneous
\qwen~\textsf{(0.5B)}$\to$\qwen~\textsf{(0.5B)} translation and intra-family
heterogeneous
\qwen~\textsf{(7B)}$\to$\qwen~\textsf{(0.5B)} translation.
The latter changes model depth, hidden dimension, attention configuration, and
KV-cache shape while retaining the same model family.

\mot remains stable across both translations.
For homogeneous translation, \mot achieves an average accuracy of \(49.0\%\)
and an average F1 of \(0.42\), while for intra-family heterogeneous translation
it achieves an average accuracy of \(51.0\%\) and an average F1 of \(0.43\).
The strong performance under \qwen~\textsf{(7B)}$\to$\qwen~\textsf{(0.5B)} indicates that \mot can
handle substantial cache-space mismatch even within a model family.

Among the baselines, \ctoc performs poorly on extractive QA, indicating that a
projection-only mapping is insufficient to reconstruct the target cache.
\interlat partially preserves closed-set QA performance but remains limited on
extractive QA, while \lsc also degrades substantially.
\kvcomm is evaluated only in the homogeneous setting because its cache-sharing
mechanism assumes homogeneous model compatibility.

\begin{figure*}[t]
    \centering

    \begin{subfigure}[t]{0.32\textwidth}
        \centering
        \includegraphics[width=\textwidth,height=3.8cm]{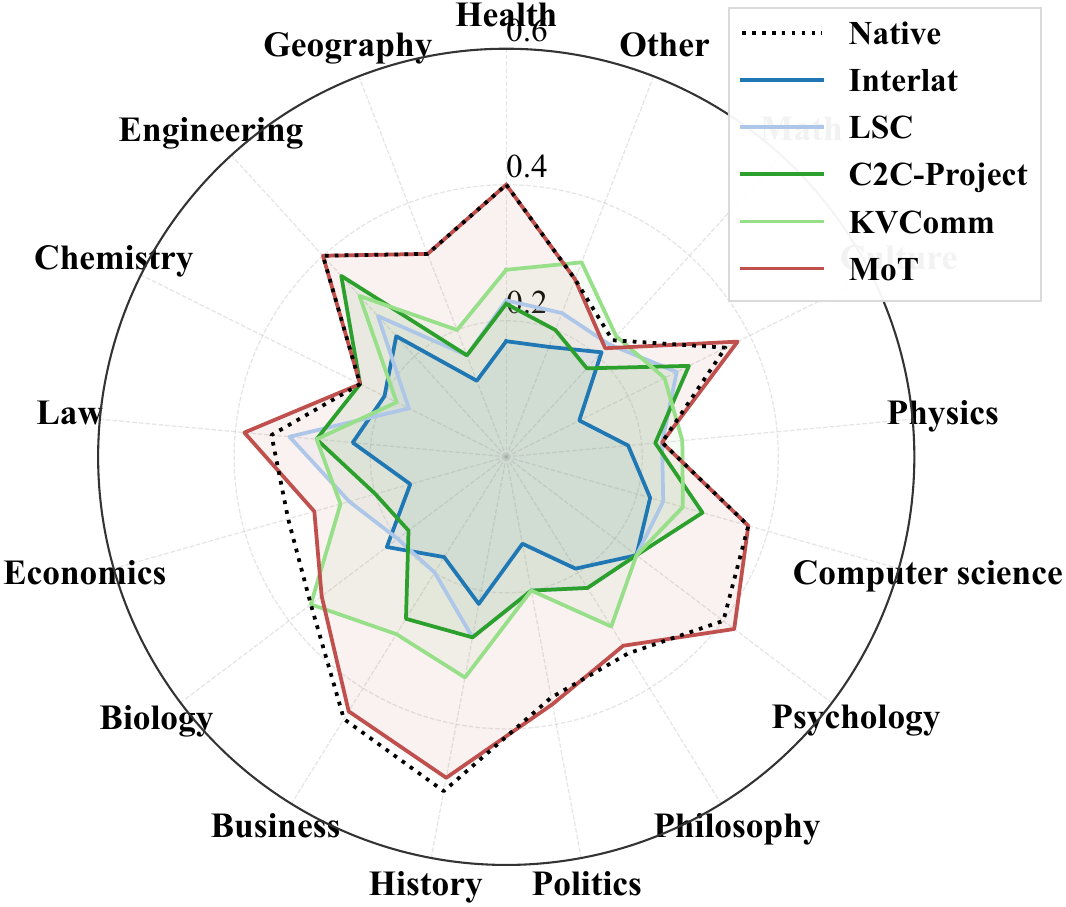}
        \caption{
        Homogeneous:\\
        \qwen~\textsf{(0.5B)}$\to$\qwen~\textsf{(0.5B)}.
        }
    \end{subfigure}
    \hfill
    \begin{subfigure}[t]{0.32\textwidth}
        \centering
        \includegraphics[width=\textwidth,height=3.8cm]{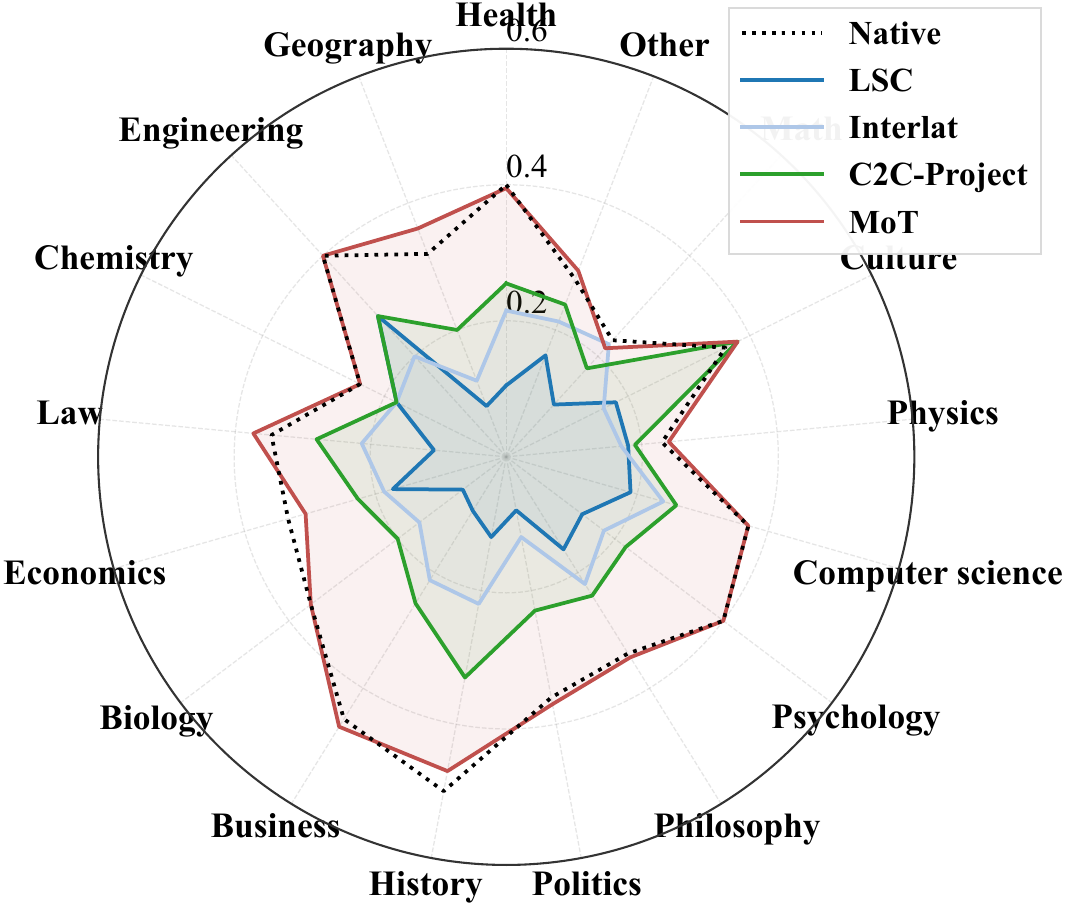}
        \caption{
        Intra-family heterogeneous:\\
        \qwen~\textsf{(7B)}$\to$\qwen~\textsf{(0.5B)}.
        }
    \end{subfigure}
    \hfill
    \begin{subfigure}[t]{0.32\textwidth}
        \centering
        \includegraphics[width=\textwidth,height=3.8cm]{figures/mmlu_cross_family.pdf}
        \caption{
        Cross-family heterogeneous:\\
        \Llama~\textsf{(3B)}$\to$\Qwen~\textsf{(4B)}.
        }
    \end{subfigure}

    \caption{
    MMLU-Redux category-wise results for 3 translation categories.
    }
    \label{fig:mmlu_comparison}
\end{figure*}

Figures~\ref{fig:mmlu_comparison} further breaks down the MMLU-Redux results by
17 categories. \mot remains consistently close to the target-model reference
across a broad range of categories, indicating that its performance does not
depend on isolated category-specific gains.
Across the three levels of the translation taxonomy, the task becomes progressively more challenging from homogeneous to intra-family heterogeneous and cross-family heterogeneous translation. Nevertheless, \mot consistently recovers performance close to \native across all three cases, whereas the other methods degrade as the translation difficulty increases. Consequently, the performance gap between \mot and the competing methods widens with increasing heterogeneity. In settings where translated-cache performance exceeds \native, one possible explanation is a regularization-like effect of cache translation, consistent with prior observations that selective KV communication can emphasize task-relevant information~\citep{kvcomm}.

\begin{table}[t]
\caption{
Homogeneous translation: \texttt{gpt2-large} $\to$ \texttt{gpt2-large}.
The unshaded \native row denotes the target model without translation; shaded rows denote translation methods. Darker green indicates higher recovery relative to \native.
}
\label{tab:gpt2-large_homogeneous}
\centering
\resizebox{\textwidth}{!}{%
\begin{tabular}{c|c|c c c c|c c c|c}
\toprule
\multirow[c]{2}{*}{\textbf{Method}} &
\multirow[c]{2}{*}{\shortstack{\textbf{Cosine}\\\textbf{Sim}}} &
\multicolumn{4}{c|}{\textbf{Closed-set QA (Accuracy)}} &
\multicolumn{3}{c|}{\textbf{Extractive QA (F1)}} &
\multirow[c]{2}{*}{\shortstack{\textbf{Peak Mem.}\\\textbf{(GiB)}}} \\
\cmidrule(lr){3-6}
\cmidrule(lr){7-9}
&
&
\textbf{BoolQ} &
\textbf{PubMedQA} &
\textbf{MMLU-Redux} &
\textbf{Avg} &
\textbf{SQuAD} &
\textbf{NewsQA} &
\textbf{Avg} &
\\
\midrule
\nativecolor \native & \nativecolor N/A & \nativecolor 51.0 & \nativecolor 36.0 & \nativecolor 17.0 & \nativecolor 35.0 & \nativecolor 0.50 & \nativecolor 0.29 & \nativecolor 0.40 & \nativecolor 7.06 \\
\transbg \lsc& \transbg 0.07 & \transbg 25.0 & \transbg 2.0 & \transbg \textbf{19.0} & \cellcolor[HTML]{DEE8EE}15.0 & \transbg 0.06 & \transbg 0.01 & \cellcolor[HTML]{E5ECF1}0.04 & \transbg 12.67 \\
\transbg \ctoc & \transbg 0.12 & \transbg 37.0 & \transbg 7.0 & \transbg 9.0 & \cellcolor[HTML]{C4DAE4}18.0 & \transbg 0.13 & \transbg 0.02 & \cellcolor[HTML]{D6E4EB}0.07 & \transbg 7.09 \\
\transbg \interlat & \transbg 0.63 & \transbg \textbf{65.0} & \transbg 38.0 & \transbg 15.0 & \cellcolor[HTML]{0C7D9B}39.0 & \transbg 0.16 & \transbg \textbf{0.04} & \cellcolor[HTML]{C7DCE5}0.10 & \transbg 7.18 \\
\transbg \kvcomm & \transbg 0.82 & \transbg 17.0 & \transbg 2.0 & \transbg 4.0 & \transbg 8.0 & \transbg 0.02 & \transbg 0.00 & \transbg 0.01 & \transbg \textbf{6.80} \\
\midrule
\transbg \textbf{\mot} & \transbg \textbf{0.89} & \transbg 62.0 & \transbg \textbf{45.0} & \transbg 18.0 & \cellcolor[HTML]{007091}\textbf{42.0} & \transbg \textbf{0.42} & \transbg 0.01 & \cellcolor[HTML]{8ABDCD}\textbf{0.22} & \transbg 13.54 \\
\bottomrule
\end{tabular}%
}
\end{table}

\begin{table}[t]
\caption{
Intra-family heterogeneous translation: \texttt{gpt2-xl} $\to$ \texttt{gpt2-large}.
}
\label{tab:gpt2_xl_gpt2-large_heterogeneous}
\centering
\resizebox{\textwidth}{!}{%
\begin{tabular}{c|c|c c c c|c c c|c}
\toprule
\multirow[c]{2}{*}{\textbf{Method}} &
\multirow[c]{2}{*}{\shortstack{\textbf{Cosine}\\\textbf{Sim}}} &
\multicolumn{4}{c|}{\textbf{Closed-set QA (Accuracy)}} &
\multicolumn{3}{c|}{\textbf{Extractive QA (F1)}} &
\multirow[c]{2}{*}{\shortstack{\textbf{Peak Mem.}\\\textbf{(GiB)}}} \\
\cmidrule(lr){3-6}
\cmidrule(lr){7-9}
&
&
\textbf{BoolQ} &
\textbf{PubMedQA} &
\textbf{MMLU-Redux} &
\textbf{Avg} &
\textbf{SQuAD} &
\textbf{NewsQA} &
\textbf{Avg} &
\\
\midrule
\nativecolor \native & \nativecolor N/A & \nativecolor 51.0 & \nativecolor 36.0 & \nativecolor 17.0 & \nativecolor 35.0 & \nativecolor 0.50 & \nativecolor 0.29 & \nativecolor 0.40 & \nativecolor 10.18 \\
\transbg \lsc & \transbg 0.10 & \transbg 35.0 & \transbg 4.0 & \transbg 17.0 & \transbg 19.0 & \transbg 0.07 & \transbg 0.00 & \transbg 0.03 & \transbg 18.67 \\
\transbg \ctoc & \transbg 0.11 & \transbg 45.0 & \transbg 14.0 & \transbg 13.0 & \cellcolor[HTML]{DBE6ED}24.0 & \transbg 0.13 & \transbg \textbf{0.02} & \cellcolor[HTML]{DFE8EE}0.08 & \transbg \textbf{10.21} \\
\transbg \interlat & \transbg 0.58 & \transbg 55.0 & \transbg 36.0 & \transbg 18.0 & \cellcolor[HTML]{1F87A3}36.0 & \transbg 0.12 & \transbg 0.02 & \cellcolor[HTML]{E4EBF1}0.07 & \transbg 10.36 \\
\midrule
\transbg \textbf{\mot} & \transbg \textbf{0.89} & \transbg \textbf{61.0} & \transbg \textbf{46.0} & \transbg \textbf{19.0} & \cellcolor[HTML]{00577E}\textbf{42.0} & \transbg \textbf{0.43} & \transbg 0.02 & \cellcolor[HTML]{92C1D0}\textbf{0.22} & \transbg 16.66 \\
\bottomrule
\end{tabular}%
}
\end{table}

Table~\ref{tab:gpt2-large_homogeneous} reports the homogeneous translation result for
\texttt{gpt2-large}$\to$\texttt{gpt2-large}. Among the translation methods, \mot
achieves the strongest overall performance, obtaining the highest average
accuracy of \(42.0\%\) and the highest average F1 of \(0.22\). It also achieves the highest
cosine similarity of 0.89, indicating that the translated cache is better aligned
with the target cache space than those produced by the baselines. In closed-set
QA, \mot obtains the best PubMedQA score and remains competitive on BoolQ and
MMLU-Redux, while in extractive QA it substantially improves SQuAD over the other
translation methods. Although \interlat achieves the best BoolQ score, its
extractive QA performance remains limited, suggesting that last-hidden-state
communication does not sufficiently preserve the token-level evidence needed for
answer extraction. Similarly, \ctoc, \kvcomm, and \lsc show limited F1,
indicating that projection-only mapping, selective KV sharing, or global
shared-space alignment is insufficient for robust cache reuse under this homogeneous translation.
Overall, \mot provides the most balanced behavior across closed-set and
extractive QA.

Table~\ref{tab:gpt2_xl_gpt2-large_heterogeneous} further evaluates
intra-family heterogeneous translation,
\texttt{gpt2-xl}$\to$\texttt{gpt2-large}, where the source and target models
belong to the same GPT-2 family but differ in depth and representation size.
For this more challenging translation, \mot again achieves the highest cache
similarity of 0.89,
the best average accuracy of \(42.0\%\), and the best average F1 of \(0.22\) among the
translation methods. Compared with the baselines, \mot improves both closed-set
QA and extractive QA: it obtains the best BoolQ, PubMedQA, MMLU-Redux, SQuAD,
and average scores, while matching the best NewsQA score. These results indicate
that \mot remains effective even when a larger source cache must be mapped into a
smaller and structurally different target cache space. We omit \kvcomm for this
heterogeneous translation because it assumes homogeneous KV cache compatibility
and does not support heterogeneous source and target architectures.

Across both GPT-2 translations, \mot consistently provides stronger cache
alignment and more robust downstream performance than the competing translation
methods. The gains are especially clear on extractive QA, where accurate
preservation of context information is critical. \mot introduces additional peak
memory due to the use of multiple translators, but this is a marginal overhead
relative to the benefit of improved cache alignment and stable downstream
performance. Together with the \qwen results above, these experiments provide
additional evidence that \mot remains effective in both homogeneous and
intra-family heterogeneous translations, complementing the cross-family
heterogeneous translations evaluated in Section~\ref{sec:results}.

\subsection{Root Cause Analysis}
\label{app:root_cause_analysis}

\paragraph{Shifts and Correction.}

\begin{figure*}[t!]
    \centering

    \begin{minipage}[t]{0.48\textwidth}
        \centering
        \includegraphics[width=\linewidth]{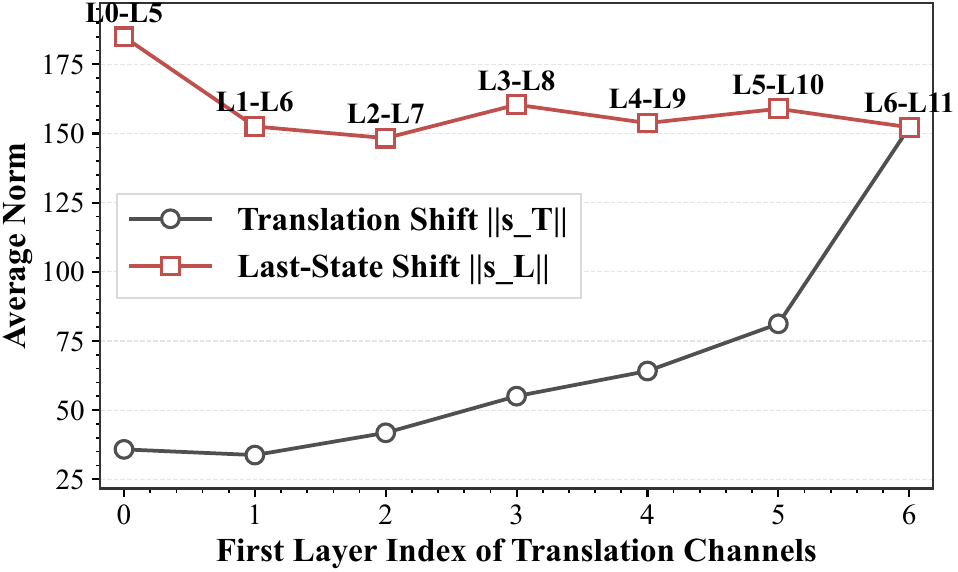}
        {\small (a) \single \& \textsf{Prompt LM Loss}.\par}
    \end{minipage}
    \hfill
    \begin{minipage}[t]{0.48\textwidth}
        \centering
        \includegraphics[width=\linewidth]{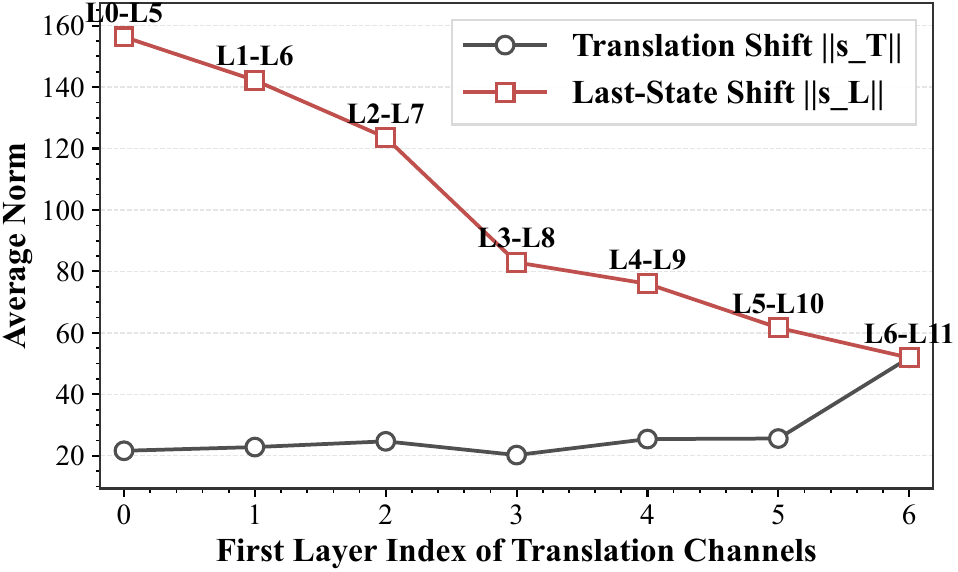}
        {\small (b) \mot \& \textsf{Context Correction Loss}.\par}
    \end{minipage}

    \caption{Shift improvement.}
    \label{fig:final_shift_improvement_appendix}
\end{figure*}

\begin{figure*}[t!]
    \centering

    \begin{minipage}[t]{0.48\textwidth}
        \centering
        \includegraphics[width=\linewidth]{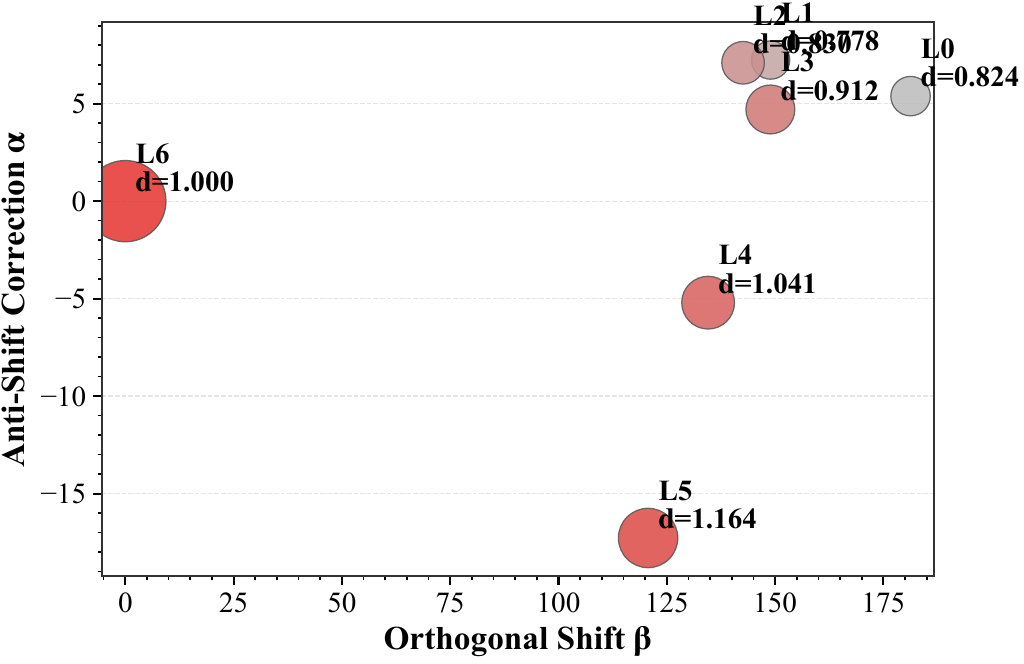}
        {\small (a) \single \& \textsf{Prompt LM Loss}.\par}
    \end{minipage}
    \hfill
    \begin{minipage}[t]{0.48\textwidth}
        \centering
        \includegraphics[width=\linewidth]{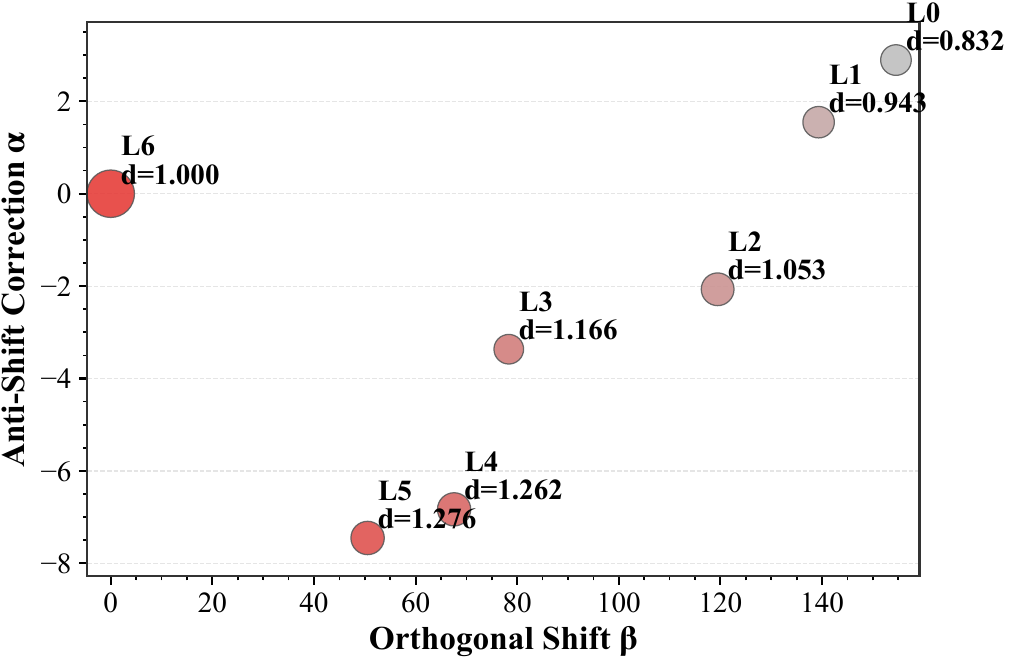}
        {\small (b) \mot \& \textsf{Context Correction Loss}.\par}
    \end{minipage}

    \caption{Correction improvement from layer L0 to L6.}
    \label{fig:correction_improvement_appendix}
\end{figure*}

Figure~\ref{fig:final_shift_improvement_appendix} and Figure~\ref{fig:correction_improvement_appendix} reorganize the shift and correction results shown separately in Figure~\ref{fig:last_state_shift}, Figure~\ref{fig:improved_shifts}, Figure~\ref{fig:correction_decomposition}, and Figure~\ref{fig:improved_correction}. These appendix figures place the before-and-after results side by side for easier comparison between the setting before applying the proposed method, \single \& \textsf{Prompt LM Loss}, and the setting after applying it, \mot \& \textsf{Context Correction Loss}.

\begin{figure*}[t!]
    \centering

    \begin{minipage}[t]{0.48\textwidth}
        \centering
        \includegraphics[width=\linewidth]{figures/rca_before/val_loss.pdf}
        {\small (a) \single \& \textsf{Prompt LM Loss}.\par}
    \end{minipage}
    \hfill
    \begin{minipage}[t]{0.48\textwidth}
        \centering
        \includegraphics[width=\linewidth]{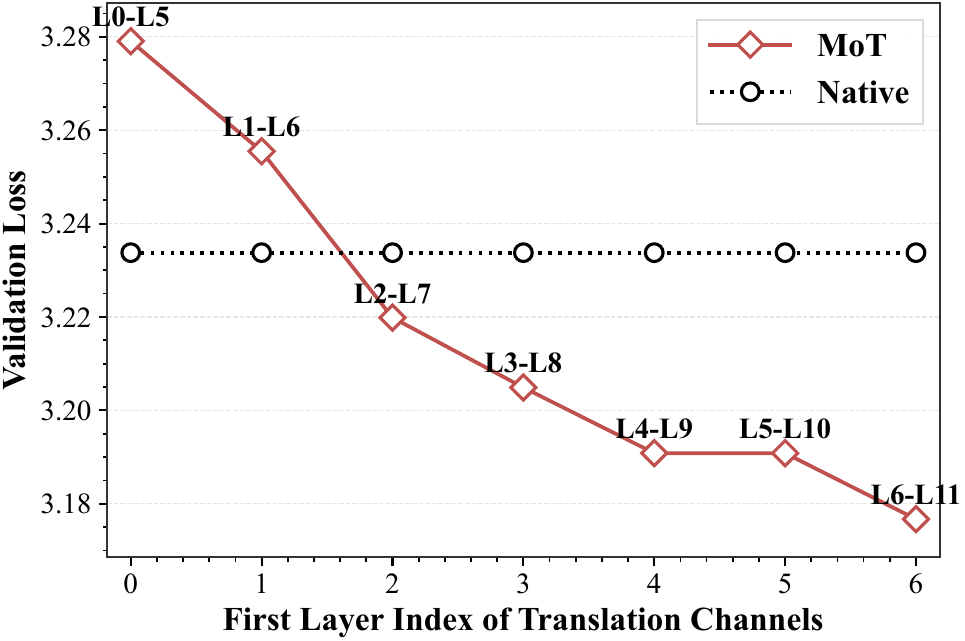}
        {\small (b) \mot \& \textsf{Context Correction Loss}.\par}
    \end{minipage}

    \caption{Validation loss improvement.}
    \label{fig:val_loss_improvement}
\end{figure*}

\begin{figure*}[t!]
    \centering

    \begin{minipage}[t]{0.48\textwidth}
        \centering
        \includegraphics[width=\linewidth]{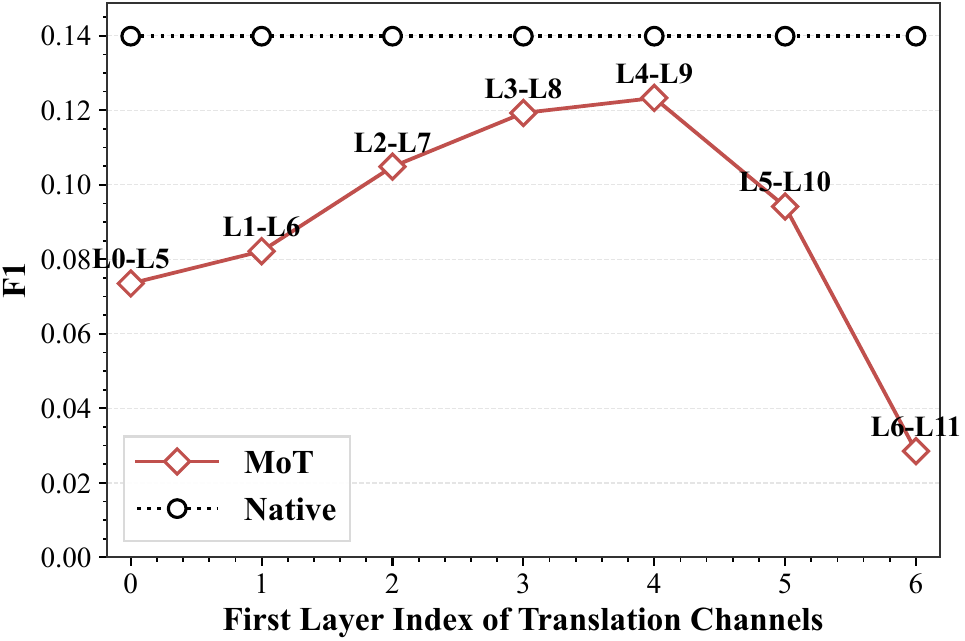}
        {\small (a) \single \& \textsf{Prompt LM Loss}.\par}
    \end{minipage}
    \hfill
    \begin{minipage}[t]{0.48\textwidth}
        \centering
        \includegraphics[width=\linewidth]{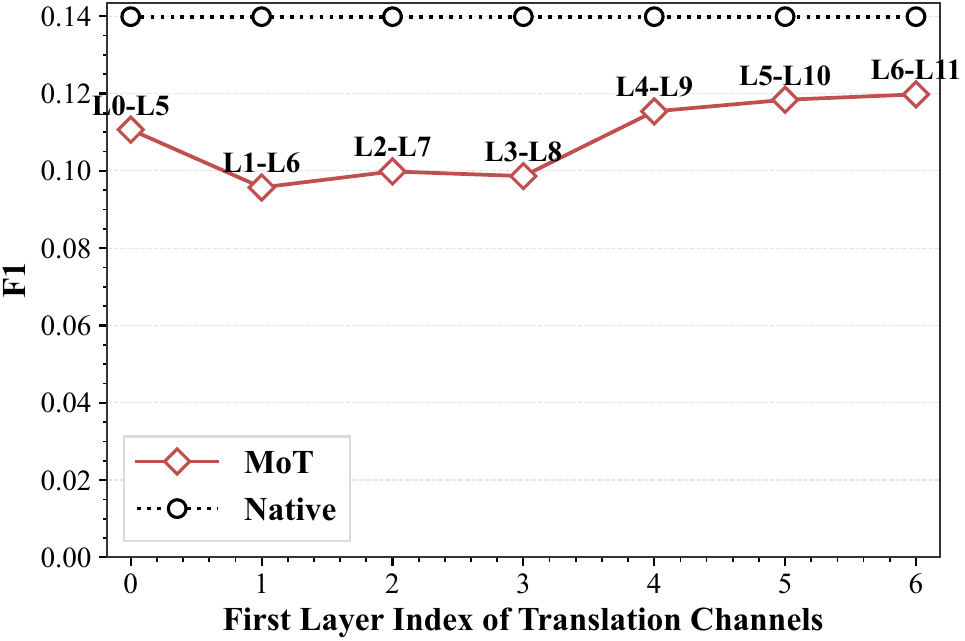}
        {\small (b) \mot \& \textsf{Context Correction Loss}.\par}
    \end{minipage}

    \caption{F1 score improvement.}
    \label{fig:f1_improvement}
\end{figure*}

Similarly, Figure~\ref{fig:val_loss_improvement} and Figure~\ref{fig:f1_improvement} show the improvements in validation loss and F1 score. The U-shaped loss curve changes substantially into a sharply decreasing trend, and the F1 curve also becomes flat across the entire layer range without large variation, rather than exhibiting an inverted U-shape. This can likewise be interpreted as a consequence of reducing both the initial translation shift $\|\bs_T\|$ and the last-state shift $\|\bs_L\|$.

\paragraph{Layer-wise Translated Error Propagation and Correction.}

\begin{figure*}[t!]
    \centering

    \begin{minipage}[t]{0.48\textwidth}
        \centering
        \includegraphics[width=\linewidth]{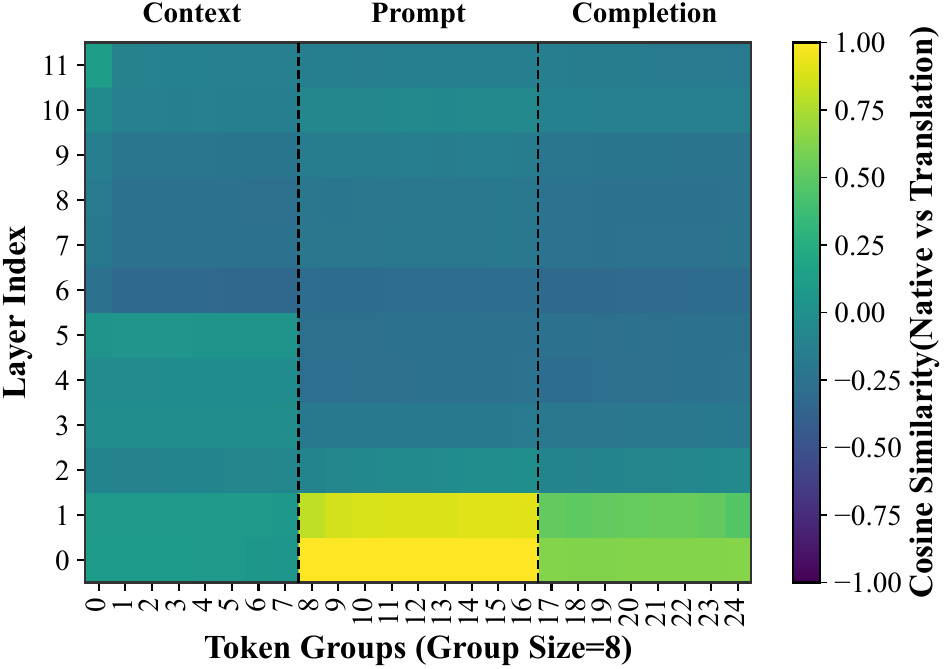}
        {\small (a) Layer 0 (\single)\par}
    \end{minipage}
    \hfill
    \begin{minipage}[t]{0.48\textwidth}
        \centering
        \includegraphics[width=\linewidth]{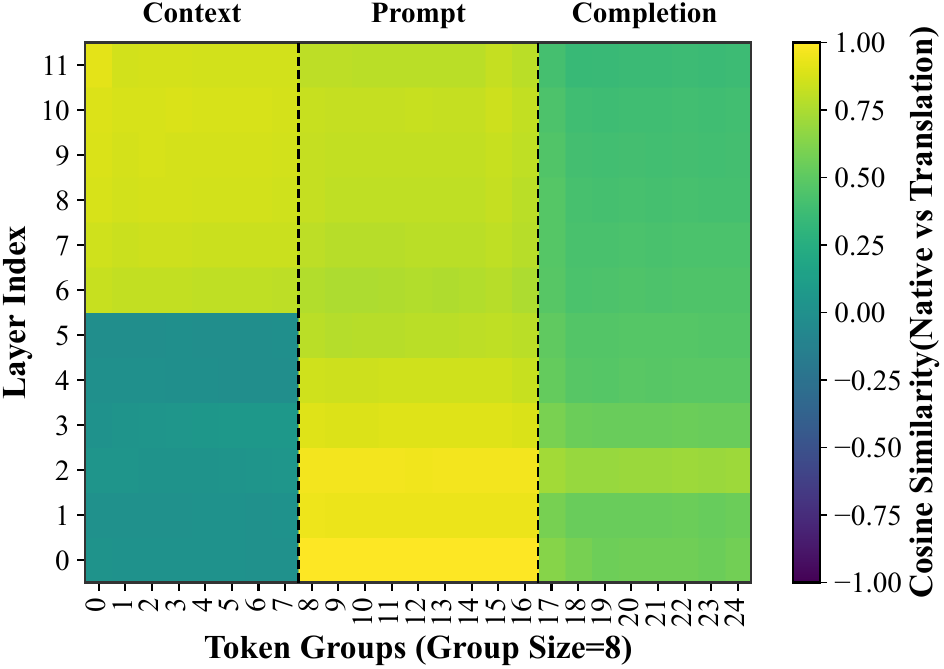}
        {\small (b) Layer 0 (\mot)\par}
    \end{minipage}

    \begin{minipage}[t]{0.48\textwidth}
        \centering
        \includegraphics[width=\linewidth]{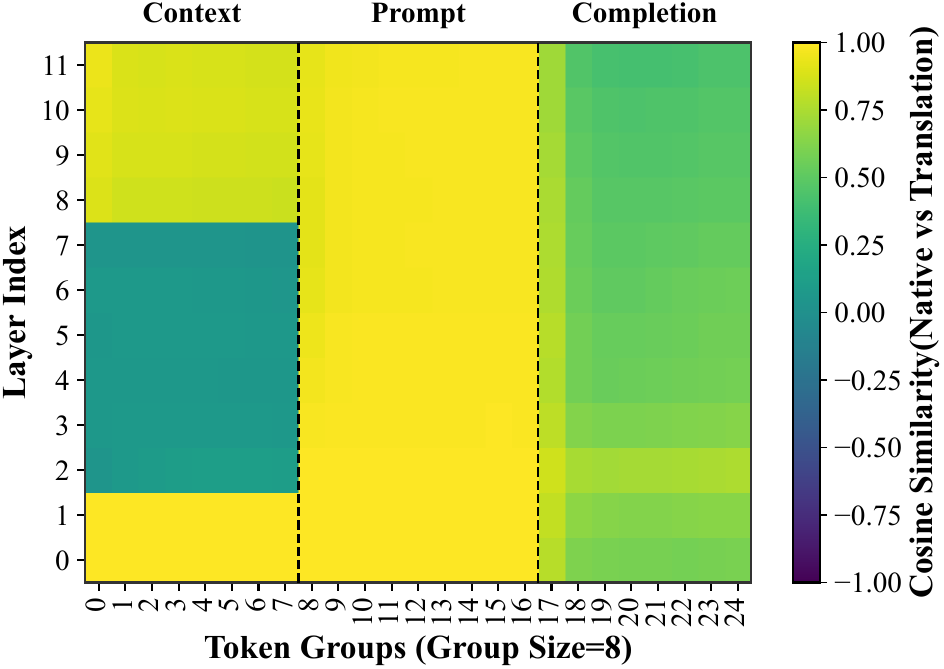}
        {\small (c) Layer 2 (\single)\par}
    \end{minipage}
    \hfill
    \begin{minipage}[t]{0.48\textwidth}
        \centering
        \includegraphics[width=\linewidth]{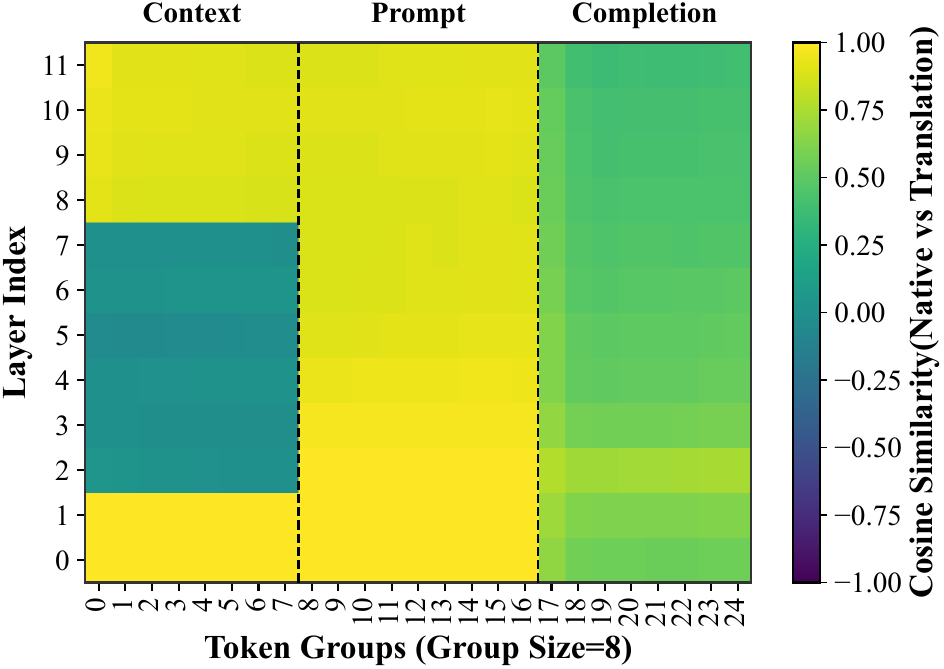}
        {\small (d) Layer 2 (\mot)\par}
    \end{minipage}

    \begin{minipage}[t]{0.48\textwidth}
        \centering
        \includegraphics[width=\linewidth]{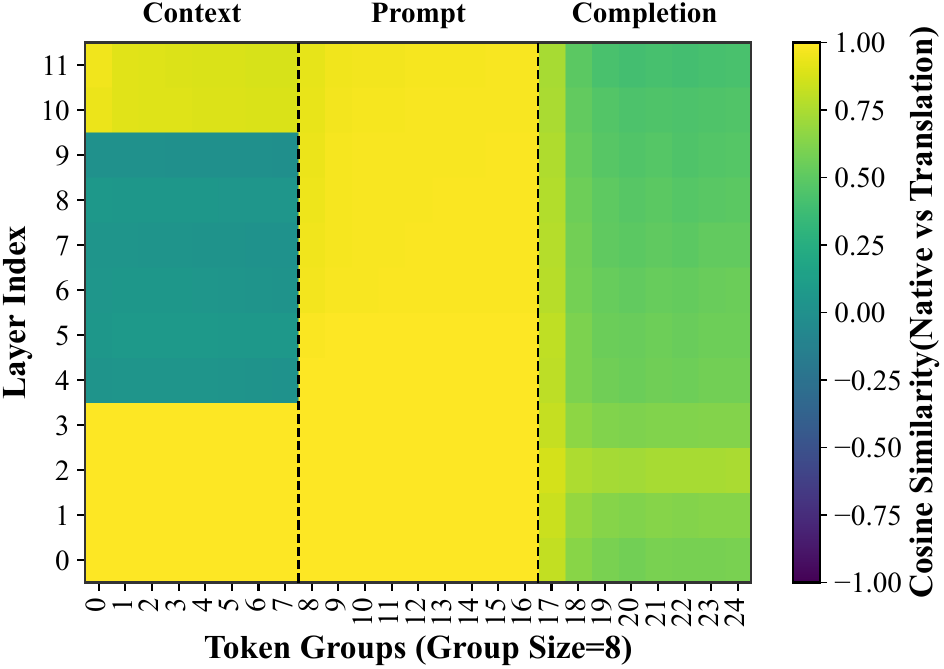}
        {\small (e) Layer 4 (\single)\par}
    \end{minipage}
    \hfill
    \begin{minipage}[t]{0.48\textwidth}
        \centering
        \includegraphics[width=\linewidth]{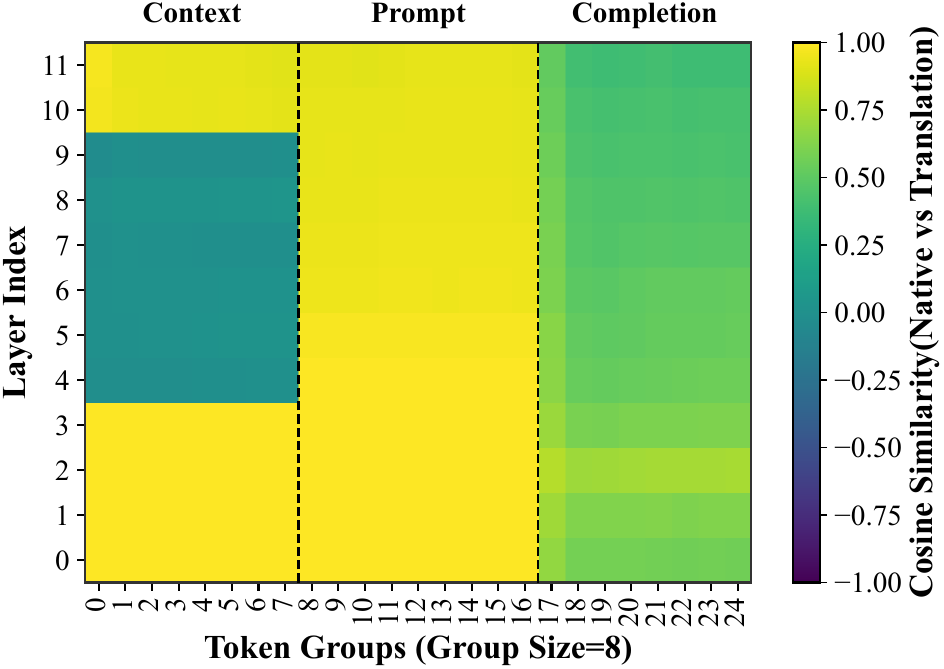}
        {\small (f) Layer 4 (\mot)\par}
    \end{minipage}

    \begin{minipage}[t]{0.48\textwidth}
        \centering
        \includegraphics[width=\linewidth]{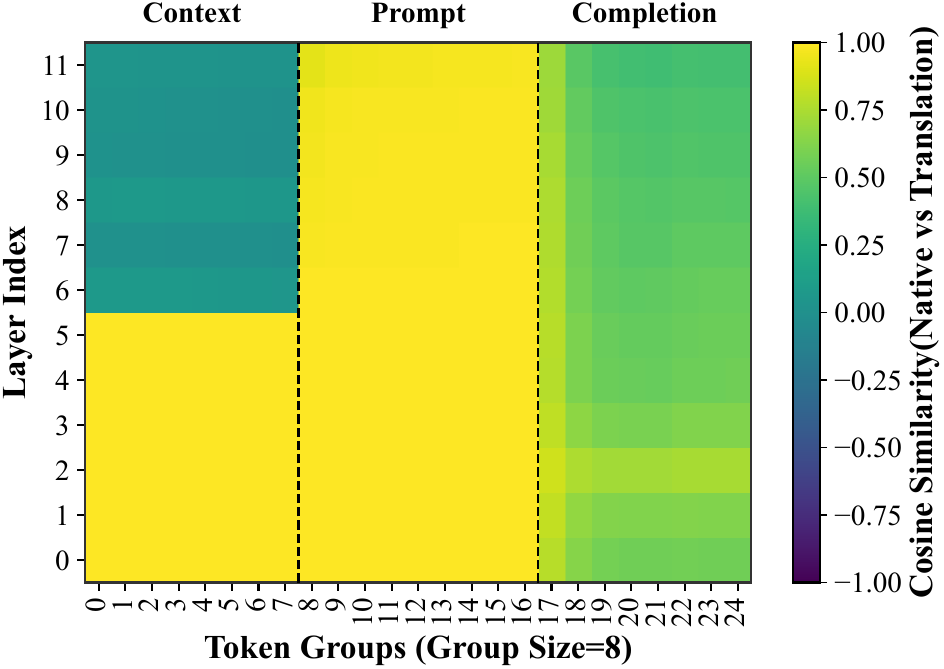}
        {\small (g) Layer 6 (\single)\par}
    \end{minipage}
    \hfill
    \begin{minipage}[t]{0.48\textwidth}
        \centering
        \includegraphics[width=\linewidth]{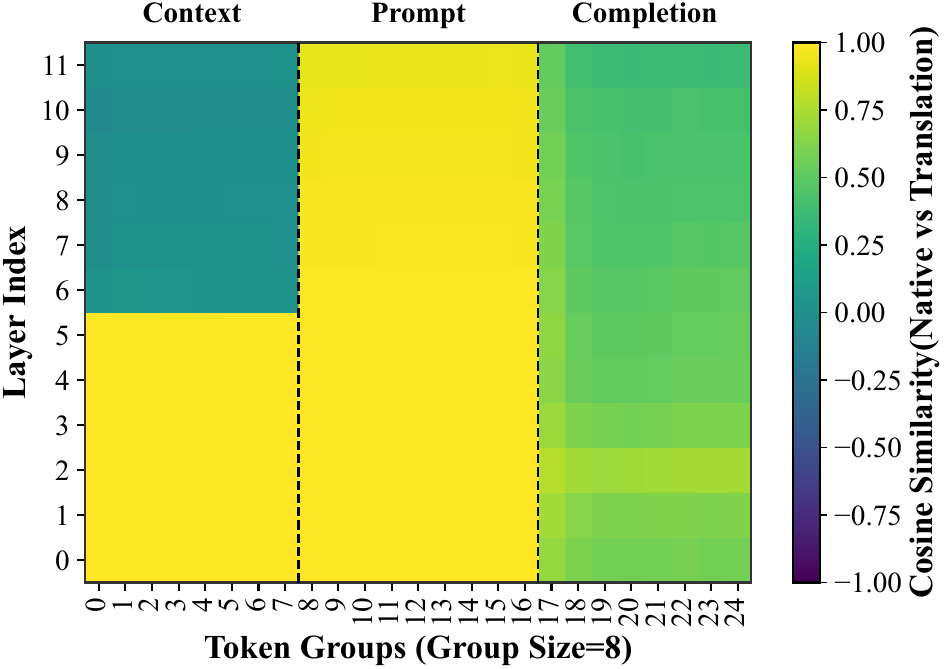}
        {\small (h) Layer 6 (\mot)\par}
    \end{minipage}

    \caption{Error propagation improvement.}
    \label{fig:error_prop_improve}
\end{figure*}

\begin{figure*}[t!]
    \centering

    \begin{minipage}[t]{0.48\textwidth}
        \centering
        \includegraphics[width=\linewidth]{figures/rca_before/kv_last_layer_violin_plot.pdf}
        {\small (a) KV Similarity (\single)\par}
    \end{minipage}
    \hfill
    \begin{minipage}[t]{0.48\textwidth}
        \centering
        \includegraphics[width=\linewidth]{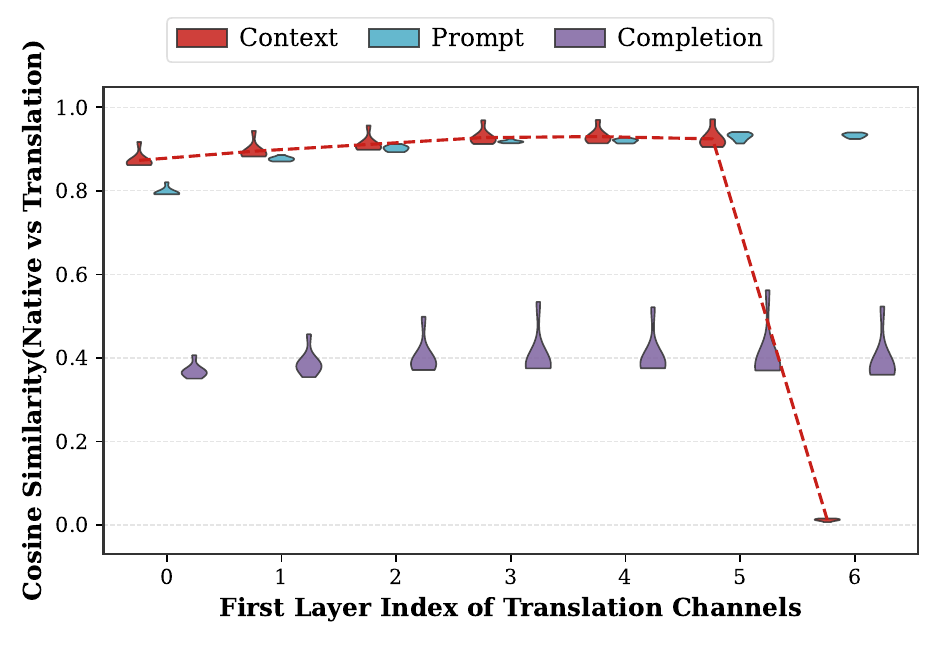}
        {\small (b) KV Similarity (\mot)\par}
    \end{minipage}

    \caption{Overall error propagation improvement.}
    \label{fig:overall_error_prop_improve}
\end{figure*}

In Figure~\ref{fig:error_prop_3}, we examined the KV cache similarity heatmap between Native and MoT when translation starts at layer 3, corresponding to the layer range [3,8]. This analysis confirmed the phenomena of error propagation and correction at the KV level. Figure~\ref{fig:error_prop_improve} provides an overall heatmap overview comparing the trends in KV similarity before and after applying \mot, while sliding the translation start layer over 0, 2, 4, and 6.

The most noticeable change appears between Figure~\ref{fig:error_prop_3}(a) and Figure~\ref{fig:error_prop_3}(b): under \single, severe error propagates across all regions, whereas \mot substantially alleviates this propagation. In the context-token region above the translation window in Figure~\ref{fig:error_prop_3}(c--f), the corresponding regions in (d) and (f) become slightly brighter than those in (c) and (e), which indirectly indicates improved correction. In particular, when focusing on token group 0, both \single and \mot show a tendency for correction to become stronger near the top layers.

Figure~\ref{fig:overall_error_prop_improve}(a) shows the inverted U-shaped similarity trend that appears especially in the context region, whereas Figure~\ref{fig:overall_error_prop_improve}(b) shows that the error propagation effect is almost eliminated under \mot. This can be attributed to the large reduction in $\|\bs_T\|$ observed earlier. However, at the final translation-layer index 6, correction-deficit error still remains, which is consistent with the analysis in Proposition~\ref{prop:correction_deficit_coefficient_bound} that the correction deficit increases to 1 at the final layer. Note that this does not mean that correction itself has failed.

\subsection{Ablation}
\label{sec:ablation}

\paragraph{Top-\(K\) and Number of Translators.}

\begin{table}[t]
\footnotesize
\centering
\begin{tabular}{cccc}
\toprule
Translation & (Top-\(K\), \(N_{\mathrm{Tr}}\)) & F1 & GPU Peak Memory (GiB) \\
\midrule

\multirow{4}{*}{\texttt{gpt2}$\rightarrow$\texttt{gpt2-medium}}
& (1,1) & 0.34 & 3.678 \\
& (1,2) & 0.34 & \textbf{5.131} \\
& (1,4) & 0.35 & 7.960 \\
& (2,4) & \textbf{0.35} & 7.959 \\
\midrule

\multirow{4}{*}{\texttt{gpt2-medium}$\rightarrow$\texttt{gpt2}}
& (1,1) & 0.12 & 3.442 \\
& (1,2) & 0.13 & \textbf{4.835} \\
& (1,4) & 0.12 & 7.526 \\
& (2,4) & \textbf{0.14} & 7.525 \\
\midrule

\multirow{4}{*}{\texttt{Qwen2.5-0.5B}$\rightarrow$\texttt{Qwen2.5-1.5B}}
& (1,1) & 0.53 & 10.280 \\
& (1,2) & \textbf{0.54} & \textbf{12.828} \\
& (1,4) & 0.53 & 17.898 \\
& (2,4) & 0.52 & 17.899 \\
\midrule

\multirow{4}{*}{\texttt{Qwen2.5-1.5B}$\rightarrow$\texttt{Qwen2.5-0.5B}}
& (1,1) & 0.39 & 10.106 \\
& (1,2) & \textbf{0.40} & \textbf{12.530} \\
& (1,4) & 0.36 & 17.321 \\
& (2,4) & 0.37 & 17.322 \\
\midrule

\multirow{4}{*}{\texttt{opt-125m}$\rightarrow$\texttt{opt-350m}}
& (1,1) & 0.05 & 3.524 \\
& (1,2) & 0.05 & \textbf{4.975} \\
& (1,4) & 0.05 & 7.806 \\
& (2,4) & \textbf{0.05} & 7.806 \\
\midrule

\multirow{4}{*}{\texttt{opt-350m}$\rightarrow$\texttt{opt-125m}}
& (1,1) & 0.04 & 3.286 \\
& (1,2) & \textbf{0.05} & \textbf{4.681} \\
& (1,4) & 0.04 & 7.371 \\
& (2,4) & 0.04 & 7.371 \\

\bottomrule
\end{tabular}
\caption{F1 and peak memory for Top-\(K\) routing over \(N_{\mathrm{Tr}}\) translators.}
\label{tab:gen_f1_gpu_memory}
\end{table}

Table~\ref{tab:gen_f1_gpu_memory} compares F1 and GPU Peak Memory under different routing and translator-count settings. The (1,1) row corresponds to \single and is included only as a reference point, while the remaining rows are \mot configurations used for hyperparameter selection. Across the \mot configurations, the best F1 is obtained either with (1,2) or
(2,4) depending on the translation, indicating that increasing Top-\(K\) and
the number of translators does not consistently yield a clear performance
advantage. In contrast, GPU Peak Memory increases mainly with the number of translators, and (1,2) consistently uses the least memory among the \mot configurations. Therefore, we conservatively choose \((\mathrm{Top}\text{-}K,N_{\mathrm{Tr}})=(1,2)\) as the default configuration.

\paragraph{Channel Mapping.}

\begin{figure*}[t]
    \centering

    \begin{minipage}{0.48\textwidth}
        \centering
        \includegraphics[width=\textwidth,height=4.6cm]{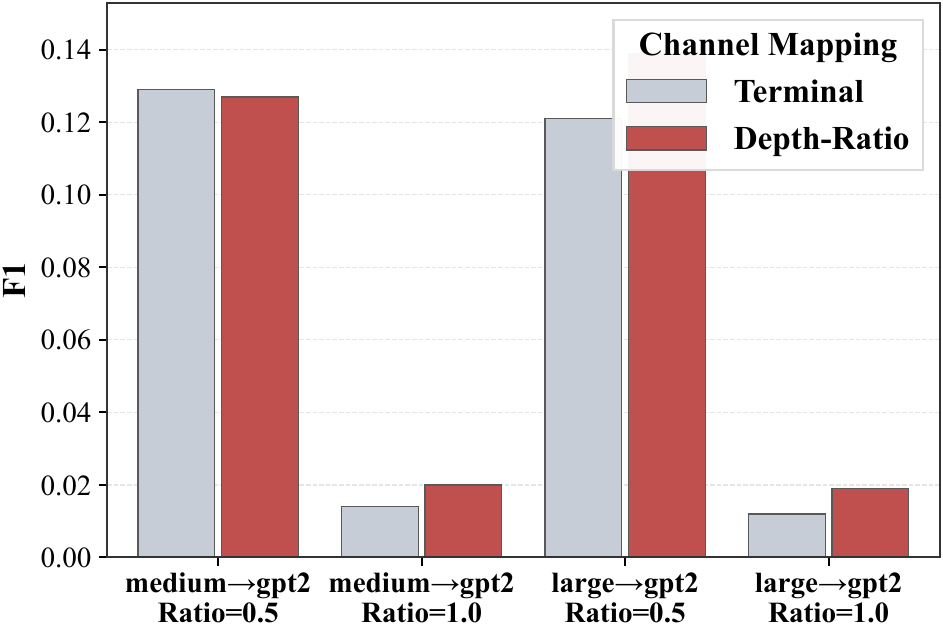}
        \captionof{figure}{Channel mapping comparison\newline: Terminal vs. Depth-Ratio.}
        \label{fig:channel_mapping}
    \end{minipage}
    \hfill
    \begin{minipage}{0.48\textwidth}
        \centering
        \includegraphics[width=\textwidth,height=4.6cm]{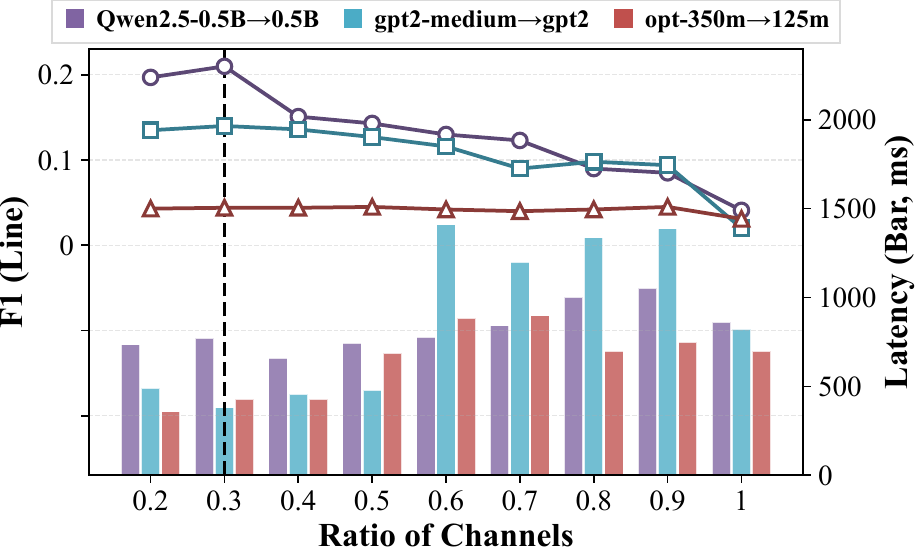}
        \captionof{figure}{Effect of the translation-channel ratio ($ChannelRatio$).}
        \label{fig:channel_ratio}
    \end{minipage}

\end{figure*}

To examine the effect of the channel-mapping strategy, we compare Terminal Mapping and Depth-Ratio Mapping. Terminal Mapping concentrates the channels near the final layers of the target model, whereas Depth-Ratio Mapping connects source and target layers that have similar relative depths.

As shown in Figure~\ref{fig:channel_mapping}, Depth-Ratio Mapping achieves higher Gen F1 than Terminal Mapping in most settings. In particular, when the Channel Ratio is \(1.0\), Terminal Mapping suffers a large drop in F1, whereas Depth-Ratio Mapping maintains relatively high F1. This shows that connecting layers with similar depth ratios provides more stable cache translation than concentrating channels in the terminal region.

Therefore, in the remaining experiments, we use Depth-Ratio Mapping as the default channel-mapping strategy.

\paragraph{Ratio of Translation Channels.}

To determine the optimal size of the channel set, we vary the ratio of channels used for actual translation among all Depth-Ratio channel candidates and measure both validation performance and latency. As shown in Figure~\ref{fig:channel_ratio}, F1 generally decreases and latency increases as the channel ratio becomes larger. This indicates that using more channels does not necessarily lead to better translation; rather, it may increase the likelihood of including propagation error or correction-deficit error, while also increasing the translation cost.

Although a low channel ratio is advantageous in terms of both accuracy and latency, using too small a ratio can limit the available translation channels and make performance unstable depending on the model and task. Therefore, we use \(ChannelRatio=0.3\) as the default setting, since it maintains high F1 stably across multiple models while keeping the latency increase small.

\paragraph{Number of Full Attention Layers at Bottom.}

\begin{figure*}[t]
    \centering

    \begin{minipage}{0.48\textwidth}
        \centering
        \includegraphics[width=\textwidth,height=4.6cm]{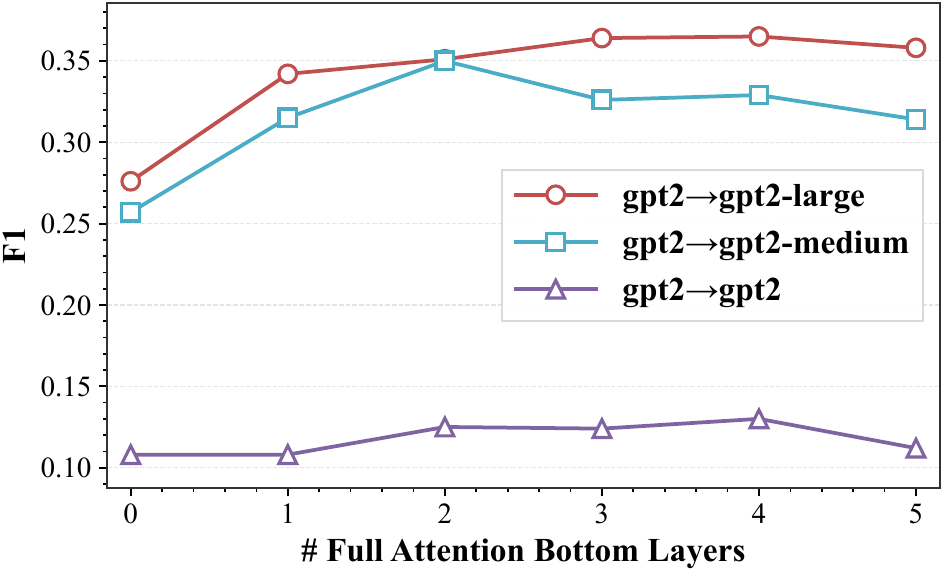}
        \captionof{figure}{Effect of the number of full attention layers at the bottom.}
        \label{fig:num_full_attn_layers}
    \end{minipage}
    \hfill
    \begin{minipage}{0.48\textwidth}
        \centering
        \includegraphics[width=\textwidth,height=4.6cm]{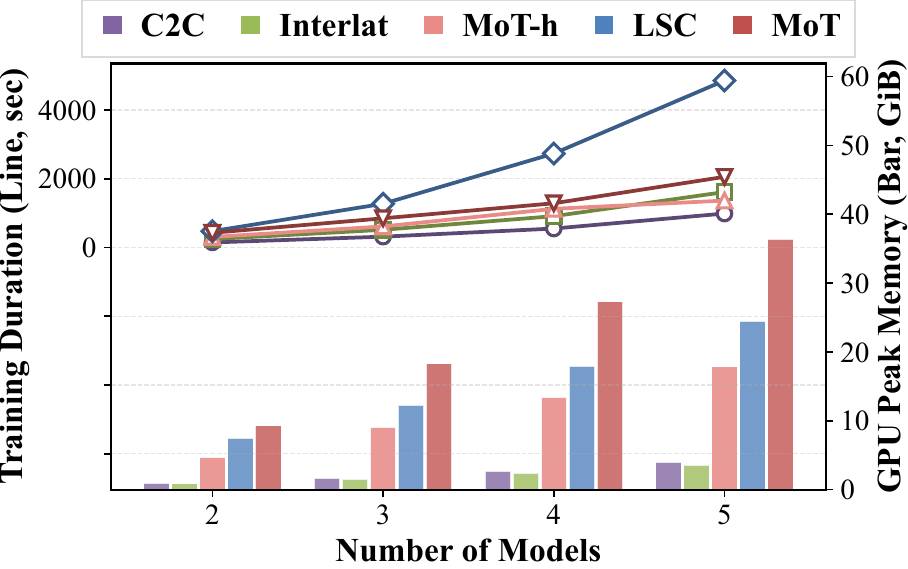}
        \caption{Training scalability of \gpt with an increasing number of models.}
        \label{fig:scalability_num_models}
    \end{minipage}

\end{figure*}

Top-\(S\) Source-Guided Sparse Attention reduces the replay cost, but small perturbations in the lowest target layers can substantially disrupt the subsequent representation mapping. Therefore, it is necessary to keep full attention in some bottom layers. To verify this, we vary the number of full attention layers at the bottom and measure F1 in three settings with target models of 12, 24, and 36 layers: (\texttt{gpt2}$\rightarrow$\texttt{gpt2}, \texttt{gpt2}$\rightarrow$\texttt{gpt2-medium}, and \texttt{gpt2}$\rightarrow$\texttt{gpt2-large}). As shown in Figure~\ref{fig:num_full_attn_layers}, even as the target model becomes deeper, most of the performance gain occurs when full attention is applied to only a small number of bottom layers, and converting additional bottom layers to full attention provides limited benefit. This is consistent with prior analysis showing that representation-mapping collapse caused by perturbations is most pronounced in the lowest layers~\citep{zheng2025spurious}. Therefore, in this work, we set the number of bottom full-attention layers to \(L_{\mathrm{Full}}=2\) to balance efficiency and accuracy. This choice minimizes the additional cost of full-attention replay while mitigating the loss of lower-layer mapping caused by sparse replay.

\subsection{Scalability}
\label{app:scalability}

This appendix provides additional scalability analyses that complement the main benchmark results.
We study two axes of scalability.
First, we evaluate training scalability as the number of models in the model
pool increases. Second, we evaluate capacity scalability by fixing the target
model and increasing the capacity of the source model within the same model
family. The purpose of these experiments is not to show that larger source
models always improve downstream accuracy, but to examine whether translation
quality, peak memory, latency, and throughput remain stable as the model pool
grows or model heterogeneity increases.

\paragraph{Model Capacity.}

\begin{figure*}[t!]
    \centering

    \begin{minipage}[t]{0.48\textwidth}
        \centering
        \includegraphics[width=\linewidth]{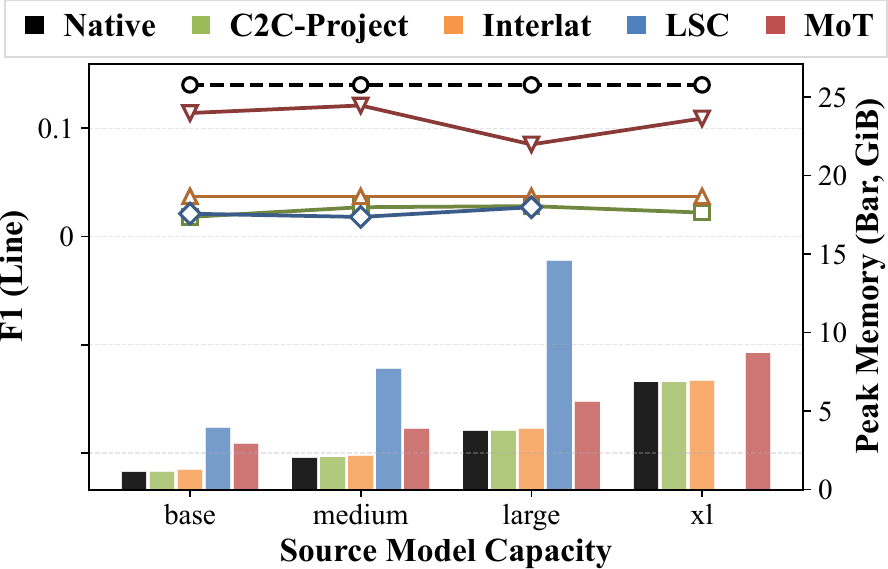}
        {\small (a) F1 and peak inference memory.\par}
    \end{minipage}
    \hfill
    \begin{minipage}[t]{0.48\textwidth}
        \centering
        \includegraphics[width=\linewidth]{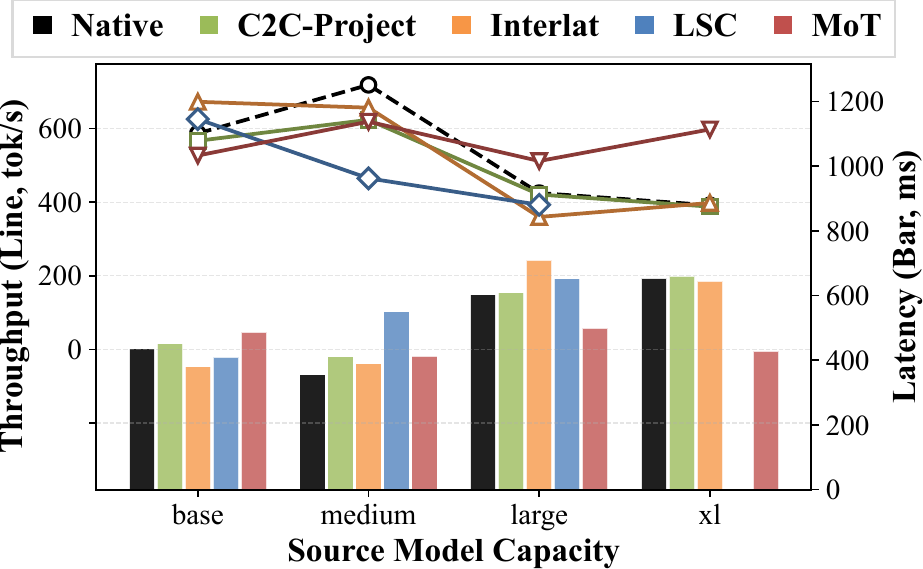}
        {\small (b) Inference throughput and latency.\par}
    \end{minipage}

    \caption{Scalability of the \gpt model family with increasing source-model capacity.}
    \label{fig:scalability_capacity_gpt2}
\end{figure*}

\begin{figure*}[t!]
    \centering

    \begin{minipage}[t]{0.48\textwidth}
        \centering
        \includegraphics[width=\linewidth]{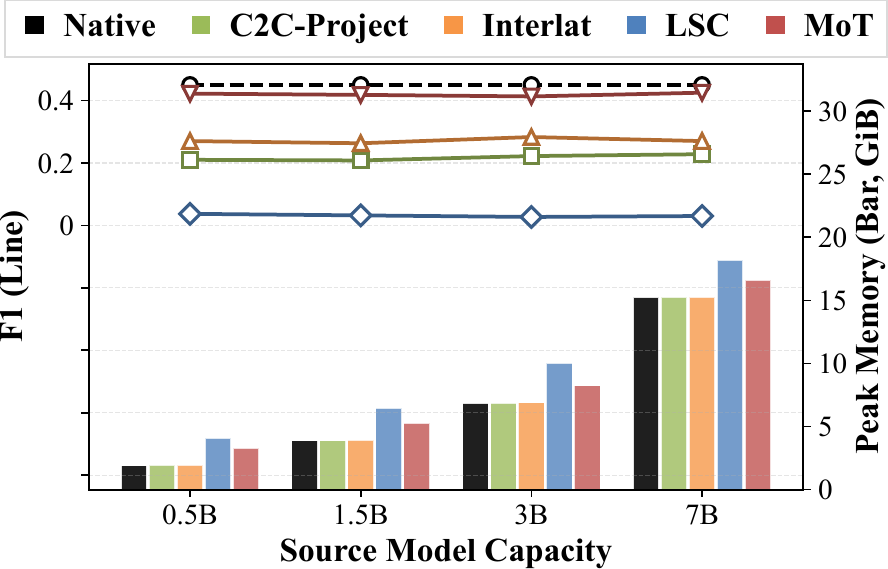}
        {\small (a) Peak memory and F1.\par}
    \end{minipage}
    \hfill
    \begin{minipage}[t]{0.48\textwidth}
        \centering
        \includegraphics[width=\linewidth]{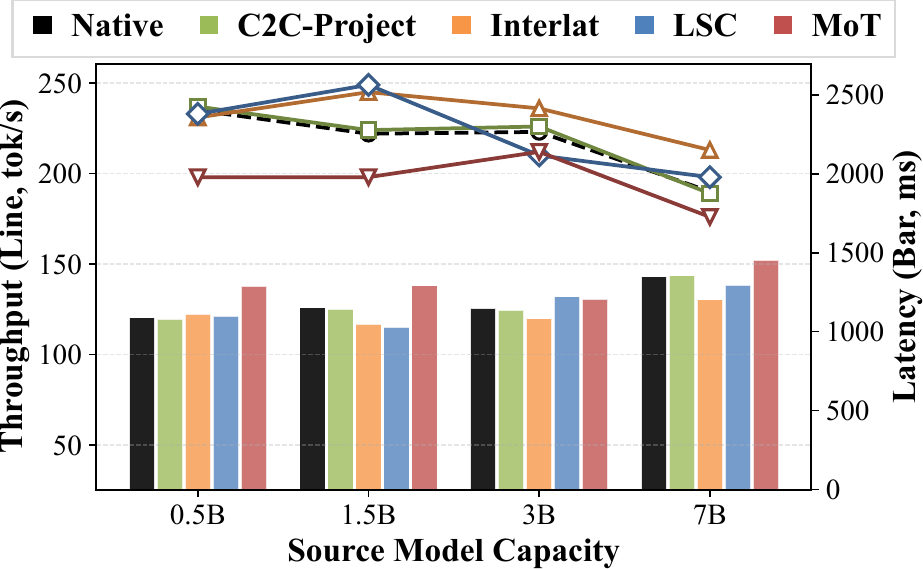}
        {\small (b) Latency and throughput.\par}
    \end{minipage}

    \caption{Scalability of the \qwen model family with increasing source-model capacity.}
    \label{fig:scalability_capacity_qwen2.5}
\end{figure*}

\begin{figure*}[t!]
    \centering

    \begin{minipage}[t]{0.48\textwidth}
        \centering
        \includegraphics[width=\linewidth]{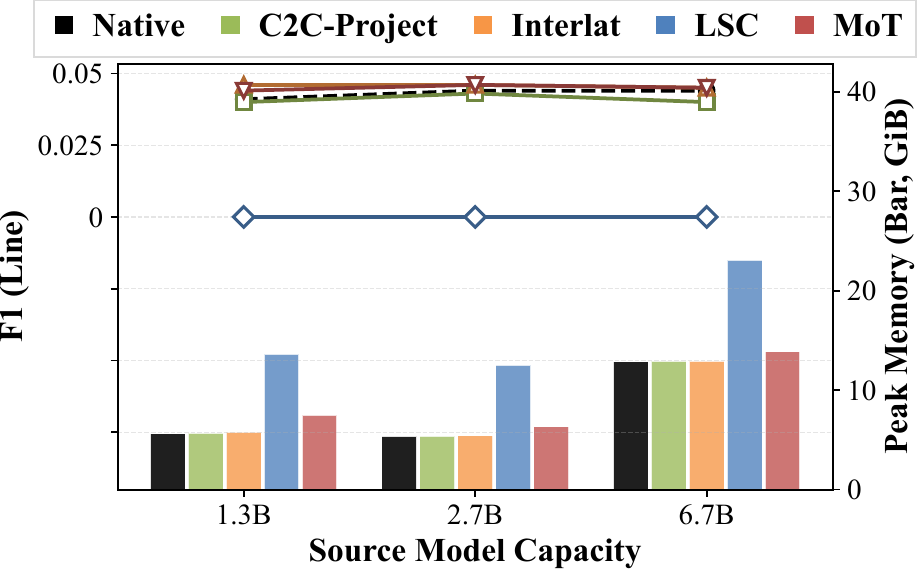}
        {\small (a) Peak memory and F1.\par}
    \end{minipage}
    \hfill
    \begin{minipage}[t]{0.48\textwidth}
        \centering
        \includegraphics[width=\linewidth]{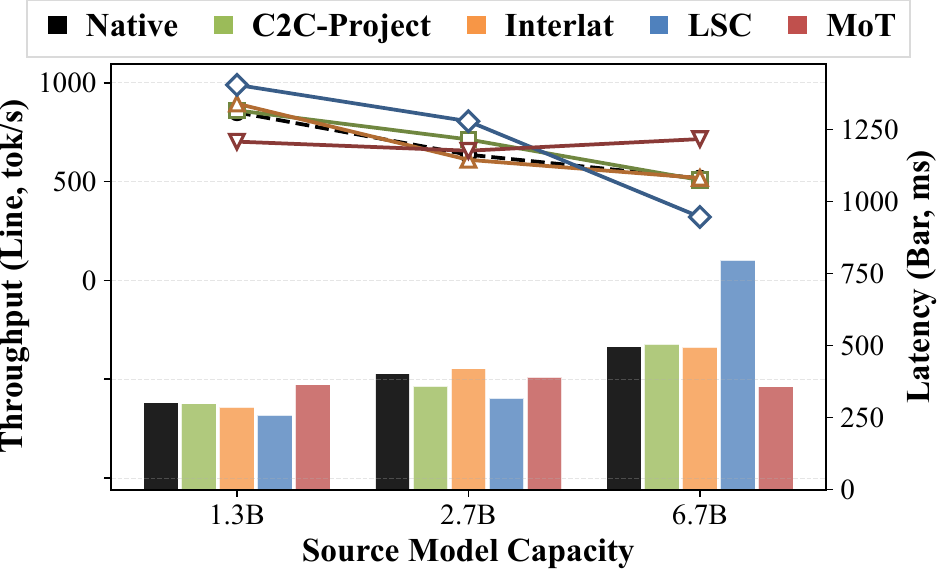}
        {\small (b) Latency and throughput.\par}
    \end{minipage}

    \caption{Scalability of the \opt model family with increasing source-model capacity.}
    \label{fig:scalability_capacity_opt}
\end{figure*}

We next evaluate capacity scalability by fixing the target model and increasing only the source-model capacity.
This is a challenging setting because the target cache space remains fixed while the source representation becomes larger and potentially more mismatched.
Empirically, we found this source-scaling regime to be more difficult than the
reverse regime where the source model is fixed and the target model capacity
increases.
Therefore, this experiment directly tests whether each method can translate increasingly larger source caches into the same target cache space without degrading downstream quality.

For the \gpt family, we fix the target model to \texttt{gpt2} and increase the source model from \texttt{gpt2} to \texttt{gpt2-xl}.
As shown in Figure~\ref{fig:scalability_capacity_gpt2}, \mot keeps F1 closest to \native even as the source model capacity increases.
\ctoc remains close to \native in peak memory and latency, but its F1 drops substantially, indicating that simple projection is not sufficient for preserving the information needed by the target model.
\interlat and \lsc also fail to preserve F1 consistently, and \lsc becomes unstable in the largest \texttt{gpt2-xl} setting.
In contrast, \mot requires additional peak memory due to the translator module, but its F1 degradation is limited and its latency and throughput remain stable in larger-capacity regimes.

Figure~\ref{fig:scalability_capacity_qwen2.5} shows the same analysis for the \qwen model family.
The target model is fixed to \texttt{Qwen2.5-0.5B}, while the source model increases from \texttt{Qwen2.5-0.5B} to \texttt{Qwen2.5-7B}.
Across this range, \mot maintains F1 closest to \native.
The F1 of \mot remains stable as source capacity increases, while \ctoc, \interlat, and especially \lsc show much larger degradation.
In terms of peak memory, \mot uses additional memory compared with \native, but this overhead remains nearly constant across source capacities.
Thus, the relative impact of the translator overhead becomes smaller as the source model grows.
Latency and throughput also remain stable, suggesting that the overhead of \mot does not accumulate with source-model capacity.

Figure~\ref{fig:scalability_capacity_opt} shows the capacity-scaling behavior for the \opt model family.
The target model is fixed to \texttt{opt-125m}, and the source model is scaled up to \texttt{opt-1.3b}, \texttt{opt-2.7b}, and \texttt{opt-6.7b}.
The \texttt{opt-125m} source result is excluded from the figure because it behaves as an outlier.
Homogeneous methods such as \kvcomm are also excluded from this intra-family heterogeneous capacity-scaling comparison.
In this setting, \mot again preserves F1 close to \native across source capacities.
While the absolute F1 values of the \opt family are lower than those of \qwen, \mot remains the most consistent method in terms of quality preservation.
\lsc collapses across all \opt capacity settings, whereas \ctoc and \interlat show less severe degradation but are less consistent than \mot.

Overall, the capacity-scaling experiments show that \mot remains stable when the source model becomes larger while the target model is fixed.
\ctoc can be memory- and latency-efficient, but often fails to preserve downstream F1.
\interlat provides stable system-level metrics but is less reliable in quality preservation.
\lsc suffers from quality collapse in several large-capacity settings, likely due to the difficulty of maintaining a global shared latent space across increasingly heterogeneous model pairs.
By contrast, \mot consistently offers the best F1 preservation, with a mostly constant translator-related memory overhead and stable latency-throughput behavior.
These results support the use of \mot in settings where a fixed target model must reuse or communicate with caches produced by larger source models.

\section{Details of Case Studies}
\label{app:details_case_studies}

\subsection{Multi-Agent Reasoning}
\label{sec:multi_agent_reasoning}

\begin{figure}[H]
    \centering
    \includegraphics[width=\textwidth]{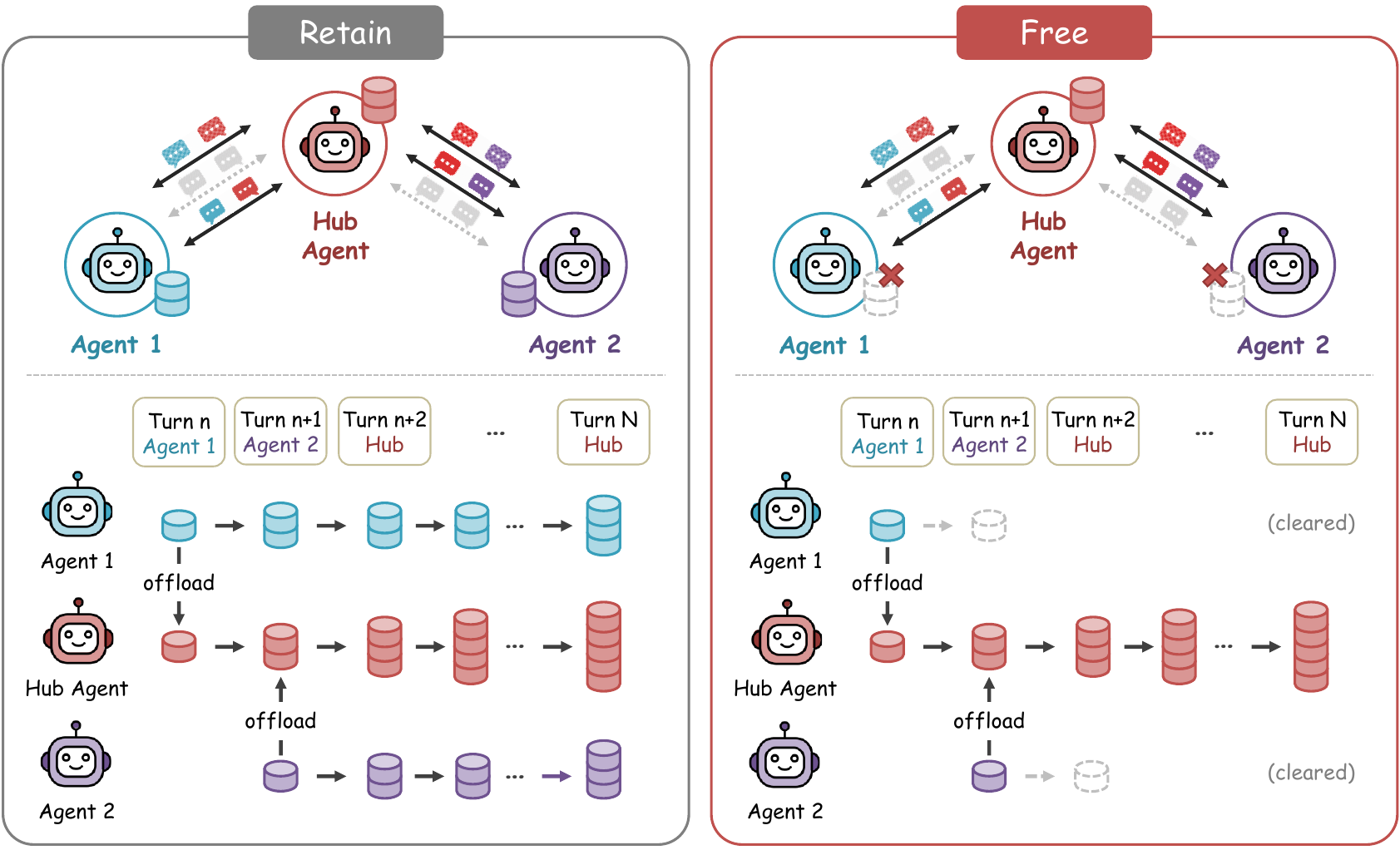}
    \caption{Comparison of \textsf{Retain} and \textsf{Free} cache management strategies in multi-agent reasoning.}
\label{fig:cache_strategies}
\end{figure}

We evaluate StrategyQA under two multi-agent settings: heterogeneous 10-agent
reasoning with alternating \Llama~\textsf{(3B)} and \Qwen~\textsf{(4B)}
agents, and homogeneous agent scalability with
\Llama~\textsf{(3B)}$\to$\Llama~\textsf{(3B)} translation. Both use the same
MALLM discussion protocol~\citep{becker2025mallm} and cache-management
strategies described below.

\paragraph{MALLM Configuration.}
We fix the four major MALLM components across both multi-agent experiments. The
\textsf{Expert} persona generator assigns task-specific and complementary roles
to the participating agents. The \textsf{Simple} response generator lets an
agent propose an initial solution and subsequently agree with or revise the
current solution. The \textsf{Memory} discussion paradigm exposes the
accumulated discussion history to subsequent agents; in our setting, this
history is the state whose KV cache is transferred between agents.
\textsf{Supermajority Consensus} is used as the decision protocol, requiring
more than 66\% agreement on the Yes/No answer.

The multi-agent pipeline and the Expert and Simple prompt templates follow the
official MALLM implementation.%
\footnote{\url{https://github.com/Multi-Agent-LLMs/mallm}}
We integrate cache translation and \textsf{Retain}/\textsf{Free} cache
management while otherwise preserving this configuration. We disable MALLM's
optional Chain-of-Thought prompting to isolate cache translation and
memory-scaling behavior.

\paragraph{Cache Strategies.}

In addition to the MALLM discussion configuration, we introduce two KV-cache
management strategies, \textsf{Retain} and \textsf{Free}, to isolate the effect
of cache lifecycle management on agent scalability.
Figure~\ref{fig:cache_strategies} conceptually illustrates the two strategies.

In \textsf{Retain}, each agent keeps its resident KV cache after completing its
turn. When the accumulated discussion state is transferred to another agent,
the already shared prefix can be retained and only the target-missing cache
suffix needs to be transferred. However, the resident KV caches of multiple
agents remain simultaneously allocated, causing memory usage to grow with the
number of agents.

In contrast, \textsf{Free} keeps the \textsf{Hub Agent} cache resident but
clears a non-hub agent's local KV cache after its state has been offloaded.
Subsequent agents reconstruct the required discussion state from the translated
cache routed through the Hub Agent. By releasing non-hub resident caches after
handoff, \textsf{Free} reduces the number of simultaneously active KV caches
and targets memory scalability as the number of agents increases.

\paragraph{Heterogeneous 10-Agent Reasoning.}

We evaluate 10 agents alternating between
\Llama~\textsf{(3B)} and \Qwen~\textsf{(4B)} on 20 StrategyQA examples, with
up to 20 agent turns and 256 generated tokens per turn.
\ctoc, \interlat, and \lsc use \textsf{Retain}, whereas \mot uses
\textsf{Free}. Table~\ref{tab:hetero_10_agent_cost} reports the one-time
training cost and inference-time memory breakdown; its accuracy and normalized
KV results are also visualized in Figure~\ref{fig:performance_landscape}.

\begin{table}[t]
\centering
\small
\setlength{\tabcolsep}{3.5pt}
\caption{
Training cost and inference efficiency for heterogeneous 10-agent reasoning.
KV/1K normalizes KV usage by successfully generated discussion tokens.
}
\label{tab:hetero_10_agent_cost}
\resizebox{\textwidth}{!}{%
\begin{tabular}{l|ccc|cccccc}
\toprule
& \multicolumn{3}{c|}{\textbf{One-Time Training}}
& \multicolumn{6}{c}{\textbf{Inference Efficiency \& Memory Breakdown}} \\
\cmidrule(lr){2-4}
\cmidrule(lr){5-10}
\textbf{Method}
& \textbf{Params.}
& \textbf{Time (h)}
& \textbf{Peak (GiB)}
& \textbf{Model (GiB)}
& \textbf{Trans. (GiB)}
& \textbf{KV (GiB)}
& \textbf{KV/1K (GiB)}
& \textbf{Succ. Turns}
& \textbf{Acc.} \\
\midrule
\ctoc
& 8.405M & 0.213 & 16.67
& 13.477 & 0.016 & 0.242 & 0.631 & 3.50 & 0.00 \\

\interlat
& 111.747M & 0.431 & 23.78
& 13.477 & 0.208 & 0.826 & 0.686 & 12.10 & 0.65 \\

\lsc
& 3.528B & 0.483 & 46.67
& 13.477 & 6.571 & 0.281 & 0.634 & 3.80 & 0.00 \\

\textbf{\mot}
& 2.192B & 2.350 & 38.00
& 13.477 & 4.082 & 0.456 & \textbf{0.235} & \textbf{12.95} & \textbf{0.70} \\
\bottomrule
\end{tabular}%
}
\end{table}

Raw KV usage is not directly comparable when methods produce different
discussion lengths: a failed agent can terminate early and therefore allocate
less KV cache. We regard a turn as successful only when the agent produces a
valid response choice required by the discussion task---\textit{Yes} or
\textit{No} for StrategyQA---and the response passes verification. A turn that
fails to produce such a valid choice is treated as failed.

Let \(\bar M_{\mathrm{KV}}\) be the mean KV usage of a run,
\(\mathcal T_e^{+}\) the successful turns of example \(e\), and
\(n_{e,t}\) the number of completion tokens generated at successful turn \(t\).
We report token-normalized KV usage as
\[
M_{\mathrm{KV/1K}}
=
\frac{1}{|\mathcal E^{+}|}
\sum_{e\in\mathcal E^{+}}
\frac{
1000\,\bar M_{\mathrm{KV}}
}{
\sum_{t\in\mathcal T_e^{+}} n_{e,t}
},
\]
where \(\mathcal E^{+}\) contains examples with at least one successful turn.
The \textbf{Succ. Turns} column reports
\[
\frac{1}{|\mathcal E^{+}|}
\sum_{e\in\mathcal E^{+}}
|\mathcal T_e^{+}|.
\]
This normalization removes the apparent memory advantage obtained simply by
terminating the discussion early.

The resulting efficiency picture is clear. \ctoc and \lsc average
only 3.50 and 3.80 successful turns and achieve zero accuracy; their small raw
KV values therefore largely reflect short failed discussions. After
normalization, \mot has the lowest KV cost by a wide margin:
\(0.235\) GiB/1K tokens versus \(0.631\)--\(0.686\) for the baselines, while
also achieving the highest accuracy (\(0.70\)) and the largest average number
of successful turns (\(12.95\)). Thus, \mot combines the strongest reasoning
quality with the best normalized KV efficiency in the heterogeneous 10-agent
setting.

\mot does incur the largest training time (2.35\,h) and substantially more
trainable parameters than lightweight baselines, but this is a one-time cost.
We release the trained translator checkpoints%
, so checkpoint users can avoid this training stage entirely. Even when training
is required, the cost is hours-scale and therefore minor relative to the
multi-day timescale of LLM pretraining or finetuning. More importantly,
minimizing this one-time cost does not recover the same downstream behavior:
none of the compared alternatives matches the accuracy and normalized KV
efficiency obtained by \mot with Context Correction Loss.

\interlat is the strongest baseline at 10 agents with \(0.65\) accuracy, but
its normalized KV cost is \(0.686\) GiB/1K tokens, nearly three times that of
\mot. The following homogeneous agent-scalability experiment further shows
that its accuracy decreases as the agent count grows, whereas \mot remains
substantially more stable.

\paragraph{Homogeneous Agent Scalability.}

We separately evaluate agent-count scalability using homogeneous
\Llama~\textsf{(3B)}$\to$\Llama~\textsf{(3B)} cache translation, increasing the
number of agents from 2 to 16 in increments of two. The maximum number of agent
turns is twice the number of agents, and each configuration is evaluated on
20 shuffled StrategyQA examples with generation limited to 256 new tokens.

We report the same successful-token-normalized KV memory defined above to
account for differences in discussion length caused by early failures. The
resulting accuracy and memory trends are shown in Figure~\ref{fig:agents} and
discussed in Section~\ref{sec:case_memory_reuse}.

\subsection{Long-Context CAG}
\label{app:long-context}

\paragraph{Long-Context Specific Ablation.}

\begin{figure*}[t]
    \centering

    \begin{minipage}{0.48\textwidth}
        \centering
        \includegraphics[width=\textwidth,height=4.4cm]{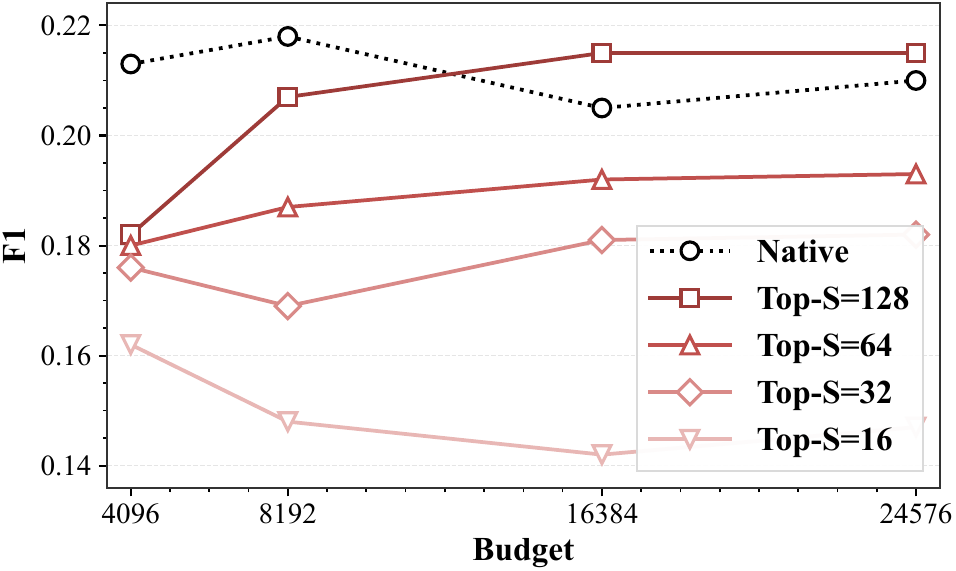}
        \captionof{figure}{Effect of Top-\(S\) for source-guided sparse attention across long-context budgets.}
        \label{fig:top_s}
    \end{minipage}
    \hfill
    \begin{minipage}{0.48\textwidth}
        \centering
        \includegraphics[width=\textwidth,height=4.4cm]{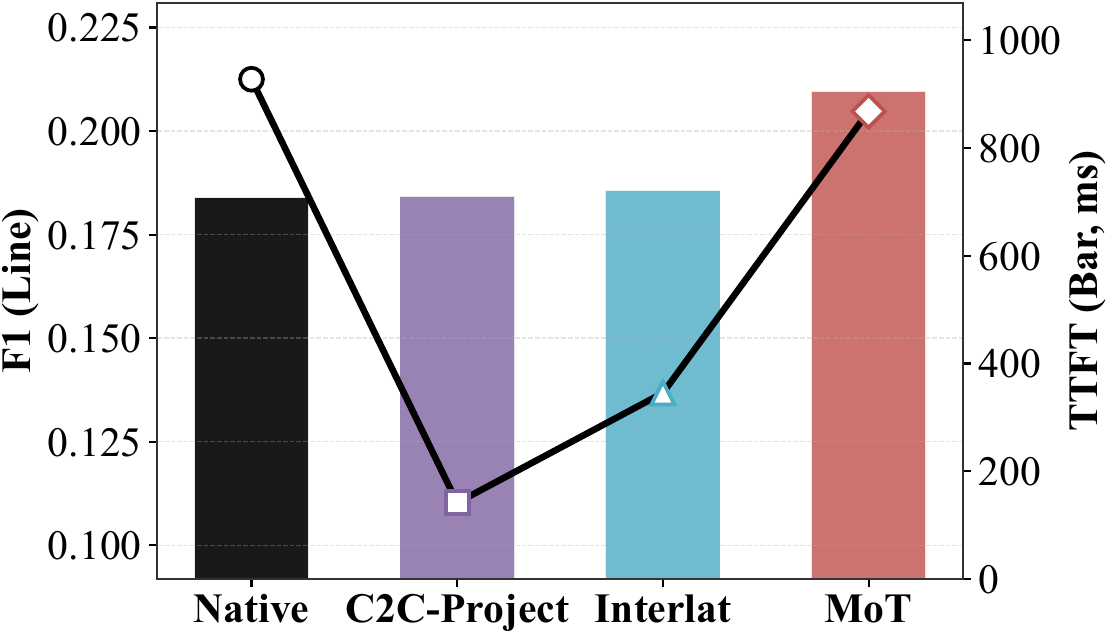}
        \caption{Average F1 and TTFT for long-context CAG.}
        \label{fig:appendix_avg_f1_ttft}
    \end{minipage}

\end{figure*}

Figure~\ref{fig:top_s} shows the F1 trend of \mot according to the Top-\(S\) size used in source-guided sparse attention. Top-\(S\) denotes the number of source-side sparse-attention candidates to be referenced during cache reuse, and this analysis is conducted to select the representative \mot configuration used in the subsequent method-level comparison. Overall, Top-\(S=128\) achieves the highest F1. Although it is lower than \native under the 4K budget, it improves rapidly from 8K onward and reaches F1 comparable to or higher than \native at 16K and 24K. This shows that, when a sufficient number of source-side candidates is used, \mot can effectively reuse long-context information from the stored KV cache.

By contrast, performance becomes limited as Top-\(S\) decreases. Top-\(S=64\) shows a gradual improvement as the budget increases, but it remains below Top-\(S=128\), while Top-\(S=32\) and Top-\(S=16\) show lower F1 overall. In particular, Top-\(S=16\) does not improve as the budget increases and instead tends to degrade, suggesting that an overly small sparse-attention candidate set is insufficient for preserving or reconstructing useful information from long contexts. Therefore, we use Top-\(S=128\), which shows the most stable F1, as the representative setting for \mot in this paper.

\paragraph{Long-Context CAG Results.}

Figure~\ref{fig:appendix_avg_f1_ttft} summarizes the average F1 and average TTFT over all context budgets. In terms of average F1, \native achieves the highest performance with 0.21, while \mot achieves 0.20, resulting in a gap of only 0.01 F1 points from \native. This means that \mot preserves approximately 96.3\% of the F1 of \native, showing that the quality loss in generation is limited even in the CAG setting where the original context is not directly prefilled.

In contrast, the average F1 scores of \ctoc and \interlat are only 0.11 and 0.14, respectively. \mot achieves approximately 85.8\% higher average F1 than \ctoc and approximately 49.6\% higher average F1 than \interlat. This gap suggests that, in the CAG setting, what matters is not merely reusing a cache, but whether the reused cache is preserved in a form that the target model can effectively use during generation.

In terms of average TTFT, \mot shows higher latency than the other CAG methods. This is because \mot performs both cache transformation and context replay. However, given that the average F1 gap from \native is very small, this overhead can be interpreted not simply as inefficiency, but as the cost of preserving long-context QA quality.

Figure~\ref{fig:budget_f1_ttft_analysis} in Section~\ref{sec:case_storage_reuse} shows that these average results are not dependent on a specific budget, but are also consistent with the trend as the budget increases. In particular, \mot maintains F1 close to \native under large context budgets, supporting that the average performance gap in Figure~\ref{fig:appendix_avg_f1_ttft} reflects quality preservation across the overall long-context regime.

\paragraph{Prompt Design.}

\begin{bluebox}[HotpotQA-E sample]{width=\linewidth, valign=top, breakable}
\label{text:prompt_hotpotqa_e_sample}

\noindent\textbf{Dataset.}
HotpotQA-E is a long-context multi-hop QA benchmark with multiple Wikipedia-style documents.
The model must find the required evidence in the long multi-document context and answer a short question.

\vspace{0.8em}

\begin{translatedcontextbox}
{\ttfamily\small
Read the following context and complete the task.

\vspace{0.6em}
Context: Title: The Oberoi Group\\
The Oberoi Group is a hotel company with its head office in Delhi. The company owns and operates luxury hotels and resorts under the Oberoi and Trident brands. \(\ldots\)

\vspace{0.4em}
Title: Oberoi Family\\
The Oberoi family is an Indian family that is famous for its involvement in hotels through The Oberoi Group. Members of the family have been associated with the development of luxury hospitality in India. \(\ldots\)

\vspace{0.4em}
Title: Delhi\\
Delhi, officially the National Capital Territory of Delhi, is a city and union territory of India containing New Delhi, the capital of India. \(\ldots\)
}
\end{translatedcontextbox}

\begin{promptbox}
{\ttfamily\small
Question or task: The Oberoi family is part of a hotel company that has a head office in what city?\\
Answer:
}
\end{promptbox}

\begin{completionbox}
{\ttfamily\small
Delhi
}
\end{completionbox}

\end{bluebox}

\newpage

\section{License}
\label{app:license}

We summarize the licenses and terms of use of the datasets, external method
implementations, and model checkpoints used in this work.

\paragraph{Benchmarks and training data.}
BoolQ is released under Creative Commons Share-Alike 3.0.
PubMedQA's official repository is released under the MIT license, while the
underlying PubMed abstracts may be subject to their original publisher terms.
MMLU-Redux is released under CC-BY-4.0.
SQuAD-v1.1 and HotpotQA are released under CC BY-SA 4.0.
StrategyQA is evaluated using the task definition distributed through
BIG-bench, which is released under Apache-2.0.
NewsQA is distributed under the official Microsoft Research download terms, and
the underlying CNN articles remain subject to CNN's rights.
OpenWebText uses an Apache-2.0 dataset script, but the underlying web documents
may have heterogeneous upstream terms; we use it only for research training and
do not redistribute raw examples.

\paragraph{External methods.}
The official implementations of \ctoc~\citep{c2c}, \interlat~\citep{interlat},
\kvcomm~\citep{kvcomm}, and MALLM~\citep{becker2025mallm} are released under
Apache-2.0.
Since no official open-source implementation of \lsc~\citep{lsc} was available,
we implemented it based on the corresponding paper.
Implementation details and source repositories are summarized in
Appendix~\ref{app:baseline-implementations}.

\paragraph{Models.}
Qwen2.5 checkpoints used in this work are released under Apache-2.0.
GPT-2 family checkpoints are released under the MIT / OpenAI Modified MIT
license.
DistilGPT2 and Pythia checkpoints are released under Apache-2.0.
DialoGPT-small is released under MIT.
Llama 3.2 checkpoints used in the multi-agent case study are governed by the
Llama 3.2 Community License.
OPT checkpoints are governed by the OPT license agreement, which permits
non-commercial research use and restricts commercial or production use.
For \texttt{MathGPT2}, we follow the license of the specific checkpoint used;
if \texttt{FlameF0X/MathGPT2} is used, it is released under MIT.

\paragraph{Usage.}
All third-party datasets, model checkpoints, and external method implementations
are used only for research, training, and evaluation. We do not redistribute the
raw benchmark data, model weights, or third-party source code.

\section{Limitations and Broader Impacts}
\label{app:limitations_broader_impacts}

This work has several limitations. First, as shown in the shift and correction
analysis in Section~\ref{sec:root_cause_analysis}, MoT reduces the final
last-state shift \(\|\bs_L\|\) mainly by decreasing the translation shift
\(\|\bs_T\|\) and the orthogonal shift component \(\beta_{T:L}\). However, the
correction-deficit factor \(d_{T:L}\) changes only marginally. This suggests
that MoT improves the translated trajectory without directly increasing the
target model's intrinsic correction ability. Developing objectives or
architectures that explicitly reduce correction deficit remains an important
future direction.

Second, our channel mapping relies on depth-ratio alignment as a simple and
effective default strategy. As discussed in Appendix~\ref{app:channel-mapping},
this assumption can become weaker when models differ substantially in instruction
tuning, training data distribution, or representation organization. In such
cases, relative layer depth may not provide a sufficiently reliable mapping
between source and target cache spaces. Future work should therefore investigate
adaptive channel mapping methods that account for stronger model mismatch.

The broader impact of MoT is primarily tied to efficient reuse of internal model
states. By reducing redundant prefilling, memory duplication, and repeated cache
storage, MoT can lower computational cost and energy consumption in multi-model
and long-context LLM workflows, potentially contributing to more resource-
efficient and lower-carbon inference. At the same time, making multi-agent and
multi-model systems more scalable may accelerate the deployment of highly
automated agentic workflows. While this can improve productivity, it also raises
questions about oversight, accountability, and the role of human judgment when
increasingly capable agents are used in decision-support settings. These risks
are not specific to cache translation, but improved system efficiency can make
such deployments easier, and therefore should be considered alongside technical
progress.